\documentclass{article}

\newif\ifpaper 

\usepackage[margin=1in]{geometry}

\papertrue

\usepackage[round]{natbib}
\renewcommand{\bibname}{References}
\renewcommand{\bibsection}{\subsubsection*{\bibname}}

\usepackage{amsmath}%nosumlimits
\usepackage{amsthm}
\usepackage{amssymb}
\usepackage{bbm}
\usepackage{paralist}
\usepackage{enumitem}
\usepackage{algorithm}
\usepackage{algorithmic}
\usepackage{xcolor}
\usepackage{graphicx}
\usepackage{epstopdf}
\usepackage{subfig}
\usepackage{multirow}
\usepackage{paralist}
\usepackage{cite}
\usepackage[pagebackref=false,breaklinks=true,colorlinks,bookmarks=false,citecolor=blue]{hyperref}

\newtheorem{theorem}{Theorem}

\newtheorem{lemma}{Lemma}
\newtheorem{proposition}{Proposition}

\DeclareMathOperator{\EXP}{\mathbb{E}}
\renewcommand{\Pr}{\mathbb{P}}
\renewcommand{\cite}{\citep}

\renewcommand{\mid}{\, | \, }
\newcommand{\longmid}{\, \middle| \, }
\newcommand{\note}[1]{[\textcolor{red}{\textit{#1}}]}
\newcommand{\MF}{\bf \sf}

\newcommand{\bayIter}{\text{BI}}
\newcommand{\nonIter}{\text{NBI}}
\newcommand{\fullBayIter}{the iterative algorithm}

\usepackage{cleveref}
\crefname{equation}{}{}
\crefrangeformat{equation}{\textnormal{(#3#1#4)--(#5#2#6)}}
\crefname{figure}{}{}
\crefrangeformat{figure}{\textnormal{(#3#1#4)-#5#2#6}}

\allowdisplaybreaks

\title{Iterative Bayesian Learning for Crowdsourced Regression} %Object Detection}

\author{
Jungseul Ok\thanks{
J. Ok and S. Oh are with
the Coordinated Science Lab at the University of Illinois at Urbana-Champaign, Illinois, USA  (e-mail: \texttt{\{ockjs, swoh\}@illinois.edu}).
},~~Sewoong Oh\footnotemark[1],~~Yunhun Jang
\thanks{
Y. Jang, J. Shin, and Y. Yi are with
the Department of Electrical Engineering at Korea Advanced Institute of Science and Technology,
Daejeon, South Korea (e-mail: \texttt{\{cirdan, jinwoos, yiyung\}@kaist.ac.kr})
}
,~~Jinwoo Shin\footnotemark[2],~~and~~Yung Yi\footnotemark[2]
}

\begin{document}



\maketitle

% !TEX root =  main.tex

\begin{abstract} 

  Crowdsourcing platforms emerged as popular venues for 
  purchasing human intelligence at low cost for large volume of tasks. 
  As many low-paid workers are prone to give noisy answers, 
  a common practice is to add redundancy by assigning multiple workers to 
  each task and then simply average out these answers.
  % using simple techniques 
  %such as averaging.  
  However, to fully harness the wisdom of the crowd, 
  one needs to learn the heterogeneous quality of each worker. 
  We resolve this fundamental challenge in 
  crowdsourced regression tasks, i.e., %in particular, when
the answer takes continuous labels,
where identifying good or bad workers becomes much more non-trivial compared
to a classification setting of discrete labels.
 %, i.e., for regression tasks.
%such as position of a target object in an image. 
  In particular, %aggregate the collected answers from the crowd, 
  we introduce a Bayesian iterative scheme and
  show that it provably achieves the optimal mean squared error. 
  Our evaluations on synthetic and real-world datasets support our theoretical results and
  show the superiority of the proposed scheme.

\end{abstract}

%%% Local Variables:
%%% mode: latex
%%% TeX-master: "main"
%%% End:

% !TEX root =  main.tex

%\vspace{-0.2cm}
\section{Introduction}

%\note{Crowdsourcing is popular. But, can be noisy. Redundancy.}
Crowdsourcing systems provide a labor market where 
numerous pieces of classification and regression tasks 
are electronically distributed to a crowd of workers, 
who are willing to solve such  human intelligence tasks at a low cost. 
%To a data analyst, 
%such systems provide unprecedented accesses to training datasets 
%at a scale and budget that was not previously feasible.
%Thus obtained training datasets can then be seamlessly integrated into  
%downstream machine learning tasks together with the state-of-the-art classification and regression methods. 
However, because the pay is low and the tasks are tedious, error is common even among those who are willing. 
This is further complicated by abundant spammers trying to make easy money with little effort. 
%\note{In aggregation, estimation of reliability is the key.}
To cope with such noise in the collected data, adding redundancy is a common and powerful strategy widely used in real-world crowdsourcing. 
Each task is assigned to multiple workers and these responses are aggregated by inference algorithms
such as averaging (for real-valued answers) or majority voting (for categorial answers). 
As workers' qualities are heterogeneous, 
such simple approaches can be significantly improved upon by 
%giving more weight to 
re-weighting the answers from reliable workers. %, the weighted averaging or majority voting can significantly improve upon the simple algorithm
Here, the fundamental challenge is identifying such workers, which requires estimating ground truth answers and vice-versa.
Our focus is solving this inference problem, 
when neither true answers nor worker reliabilities are known. 
% such fully {\em unsupervised} inference problems.

For a simpler problem of  {\em classification tasks}, 
where each task asks a worker to choose one label from a discrete set, 
significant advances have been made in the past decade 
\cite{kos2011,liu2012,KO16,shah2016permutation,ZLPCS15,zhang2014} 
based on the model proposed in the seminal work of \cite{dawid1979}. 
Deep theoretical understanding of the model under a 
simple but canonical case of binary classification has led to 
the design of powerful inference algorithms, which significantly improve upon the common practice of majority voting on  
real-world datasets. 
However, neither the model nor the algorithms generalize to 
{\em regression tasks}, where each task asks for a continuous valued assessment, %of infinite possibilities, 
and possibly in multiple dimensions.
%It has remained a mystery if 
%the common practice of averaging can be provably improved upon in theory 
%and can be outperformed in real-world datasets. 
	Despite of the significance 
of the crowdsourced regression evidenced by the empirical studies \cite{pascal15,su2012crowdsourcing,deng2009imagenet,de2014crowdgrader,piech2013tuned},
%and such advance of theoretical understanding on the unsupervised classification,
the theoretical understanding of the crowdsourced regression has remained limited.

To bridge this gap,  
we take a principled approach on this crowdsourced regression problem to theoretically investigate the tradeoff involved. % in the {\em unsupervised} setting.
%\note{I'm not sure this positioning is the best...}
%while on crowdsourced regression, the most existing work is empirical study and theoretical understanding is limited to the simple supervised setting when some answers are revealed in advance.
More precisely, we ask the fundamental question of 
how to achieve the best accuracy given a budget constraint, or 
equivalently how to achieve a target accuracy with minimum budget. 
As in typical crowdsourcing systems, we assume we pay a fixed amount for  each response, 
and thus the budget per task is proportional to the redundancy: how many answers we collect for each task.

\noindent {\bf Contribution.}
%We propose a probabilistic model for crowdsourced regression, where each worker 
%provides a noisy observation on the (continuous valued) ground truths with his/her own noise level.
Inspired by the simplicity of the model in \cite{dawid1979} for crowdsourced classification, 
we propose a simple, yet effective model for crowdsourced regression. 
We introduce a Bayesian Iterative algorithm ({\bayIter}) to solve the inference problem efficiently. 
We provide an upper bound on the error achieved by the proposed {\bayIter}
(Theorem~\ref{thm:quantify}) 
that captures $(i)$ the fundamental tradeoff between redundancy 
%(as measured by how many answers we collect per task) 
and the accuracy, %(as measured by the mean squared error). 
and $(ii)$ the performance loss due to the difficulty in estimating workers' reliability. 
%To overcome these challenges, we first propose a simple, but canonical probabilistic model for crowdsourced regression,
%in which we simplify the workers' noise level into several classes, while workers provides continuous valued answers.
%%We formulate the inference problem to minimize mean squared error (MSE).
%Despite of the simplified noise model, 
%As  the optimal inference algorithm is computationally intractable, 
Further, we prove that it is information theoretically impossible for 
any other algorithm to improve upon {\bayIter}. 
This is achieved by coupling the proposed inference algorithm with 
a carefully constructed oracle estimator, and 
showing that there is no gap in the performance between those two algorithms (Theorem~\ref{thm:optimality}).  
Such strong guarantees are only known for a few other cases % in similar probabilistic inference problems, 
even under more strict assumptions (which we discuss later). 
%while our simulation suggests the optimality of {\bayIter} without the assumptions.
Finally, in numerical evaluation, we confirm our theoretical findings on synthetic and real-world datasets.

\noindent {\bf Related work.} 
Crowdsourcing systems are widely used in practice for a variety of real-world tasks 
%ranging from classification and 
%regression \cite{MKM12} to 
%more complex tasks 
such as 
protein folding \cite{PQIB13},
searching videos \cite{BBMK11, DGK13}, 
ranking \cite{LSD12}, peer assessment \cite{piech2013tuned, goldin2011peering} and
natural language processing \cite{WFY12}. 
However, recent theoretical advances have been focused on  
crowdsourced classification tasks to 
$(a)$ design  algorithms for aggregating answers from multiple workers on the same task;
$(b)$ analyze the performance achieved by such algorithms; 
and $(c)$ identify and compare against the fundamental limit 
\cite{kos2011,KOS13SIGMETRICS,GKM11,zhang2014,ok2016icml,ZPBM12,DDKR13,liu2012,kos2014}. %ZPBM12,DDKR13, ,liu2012,kos2014
In this paper, we theoretically  investigate these fundamental questions for {\em crowdsourced regression}. 

%\note{Please make the following paragraph compact and beautiful :)}
% It is worth to note previous approaches to address real-valued answers from crowdsourcing system.

There has been several novel algorithms recently proposed for the crowdsourced regression. 
\citet{RYZ10} proposed a probabilistic model and a corresponding maximum likelihood estimator, 
but no supporting theoretical or empirical analysis is provided (as the estimator is intractable).
\citet{ZLPCS15} propose a heuristic of 
quantizing the continuous valued answers and 
reducing it to discrete models, i.e.~crowdsourced classification. 
On top of being sensitive to hyperparameter choices such as the quantization level, 
treating the answers as categories loses the fundamental aspect that the answers are given in a metric space where distances are well-defined. 

A related work is \cite{liu2013scoring}, 
in which the authors provide a theoretical understanding in a {\em semi-supervised} setting. 
All workers are first asked golden questions with known answers, 
which is used to estimate all unknown parameters of the workers. 
Then, they are assigned to tasks with unknown answers, and
their responses are aggregated using the estimated parameters. 
As this two phase approach completely de-couples 
the uncertainty in worker parameters and task answers, 
 the analysis is extremely simple and 
%the analysis 
is not applicable to our {\em unsupervised} setting. % where golden questions are not used.

Finally, we remark that the proposed algorithm {\bayIter} 
is a variant of the popular Belief propagation (BP).
%Technically, 
%we extend the analysis of 
%Belief propagation (BP), which  
%is a widely used heuristic for solving inference problems on probabilistic graphical models. 
%such as our crowdsourcing regression model. 
Although BP enjoys numerous empirical successes in various fields
%such as bioinformatics, information retrieval, speech processing, image processing and communications 
\cite{jordan2004graphical}, 
its theoretical analysis has been limited to  a few
instances including community detection \cite{mossel2014}
%crowdsourced classification \cite{ok2016icml},
 and error
correcting codes \cite{kudekar2013spatially}.
%and combinatorial
%optimization \cite{park2015max}.
%Basically, our analysis starts with the locally tree like structure of spar
In particular, those analyses showing the optimality of loopy BP
\cite{mossel2014, ok2016icml} are limited to cases where the corresponding factor graph has only factor degree two.
%\note{Need to check if \cite{kudekar2013spatially} is also the case with factor degree two.}
Our main result (Theorem~\ref{thm:optimality}) extends the horizon of such 
cases where {\bayIter} provably finds the optimal inference %on factor graph with 
under an arbitrary factor degree while the regression problem is more challenging to analyze than 
the discrete models studied in \cite{mossel2014, ok2016icml} as the regression error is unbounded.

%follows the proof strategy initially introduced in  \cite{mossel2014} where
% BP is shown to be optimal for community detection, and generalized in \cite{ok2016icml} for crowdsourced classification problems.

%{\bf[Jungseul should fill in with  technical related work for Theorem 1, but
%I'm not sure there is a related work on this. Let me check....]}. 

%The optimal estimator is computationally intractable, similar to the crowdsourced classification problem \cite{liu2012, ok2016icml}. 
%However, the marginalization is more challenging than the one in the classification problem
%since it is often intractable to just run an algorithm for approximating 
%the marginal probability of continuous variables, e.g., 
%  Belief Propagation is the popular message passing algorithm to approximate the marginalization
%  but the messages corresponds to functions of continuous variables 
%  so that they are expensive to compute and transmit.
%However, our model discretizes the worker's noise level
%and this special structure allows us to
% bypass such challenges by transforming the regression problem into
%the marginalization for discrete noise levels instead of continuous ground truths.

%%% Local Variables:
%%% mode: latex
%%% TeX-master: "main"
%%% End:

% !TEX root =  main.tex
%\vspace{-0.05in}
\section{Problem Formulation}

\subsection{Crowdsourced Regression Model}
The task requester has a set of $n$ regression tasks, denoted by $V=\{1,\,\ldots,n\}$,
where task $i \in V$ is associated with the true position $\mu_i \in \mathbb{R}^d$.
%As a running example, consider object detection 
% where a worker is asked to locate the position 
%of an object of interest. , e.g., the center of a galaxy or the center of a marker in a Cryo-EM image. 
%For intuitive explanation and without loss of generality,
% e.g., 
%we assume that each task
%$i\in V$ is an image with an object at position
%$\mu_i \in \mathbb{R}^d$. 
%We also assume that position $\mu_i$'s are
%arbitrarily given but unknown to the learner. 
To estimate these unknown true positions, we assign the
tasks to a set of $m$ workers, denoted by $W=\{1,\ldots,m\}$ according to a bipartite
graph $G = (V, W, E)$, where edge $(i, u) \in E$ indicates that task
$i$ is assigned to worker $u$.  We also
let
$N_u : = \{i \in V : (i, u) \in E\}$ and
$M_i : = \{u \in W : (i, u) \in E\}$
denote the set of tasks assigned
to worker $u$ and the set of workers to whom task $i$ is assigned, respectively.

When task $i$ is assigned to worker $u$, she provides  
her estimation/guess $A_{iu} \in \mathbb{R}^d$ for the true location $\mu_i$. 
Each worker $u$ 
is parameterized by her noise level $\sigma^2_u $, 
such that the response $A_{iu}$ suffers from 
an additive spherical Gaussian noise with variance $\sigma_u^2$. 
Precisely, conditioned on $\mu_i$ and $\sigma_u^2$, 
 $A_{iu}$ is independently distributed with Gaussian pdf 
$ f_{A_{iu}}(x \mid \mu_i, \sigma^2_u ) = 
\phi (x \mid \mu_i, \sigma^2_u ) := \exp (-{\|x - \mu_i\|^2_2}/(2{\sigma^2_u}) )/{\sqrt{(2\pi \sigma^2_u)^d}}.
$
%\begin{align*}
%	\Pr[a\leq A_{iu}\leq b | \mu_i, \sigma^2_u] \int^a_b
%	\;\; = \;\;  f_{\mathcal{N}} (x \mid \mu_i, \sigma^2_u )~dx,
%\end{align*}
%where $\mathcal{N} (x \mid \mu_i, \sigma^2_u )$ is the  density function with mean $\mu_i$ and variance $\sigma_u^2$, i.e.,
% We assume that each worker $u$'s variance $\sigma^2_u $ is
% independently drawn from the discrete uniform distribution with
% support
% $\mathcal{S}(\varepsilon) = \{\sigma^2_1, ..., \sigma^2_k\} \subset
% (0, \infty)$
% where $\sigma^2_1 < ... < \sigma^2_k$ and
% $\sigma^2_{k'+1}-\sigma^2_{k'} \ge \varepsilon$ for each $k' \le k$.

We assume that each worker $u$'s variance $\sigma^2_u $ is
independently drawn from  a finite set 
%the discrete uniform distribution with finite
%support
$\mathcal{S}= \{\sigma^2_1, ..., \sigma^2_{S}\}$ 
uniformly at random. %For the mathematical rigorousness,
 We further assume that the true position $\mu_i$ is independently drawn 
 from a Gaussian prior distribution $\phi(x \mid \nu_{i}, \tau^2)$ for given mean $\nu_i \in \mathbb{R}^d$ 
 and variance $\tau^2 \in (0, \infty)$, which can be interpreted as a side information on true positions. 
Note that we just take the Gaussian prior for the simple expression and our analysis can be generalized to other distributions, e.g., a uniform distribution on a Euclidean ball.
Our analysis is valid for arbitrarily large $\tau$, i.e., no prior information,
and our numerical experiments assume no knowledge of the prior distribution by taking $\tau \to \infty$.
% Even though this results in mismatched inference algorithm, 
% it achieves impressive results in all numerical experiments 
% and bypasses the knowledge of prior distribution. 
%A natural extension of this model is to assume a bias for each of the workers, 
% which we discuss in 
%Section \ref{sec:conclusion}. 
Theoretical understanding of such a simple but canonical model 
allows us to characterize the tradeoffs involved and 
provides guidelines for designing practical algorithms.
\subsection{Optimal but Intractable Algorithm}

%\subsection{Optimal Algorithm of Intractable Complexity}
\label{sec:MMSE-problem}
% Object Detection Problem}

Under the crowdsourcing model, our goal is to design 
an efficient estimator $\hat{\mu}(A) \in \mathbb{R}^{d \times V}$ of the unobserved
true position $\mu$ from the noisy answers $A := \{A_{iu} : (i,u) \in E\}$
reported by workers. 
In particular, we are interested in minimizing the
average of {\it (expected) mean squared error (MSE)}, i.e.,
\begin{align}   \label{eq:optimization}
\underset{\hat{\mu}: \text{estimator}}{\text{minimize}} 
\quad \frac{1}{n} \sum_{i \in V} \EXP [\text{MSE}(\hat{\mu}_i(A))]
\end{align}
where we define $
\text{MSE} (\hat{\mu}_i(A)) := 
 \EXP  [ \left\|\hat{\mu}_i(A) - \mu_i \right\|^2_2 | A ] %\footnote{$\| \cdot \|_2$ is the Euclidean norm. }
 $  as the MSE conditioned on $A$.
%\begin{align*}
%\end{align*}
%where %the expectation is taken over the probability measure given $A$and 
%Let $\hat{\mu}^*_i (A) := \underset{\hat{\mu}_i}{\arg \min} \EXP  \left[ \left\|\hat{\mu}_i(A) - \mu_i \right\|^2_2 \longmid A \right]$.
%Let $\hat{\mu}^* (A)$ denote the optimal estimator in \eqref{eq:optimization}.
%Then, 
%By rewriting MSE with
Using the equality 
$\left(\hat{\mu}_i(A) - \mu_i \right) = \left(\hat{\mu}_i(A) - \EXP[\mu_i \mid A] \right) +
\left(\EXP[\mu_i \mid A] - \mu_i \right),$
it is straightforward to check that for each $i \in V$, MSE is minimized at % the optimal estimator
the {\it Bayesian} estimator
$\hat{\mu}^*_i (A)  := \EXP [\mu_i \mid A]$, which is
%$\hat{\mu}^*_i (A)$ of \eqref{eq:optimization} is the conditional
%expectation of $\mu_i$ given $A$, i.e., % $\hat{\mu}^*_i (A)  = \EXP [\mu_i \mid A].$
\begin{align} 
\hat{\mu}^*_i (A) 
 = \sum_{\sigma^2_{M_i} \in \mathcal{S}^{M_i}} 
 \bar{\mu}_i\left(A_i, \sigma^2_{M_i} \right)  \Pr [\sigma^2_{M_i} \mid A]  \label{eq:MAP}
% \EXP [\mu_i \mid A_i, \sigma^2_{M_i}]  \cdot \Pr [\sigma^2_{M_i} \mid A]  \label{eq:MAP}
\end{align}
where we let $A_i := \{A_{iu} : u \in M_i\}$ and $ \bar{\mu}_i (A_i, \sigma^2_{M_i} ):=
\EXP [\mu_i \mid A_i, \sigma^2_{M_i}] 
= \bar{\sigma}^2_i \left(\sigma^2_{M_i}\right)
  ( {\nu_i}/{\tau^2} + \sum_{u \in M_i} {A_{iu}}/{\sigma^2_{u}} ) $
%  \label{eq:optimal_mu_given_sigma}
%\end{align} 
with 
$\bar{\sigma}^2_i(\sigma^2_{M_i} ) := ({{1}/{\tau^2} + \sum_{u \in M_i}  {1}/{\sigma^2_{u}}} )^{-1}$.
%\begin{align*}
%\bar{\sigma}^2_i \left(\sigma^2_{M_i} \right) := \frac{1}{\left({\frac{1}{\tau^2} + \sum_{u \in M_i}  \frac{1}{\sigma^2_{u}}} 
%\right)}.
%\end{align*}
%
We provide a derivation of this formula in the supplementary material.
The calculation of the marginal posterior $\Pr[\sigma^2_{M_i} \mid A]$ 
is computationally intractable in general. More formally,
the marginal posterior of $\sigma^2_{M_i}$
%conditional probability of $\sigma^2_{M_i}$ given $A$
%$\Pr[\sigma^2_u \mid A]$ 
%in \eqref{eq:MAP_calc} 
can be calculated by
 marginalizing out
$\sigma^2_{-i} := \{\sigma^2_v : v \in W \setminus M_i \}$ 
from the joint probability of $\sigma^2$, i.e.,
\begin{align} 
  \Pr[\sigma^2_{M_i} \mid A] = \sum_{\sigma^2_{-i} \in  \mathcal{S}^{W \setminus M_i} }  \Pr[\sigma^2 \mid A]  \label{eq:marginal_sum}
\end{align}
which requires exponentially many summations with respect to $m$.
 Thus,
the optimal estimator $\hat{\mu}^*(A)$ in
\eqref{eq:MAP}, requiring the marginal posterior $\Pr[\sigma^2_{M_i} \mid A]$
in \eqref{eq:marginal_sum},
 is {\it  computationally intractable} in general.

\section{%Provable Performance Guarantees on the 
Iterative Bayesian Learning}
\label{sec:result}

We now introduce a computationally tractable scheme, the Bayesian iterative (BI) algorithm,
and 
provide its theoretical guarantees %of the Bayesian iterative algorithm %{\fullBayIter}
under the crowdsourced regression model. 
For its analytic tractablity, 
%To do so, we first describe our proposed task assignment. 
%\noindent{\bf $ (\ell,r)$-regular task assignment.}
%In general, the performance of an estimator 
%depends on  how tasks are assigned to workers. 
we consider a popular 
assignment scheme, referred to as $(\ell,r)$-regular task assignment, 
widely adopted in crowdsourcing \cite{kos2011,ok2016icml}.
The assignment graph $G$ is 
 a random $(\ell,r)$-regular bipartite graph
drawn uniformly at random out of all $(\ell,r)$-regular graphs, 
where each task is assigned to $\ell$ workers
and each worker is assigned $r$ tasks.
Nevetheless, we remark that the BI algorithm is applicable to any (even, non-regular) task assignments.

\subsection{Bayesian Iterative (BI) Algorithm}
 \label{sec:bp}
 
We first factorize the joint probability of $\sigma^2$ in \eqref{eq:marginal_sum} as
%\footnote{The detailed derivation of $\Pr[ \sigma^2 \mid A]$ 
%\ifpaper %
%is in our supplementary material. % due to the space limitation.
%\else
%is given in Appendix~\ref{sec:model_calc2}.
%\fi
%}
%\begin{align}
%\Pr[ \sigma^2 \mid A] %& \propto \Pr[A\mid \sigma^2] =  \prod_{i \in V}  \Pr[A_i \mid \sigma^2_{M_i}] \nonumber \\
%\propto \prod_{i\in V}
% \mathcal{C}_i \left(A_i, \sigma^2_{M_i} \right) \; ,  \label{eq:sub_factor_graph}
%\end{align}
%$\Pr[ \sigma^2 \mid A]  \propto \prod_{i\in V} \mathcal{C}_i \left(A_i, \sigma^2_{M_i} \right)$
\begin{align*}
\Pr[ \sigma^2 \mid A]  \propto \prod_{i\in V} \mathcal{C}_i \left(A_i, \sigma^2_{M_i} \right)
\end{align*}
where 
%we define $\mathcal{C}_i (A_i, \sigma^2_{M_i} )$ as
%$\mathcal{C}_i (A_i, \sigma^2_{M_i} )
%:=  {({\bar{\sigma}^2_i (\sigma^2_{M_i})}/{ \tau^2  \prod_{u \in M_i} (2\pi \sigma^2_u)})^{\frac{d}{2}}} 
%\exp{\left(-\frac{1}{2} \bar{\sigma}^2_i (\sigma^2_{M_i})  \mathcal{D}_i (A_i, \sigma^2_{M_i} ) \right) }
%$
%
%$
%\mathcal{D}_i (A_i, \sigma^2_{M_i} )
%:= 
%\sum_{u \in M_i}  \frac{\|A_{iu} - \nu_{i}\|_2^2}{\sigma_u^2 \tau^2} 
%+
%\sum_{v \in M_i \setminus \{u\}} \frac{\|A_{iu} - A_{iv}\|_2^2}{\sigma_u^2 \sigma_v^2}$.
%$\mathcal{C}_i (\! A_i, \sigma^2_{M_i} \!) 
%:=  {\big(\frac{ 2\pi \cdot \bar{\sigma}^2_i (\sigma^2_{M_i})}{ 2\pi \tau^2_i  \prod_{u \in M_i} (2\pi \sigma^2_u)}\big)^{\frac{d}{2}}} 
%\cdot  \exp \big({
%-\frac{1}{2} \bar{\sigma}^2_i (\sigma^2_{M_i}) \cdot 
%\left(
%\sum_{u \in M_i}  \frac{\|A_{iu} - \nu_{i}\|_2^2}{\sigma_u^2 \tau^2} 
%+
%\sum_{v \in M_i \setminus \{u\}} \frac{\|A_{iu} - A_{iv}\|_2^2}{\sigma_u^2 \sigma_v^2}
% \right)} \big) .
%$
\iffalse
{
\begin{align*} 
&\mathcal{C}_i (A_i, \sigma^2_{M_i} ) \\
&:=  {\Big(\frac{\bar{\sigma}^2_i (\sigma^2_{M_i})}{ \tau^2  \prod_{u \in M_i} (2\pi \sigma^2_u)}\Big)^{\frac{d}{2}}} 
\exp{\left(\mathcal{D}_i (A_i, \sigma^2_{M_i} )   \right) }
\;, \quad \text{and}\\
&\mathcal{D}_i (A_i, \sigma^2_{M_i} ) \\
&\!:=\!
-\frac{\bar{\sigma}^2_i (\sigma^2_{M_i})}{2} \! 
\Bigg( \!
\sum_{u \in M_i}\!\!
  \frac{\|A_{iu} - \nu_{i}\|_2^2}{\sigma_u^2 \tau^2} 
+\!\!\!\!
\sum_{v \in M_i \setminus \{u\}} \!\!\!\!\!\!
 \frac{\|A_{iu} - A_{iv}\|_2^2}{\sigma_u^2 \sigma_v^2}
 \! \Bigg) \;.
\end{align*}
}
\fi
$\mathcal{C}_i (\!A_i, \sigma^2_{M_i}\!) \!:=\! \Big(\!\frac{\bar{\sigma}^2_i (\sigma^2_{M_{i\!}})}{ \tau^2  \prod_{u \in M_i} \! 2\pi \sigma^2_u}\!\Big)^{\!\frac{d}{2}} 
e^{-\mathcal{D}_{i\!}(\!A_{i\!}, \sigma^2_{M_i}\!)}$,
and 
$\mathcal{D}_i (\!A_i, \sigma^2_{M_i}\!) 
\!:=\!
\frac{\bar{\sigma}^2_i (\!\sigma^2_{M_i}\!)}{2}
\Big(
\underset{\scriptscriptstyle u \in M_i}{\sum} \!\!
  \frac{\|\!A_{iu\!} - \nu_{i} \|_2^2}{\sigma_u^2 \tau^2}
+\!\!\!
\underset{\scriptscriptstyle v \in M_i \!\setminus\! \{u\}}{\sum}
 \!\!\!
 \frac{\|\!A_{iu\!} - A_{iv\!}\|_2^2}{\sigma_u^2 \sigma_v^2}
 \! \Big).$

%\begin{align*}
%{\footnotesize
%{\bigg(\frac{\bar{\sigma}^2_i (\sigma^2_{M_i})}{ \tau^2_i  \prod_{u \in M_i} (2\pi \sigma^2_u)}\bigg)^{\frac{d}{2}}}
%\!   \exp{
%\bigg(\!
%-\frac{1}{2} \bar{\sigma}^2_i (\sigma^2_{M_i}) \!\cdot\!
%\bigg(
%\sum_{u \in M_i}  \frac{\|A_{iu} - \nu_{i}\|_2^2}{\sigma_u^2 \tau^2} 
%+\!\!
%\sum_{v \subset M_i \setminus \{u\}} \frac{\|A_{iu} - A_{iv}\|_2^2}{\sigma_u^2 \sigma_v^2}
% \bigg) \! \bigg)} .
% }
%\end{align*}
\ifpaper
We provide a derivation of this formula in the supplementary material.
\else
We provide a derivation of this formula in the appendix. %~\ref{sec:model_calc2}. 
\fi
This factorization of the joint probability of $\sigma^2$ given $A$
 forms a factor graph \cite{jordan1998learning} where each worker $u$'s
variance $\sigma^2_u$ and each task $i$
correspond to a variable and a local factor $\mathcal{C}_i(A_i, \sigma^2_{M_i})$ 
on the set of workers, $M_i$, to whom task $i$ is assigned,
respectively. 
This probabilistic graphical model 
motivates us to use  the popular (sum-product) belief propagation (BP)
algorithm \cite{pearl1982} on the factor graph of
$\Pr[ \sigma^2 | A]$ to approximate
the intractable computation of $\Pr[\sigma^2_{M_i} \mid A]$
in \eqref{eq:marginal_sum}. However, BP is typically used 
for approximating the marginal probability of a single variable 
$\sigma^2_u$, while we need the marginal probability
of a subset of variables $\sigma^2_{M_i}$
depending on each other.
Hence, to approximate the optimal Bayesian estimator in \eqref{eq:MAP},
we build upon BP and 
propose an iterative algorithm (\bayIter)
updating belief $b_i (\sigma^2_{M_i})$ from 
messages $m_{i \to u\!}$ and $m_{u \to i}$ between task~$i$ and worker~$u$:
\begin{align}
%&\text{\it From task $i$ to worker $u$:} \quad 
m^{t+1}_{i \to u} (\sigma^2_u)  
&\propto \sum_{\sigma^2_{M_i \setminus \{u\}}} \!\! \mathcal{C}_i(A_i, \sigma^2_{M_i}) 
\!\!\prod_{v \in M_i \setminus \{u\}} \!\!
m^{t}_{v \to i} (\sigma^2_v) \label{eq:iu-message}\\
%&\text{\it From worker $u$ to task $i$:} \quad
m^{t+1}_{u \to i} (\sigma^2_u) 
 &\propto \prod_{j \in N_{u} \setminus \{i\}} 
m^{t+1}_{j \to u} (\sigma^2_u) \label{eq:ui-message}\\
%b^{t+1}_{u} (\sigma^2_u ) & \propto \prod_{i \in N_{u}} m^{t+1}_{i \to u} (\sigma^2_u) 
%b^{t+1}_{iu} (\sigma^2_u )
%& \propto \prod_{j \in N_u \setminus \{i\}} m^{t+1}_{j \to u} (\sigma^2_{u}) \;,
%&\text{\it Calculating belief:} \quad
b^{t+1}_{i} (\sigma^2_{M_i} ) 
&\propto \mathcal{C}_i (A_i, \sigma^2_{M_i}) \prod_{u \in M_{i}} m^{t+1}_{u \to i} (\sigma^2_u ) 
\label{eq:belief}
\end{align}
where we initialize the messages 
with a trivial constant ${1}/{|\mathcal{S}|}$ 
and normalize the messages and beliefs so that
$\sum_{\sigma^2_u} m^t_{i
  \to u} (\sigma^2_u) =
   \sum_{\sigma^2_u} m^t_{u \to i} (\sigma^2_u) = 
   \sum_{\sigma^2_{M_i}} b^t_{i} (\sigma^2_{M_i}) = 1.$
At the end of $k$ iterations, as an approximation of the optimal Bayesian estimator in \eqref{eq:MAP},
we estimate
 $\hat{\mu}^{{\MF \bayIter}(k)}(A)$ using
\eqref{eq:MAP} with belief $b^{k}_{i}(\sigma^2_{M_i})$
as an approximation of $\Pr[\sigma^2_{M_i} \mid A]$. Formally,
\begin{align} 
\hat{\mu}^{{\MF \bayIter}(k)}_i (A)  
:= \sum_{\sigma^2_{M_i} \in \mathcal{S}^{M_i}} 
\bar{\mu}_i (A_i, \sigma^2_{M_i})
  b^{k}_{i}(\sigma^2_{M_i}) \;.
  \label{eq:BP-estimator}
\end{align}
Although the messages and their updates are the same as those of 
the typical BP, %for approximating the marginal probability 
%of a single variable. However, differently from the typical BP,
we use a specific form of belief in \eqref{eq:belief}
for approximating the marginal probability 
of a {\it subset} of dependent variables.
This allows us to provide sharp performance guarantees in Section \ref{sec:result},
while the typical BP for {\it single} variable marginalization has little known provable guarantees.

%we need to approximate the marginal probability of a subset of variables $\sigma^2_{M_i}$ depending on each other.
%We hence adjust BP for the crowdsourced regression
%and use a special form of belief in \eqref{eq:belief},
%which allow us to provide sharp performance guarantees in Section \ref{sec:result}, while the typical BP for single variable has little known provable guarantees.

 %after sufficient iterations
%\cite{pearl1982}, i.e., 
 % within $2|E|$ iterations. Hence the
%BP algorithm has the following property:
%\begin{property} \label{prop:tree-bp} If assignment graph $G$ is a
 % tree so that the corresponding factor graph is a tree as well, then
We note that if the factor graph is a {\it tree}, i.e., having no loop, 
then it is not hard to check that {\fullBayIter} calculates
the exact value of the marginal posterior of multiple variables $\sigma^2_{M_i}$ 
%as the typical BP does that of single variable $\sigma^2_{u}$  \cite{pearl1982}
 since 
 \begin{align*} {
  \Pr [\sigma^2_{M_i} \mid A ]
 \propto \mathcal{C}_i (A_i, \sigma^2_{M_i})  
\prod_{u \in M_i} \Pr[\sigma^2_u \mid A_{-i}]}
\end{align*}
where $A_{-i} := A \setminus A_i$.
%\begin{align*}
%\Pr [\sigma^2_{M_i} \mid A ]
%%&= \frac{f_{A_i} [y_i \mid \sigma^2_{M_i}]}{f_{A_i} [y_i \mid A_{-i} = y_{-i}]} \cdot \prod_{u \in M_i} \Pr[\sigma^2_u \mid A_{-i} = y_{-i}] \\
%& \propto \mathcal{C}_i (A_i, \sigma^2_{M_i})  
%\prod_{u \in M_i} \Pr[\sigma^2_u \mid A_{-i}] 
%\end{align*}
% 
%
 More formally, if the assignment graph $G$ is a tree from task $i$ with depth $2k$,
 then we have $b^{t}_{i} (\sigma^2_{M_i}) = \Pr[\sigma^2_{M_i}  \mid A]$  for all $t \ge k$.
%\begin{align} \label{eq:prop:tree-bp}
%%\text{if $G$ is tree,} ~
%b^{t}_{i} (\sigma^2_{M_i}) 
%~=~ 
%\Pr[\sigma^2_{M_i}  \mid A] \quad \text{\it for all $t \ge k$} \; .
%\end{align}
However, for general graphs with loops, the typical BP has no
guarantee on neither the approximation error
nor the convergence of BP
while it has been successfully applied to many applications \cite{murphy1999loopy, yanover2006linear}.
Perhaps surprisingly, we can analytically explain such empirical success for crowdsourced regression
with strong guarantees in the following section.

\subsection{Quantitative Performance Guarantee}
\label{sec:bp_quanti}
We first present a performance guarantee of 
{\bayIter} estimator
that is close to that of an {\em oracle} estimator.
The proof is in Section~\ref{sec:quantify_pf}.
\begin{theorem}\label{thm:quantify}
  Consider the crowdsourced regression model
  with $\mathcal{S} = \{\sigma^2_1, ..., \sigma^2_S\}$
  and a random $(\ell, r)$-regular graph $G$ consisting of $n$ tasks and 
$(\ell/r)n$ workers.
For given $\varepsilon, \sigma^2_{\min}, \sigma^2_{\max} > 0$ and $\ell \ge 2$, 
if  \textnormal{(i)}
$|\sigma^2_s - \sigma^2_{s'} | > \varepsilon $ 
and $\sigma^2_{\min} \le \sigma^2_s \le \sigma^2_{\max}$ for all $1 \le s \neq s' \le S$, 
and \textnormal{(ii)} $2 \le r, k \le \log \log n$,
 then for sufficiently large $n $, 
{\bayIter} in \eqref{eq:BP-estimator} with $k$ iterations
 achieves  \begin{subequations} \begin{align}
 \EXP \bigg[\frac{1}{n} \sum_{i \in V} 
 \textnormal{MSE} (\hat{\mu}_i^{{\MF \bayIter}(k)}(A)) 
 \bigg]  
&\le  \frac{d}{n}  \sum_{i \in V}  \EXP \left[  
\bar{\sigma}^2_i \left(\sigma^2_{M_i} \right) \right] 
  \label{eq:quantify_a} \\
&\quad  + 
\mathcal{E}_{\ell, \mathcal{S}}
\ell^{1/4}
 \Big(
 {4 \exp\Big({-\frac{\varepsilon^2 r}{8(8\varepsilon+ 1) \sigma^2_{\max}} } \Big) + 2^{-k}}
 \Big)^{1/4} \!\! \label{eq:quantify_b}%
\end{align} \label{eq:quantify}%
\end{subequations} where
$\mathcal{E}_{\ell, \mathcal{S}} := 
2d 
{(\frac{1}{\tau^2} + \ell  \frac{{\sigma}^2_{\max}}{\sigma^4_{\min}} )}
  {
  (\frac{1}{\tau^2} + \frac{\ell}{\sigma^2_{\min}} )^{-2}
  }
$ and
 the expectation is taken w.r.t. $G$ and $A$.
%\begin{align*} % \label{eq:epsilon_def}
%\mathcal{E}_{\ell, \mathcal{S}} := 
%%\sqrt{ d(2+d) }
%2d
% \cdot 
%{\left(\frac{1}{\tau^2} + \ell \frac{{\sigma}^2_{\max}}{\sigma^4_{\min}} \right)}
%  {
%  \left(\frac{1}{\tau^2} +\ell \frac{1}{\sigma^2_{\min}} \right)^{-2}
%  }
% \;.
%\end{align*}
%where $\sigma^2_{\max} := \max_{\sigma^2_{s} \in \mathcal{S}} \sigma^2_s.$
\end{theorem}
We provide three interpretations of Theorem~\ref{thm:quantify}. 
First, consider an oracle estimator that knows the 
hidden variances $\sigma^2_u$'s 
and makes optimal inference as  
$\hat{\mu}_i^{\MF ora}(A, \sigma^2) := \EXP[\mu_i \mid A, \sigma^2] = \bar{\mu}_i\left(A_i, \sigma^2_{M_i} \right)$.
%\begin{align}
%\label{eq:oracle}
%\hat{\mu}_i^{\MF ora}(A, \sigma^2) := \EXP[\mu_i \mid A, \sigma^2] = \bar{\mu}_i\left(A_i, \sigma^2_{M_i} \right) \;.
%\end{align}
This gives the MSE of $\hat{\mu}_i^{\MF ora}(A, \sigma^2)$: % for any given graph $G$: 
\begin{align*} 
 \EXP\bigg[ \frac{1}{n} \sum_{i \in V} \text{MSE}(\hat{\mu}_i^{\MF{ora}}(A, \sigma^2))  \bigg]
= \frac{d}{n} \sum_{i \in V} \EXP \left[  \bar{\sigma}^2_i \left(\sigma^2_{M_i} \right)  \right] \;.
\end{align*} 
Note that the oracle estimator $\hat{\mu}^{\MF ora}$ always 
outperforms even the optimal estimator $\hat{\mu}^{*}$ in \eqref{eq:MAP}, 
%due to the free access to the hidden workers' variances.
% , since the
% oracle estimator exploits optimally the free access to the hidden
% workers' variances which the optimal estimator doesn't have.
providing a lower bound on the MSE of any estimator.  
This coincides with \eqref{eq:quantify_a} in our bound, % \eqref{eq:quantify}, 
implying that the gap \eqref{eq:quantify_b} to the oracle performance \eqref{eq:quantify_a} quantifies the {\em difficulty} in identifying reliable workers.
We stress that considering a weaker oracle that captures the difficulty in estimating worker reliability,
should give a tighter lower bound than \eqref{eq:quantify_a}. % the term $(b)$. 
This is stated precisely in the following section (see Theorem~\ref{thm:optimality}). 
% corresponds to the
%performance of the oracle estimator, because using \eqref{eq:oracle},

%where the expectation is taken with respect to the distribution of
%$A$ and $\sigma^2$. 
Second, for sufficiently large $n$, when the number $r$ of per-worker tasks
and the total iterations $k$ grow with $n$, 
the performance of {\bayIter} quickly approaches  
that of the oracle estimator, %which is the fundamental limit, 
as \eqref{eq:quantify_b} vanishes exponentially. 
%, i.e.,  
%\begin{align*}
%  \lim_{n \to \infty}    
%  &   \EXP  \left[\frac{1}{n} \sum_{i \in V} 
%    \textnormal{MSE} (\hat{\mu}_i^{{\MF \bayIter}(k)}(A)) 
%    \right] 
%     = 
%    \frac{d}{n}  \sum_{i \in V}  \EXP  \left[  
%    \bar{\sigma}^2_i \left(\sigma^2_{M_i} \right) \right]
%    \;.
%\end{align*}
This is because under $(\ell,r)$-regular task assignment, for increasing
$r$ with the total number of tasks $n$, {\fullBayIter} accurately infers
all workers' variances and thus optimally estimates the true positions
$\mu.$ Note that the above performance limit holds for any
$r = \omega (1),$ implying that a reasonable number of tasks per worker
is enough to achieve a performance close to the oracle bound. % optimality of {\fullBayIter}.
Third, we compare {\bayIter} with simple averaging, i.e.,
$\hat{\mu}_i^{\MF avg}(A):=  \sum_{u \in M_i} A_{iu}/{|M_i|},$ which achieves
%to quantify the performance gap between BP and a vanilla approach. 
%It is not hard to see $\hat{\mu}_i^{\MF avg}(A)$ performs as in:
\begin{align*}
\EXP  \bigg[ \frac{1}{n} \sum_{i \in V} \text{MSE}(\hat{\mu}_i^{\MF{avg}}(A))  \bigg]
 = 
\frac{d}{n} \sum_{i \in V} 
 \EXP  \bigg[\frac{\sum_{u \in M_i} \sigma^2_u}{|M_i|^2}\bigg] \;.
\end{align*} 
%where the expectation is taken with respect to the distribution of $A$.
Note that $\EXP[\text{MSE}(\hat{\mu}^{\MF avg}_i(A))]$ increases proportionally to the
arithmetic mean of variances of workers assigned to each task, while
%$\text{MSE}(\hat{\mu}^{\MF ora}_i(A))$ or 
$\EXP[\text{MSE}(\hat{\mu}^{{\MF \bayIter}(k)}_i(A))]$ 
is proportional to the harmonic
mean of variances of workers and prior, i.e.,
$\mathbb{E} [\text{MSE}(\hat{\mu}_i^{\MF{avg}}(A)) ] \ge \mathbb{E}
[\text{MSE}(\hat{\mu}_i^{\MF{\bayIter}(k)}(A)) ].$
This gap can be made arbitrarily large by 
increasing the difference between the maximum and minimum variances of workers.
%Since the gap between the arithmetic and harmonic means can be
%made arbitrarily large by involving a single worker with low variance 
For example, if a single worker $u \in M_i$ assigned to task $i$
has high accuracy, i.e., $\sigma^2_u \simeq 0$, and the others' variances are $x$'s,
 then 
$\EXP[\text{MSE}(\hat{\mu}_i^{\MF{avg}}(A)) ] \simeq ({d}/{|M_i|})   x$ but 
$\EXP[\text{MSE}(\hat{\mu}_i^{\MF{\bayIter}(k)}(A)) ] \simeq 0.$
%In other words, BP estimator is tolerant of baleful worker while 
%the simple average is not.
Hence, the existence of a single worker with high precision in each task can reduce MSE significantly.
Our estimator iteratively refines its belief and identifies those good workers, when $r$ is sufficiently large. 

%However, since we know neither the true positions nor the workers' variances at the beginning,
%identify good workers is hard, in particular, when the number of tasks assigned to a worker is small.

\subsection{Relative Performance Guarantee}
\label{sec:bp_opt}
% \setlength{\abovedisplayskip}{7pt}%
 %\setlength{\belowdisplayskip}{5pt}%

%As the comparisons to the oracle estimator becomes loose for small 

We present the relative performance of {\bayIter} by comparing to the optimal estimator,
in particular, when the quantitative guarantee in Theorem~\ref{thm:quantify} is not tight, i.e., $r$ is small and thus
estimating reliability is difficult. 
%The proof is in Section~\ref{sec:proof-optimality}.
\begin{theorem}\label{thm:optimality} 
  Consider the crowdsourced regression model
  with $\mathcal{S} = \{\sigma^2_{\min}, \sigma^2_{\max}\}$
  and a random $(\ell, r)$-regular graph $G$ consisting of $n$ tasks and 
  $(\ell/r)n$ workers.
  For given $\varepsilon > 0$ and $\ell$, 
  there exists a constant $C_{\ell, \varepsilon}$, depending on only $\ell$ and $\varepsilon$, such that if 
  \textnormal{(i)} $\sigma^2_{\min}+ \varepsilon   \le \sigma^2_{\max} \le 2\sigma^2_{\min} $,
  and
  \textnormal{(ii)}
  $C_{\ell, \varepsilon} \le r \le \log \log n$,
  then 
  {\bayIter} in \eqref{eq:BP-estimator} with $k= \log \log n$ iterations achieves
\begin{align}
\!\!\!\! \EXP \bigg[\frac{1}{n} \sum_{i \in V} \Big( \textnormal{MSE} (\hat{\mu}^*_i(A)) - 
    \textnormal{MSE} (\hat{\mu}_i^{{\MF \bayIter}(k)}(A))  \Big)
    \bigg]  \to 0
    % \to 0  \quad \text{as} \quad n \to \infty \;, 
    \label{eq:main_relative} 
\end{align}
%  \begin{align*}
%       \EXP   \left[\frac{1}{n} \sum_{i \in V} \left| \textnormal{MSE} (\hat{\mu}^*_i(A)) - 
%      \textnormal{MSE} (\hat{\mu}_i^{{\MF \bayIter}(k)}(A))  \right|
%	\right]  
%       ~\le~ %C_{\mathcal{S}, \ell} \cdot 2^{-k}
%       \mathcal{E}_{\ell, \mathcal{S}}  \cdot  2^{-k} 
%       \cdot 6\ell \cdot S^{2\ell} \cdot  
%      \left(2  + \frac{d\sigma^6_{\max}}{\ell^3 \sigma^2_{\min}} \right)
%  \end{align*} 
  %  \textnormal{(iii)} $k \le \log \log n$, then for sufficiently large $n$,
as $n \to \infty$. The expectation here is taken w.r.t. the distribution of $G$ and $A$.
%  and $\mathcal{E}_{\ell, \mathcal{S}}$ is defined  in \eqref{eq:epsilon_def}.
\end{theorem}
\iffalse
As a corollary, it follows that  when we set $k$ increasing with $n$,
e.g., $ k = \log \log n$, we have an asymptotic optimality of {\fullBayIter}: 
\begin{align}
  \lim_{n \to \infty}    
  &   \EXP \left[\frac{1}{n} \sum_{i \in V} \left| \textnormal{MSE} (\hat{\mu}^*_i(A)) - 
    \textnormal{MSE} (\hat{\mu}_i^{{\MF BP}(k)}(A))  \right|
    \right]     = 0 \;.
    \label{eq:main_relative}
\end{align}
\fi
This result is not directly comparable to Theorem \ref{thm:quantify} 
as it applies to different regimes of the parameters. 
%In particular, %in Theorem \ref{thm:quantify}, 
The {\em oracle} optimality gap 
\eqref{eq:quantify_b} does not vanish for finite $\ell$ and $r$.
This is perhaps because the 
oracle is too strong to compete against when $\ell$ and $r$ are small. 
Hence, to obtain the tight result in \eqref{eq:main_relative}, we construct a more practical lower bound on the 
optimal estimator in \eqref{eq:marginal_sum} that takes account of the worker reliability estimation.
%Such a comparison can be made rigorous by constructing the following lower bound on the fundamental limit. 
%To analyze the relative performance of BP estimator to the optimal estimator,
We use the fact that the random $(\ell, r)$-regular bipartite
graph has a {\it locally tree-like structure} with depth
$k \le \log \log n$   and our message update is exact on the local
tree \cite{pearl1982}. 
By revealing the ground truths at the boundary of this local tree of depth $k$,
we construct a {\em weaker} oracle estimator that gives a tighter lower bound. 
%Directly analyzing the performance of such a weaker oracle is 
%hard. Instead, 
We show that the gap between our estimator 
(without the ground truths at the boundary) 
and the weaker oracle vanishes as the tree depth increases. 
This is made clear by establishing {\it decaying correlation} 
from the information on the outside of the
local tree to the root. 
A formal proof of Theorem~\ref{thm:optimality} is presented in Section~\ref{sec:proof-optimality}.

%is fast enough to guarantee the convergence of
%MSE's in mean, while MSE can be arbitrarily large % infinite
%in some random space. 

For the analytic tractability, we need a constant lower bound of 
$r \ge C_{\ell, \varepsilon}$ and $|\mathcal{S}| = 2$.
Similar conditions are also required in other BP analysis \cite{ok2016icml, mossel2014}, while
ours is more general in terms of $\ell$, i.e., factor degree since the other analysis made on only factor degree $2$
but also more challenging due to the unboundedness of the regression error.
%In addition to $|\mathcal{S}| = 2$, 
We also assume $\sigma^2_{\min}+ \varepsilon   \le \sigma^2_{\max} \le 2 \sigma^2_{\min}$. However, this is the most challenging regime for any inference algorithms
 since it is hard to distinguish the workers' variances.  
Note that when this assumption is violated, i.e., $\sigma_{\min} \ll \sigma_{\max}$, Theorem~\ref{thm:quantify} provides the near-optimality of {\bayIter} since the MSE gap between BI and Oracle vanishes as the variance gap increases. 
%we believe that Theorem \ref{thm:quantify} still holds without it. % the conditions.
%  Although proving the optimality for more general $\mathcal{S}$
  %requires new analysis techniques, we believe that Theorem \ref{thm:quantify} holds without the assumptions.
%  , beyond those we develop in this paper,
 The experimental results in Section~\ref{sec:exp} indeed suggest the {\bayIter}'s optimality even when such assumptions are violated.
\section{Proofs of Theorems}
%\vspace{-0.05cm}

\subsection{Proof of Theorem~\ref{thm:quantify}}
\label{sec:quantify_pf}

% \setlength{\abovedisplayskip}{4pt}%
% \setlength{\belowdisplayskip}{4pt}%
% similarly to the bound in \eqref{eq:MSE_bdd},
%Pick an arbitrary $\rho \in W$. 
We start with an upper bound 
on the conditional expectation of MSE of $\hat{\mu}^{{\MF {\bayIter}} (k)}_i (A)$ 
conditioned on $\sigma^2 = \tilde{\sigma}^2 \in \mathcal{S}^{W}$. 
Let $ \EXP_{\tilde{\sigma}^2}$ be the conditional expectation given 
$\sigma^2 = \tilde{\sigma}^2$. % \in \mathcal{S}^{W}$.
Using Cauchy-Schwarz inequality for random variables 
 $X$ and $Y$, i.e., $|\EXP[XY]| \le \sqrt{\EXP[X^2] \EXP[Y^2]}$,
%and some calculus,
it is not hard to obtain that (see the supplementary material for the detailed derivation)
\begin{align}
  \EXP_{\tilde{\sigma}^2}\!  \left[ \| \hat{\mu}^{{\MF {\bayIter}} (k)}_i (A)  - \mu_i \|^2_2  \right] \le 
d  \bar{\sigma}^2_i (\tilde{\sigma}^2_{M_i}) 
+
 \mathcal{E}_{\ell, \mathcal{S}} 
\left(
1- 
\EXP_{\tilde{\sigma}^2} 
\left[  b^k_i (\tilde{\sigma}^2_{M_i})    \right]
\right)^{1/4} \;.
\label{eq:MSE_BP_bdd}  
\end{align}
%whose derivation is %for \eqref{eq:MSE_BP_bdd} 
%in the supplementary material. % Appendix~\ref{sec:MSE_BP_bdd_pf}.
%\ifpaper %
%in our supplementary material due to the space limitation.
%\else %
%in Appendix~\ref{sec:MSE_BP_bdd_pf}.
%\fi %
%\note{up to here}
%
% and 
%We further obtain an upper bound of the term in the last summation 
%in \eqref{eq:MSE_BP_bdd} as follows:
%\begin{align}
%asdf
%\end{align}
%Noting that $\prod_{u \in M_{i}} b^k_u (\sigma'^2_u)  \le b^k_u $ 

To complete the proof, we will obtain an upper bound of the last term in the RHS of \eqref{eq:MSE_BP_bdd}
using the known fact that a random $(\ell, r)$-regular bipartite graph $G$ is
 a locally tree-like.
Pick an arbitrary task $\tau \in V$. 
Let $G_{\tau, 2k+1}=(V_{\tau, 2k+1}, W_{\tau, 2k+1}, E_{\tau, 2k+1}) $ denote
the subgraph of $G$ induced by all the nodes within (graph) distance $2k+1$ from
{\it root} $\tau$. From Lemma~5 in \cite{kos2014},
we have that for sufficiently large $n$,
\begin{align}  
\Pr[\text{$G_{\tau, 2k+1}$ is not tree}] 
%\le \frac{3 \ell r }{n} ((\ell - 1)( r-1))^{2k+1}  
\le \frac{3 (\ell r)^{2k+2} }{n} 
 \le 2^{-k}
 \label{eq:tree-probability}  
\end{align}
where the last inequality follows from 
the choice of $r, k \le \log \log n$ and large $n$. % and constant $\ell$.
Thus, we obtain that
\begin{align} 
 \EXP \big[ 1- \EXP_{\tilde{\sigma}^2} 
[b^k_\tau (\tilde{\sigma}^2_{M_\tau}) ] \big] 
\le 
\EXP 
\big[ 
1- \EXP_{\tilde{\sigma}^2} 
[ 
 b^k_{\tau}
 (\tilde{\sigma}^2_{M_{\tau}})  
 \mid 
 \text{$G_{\tau, 2k+1}$ is a tree}
]
\big] 
+  2^{-k}    \label{eq:b_bdd} 
\end{align}
where $\EXP $ is taken w.r.t. $G$ and $\sigma^2$.
%Recalling the exactness of {\bayIter} in tree, it follows that if $G_{\tau, 2k+1}$ is a tree,
%\begin{align}  \label{eq:bp-rho-2k}
%b^k_{\tau}(\sigma'^2_{M_{\tau}}) 
% =  \Pr[\sigma^2_{M_{\tau}} = \sigma'^2_{M_{\tau}} \mid A_{\tau, 2k+1}] \; .
%\end{align}
%From the choice of $r, k \le \log \log n$ and constant $\ell$, 
%we note that for sufficiently large $n$,
%\begin{align}
%\frac{3 \ell r }{n} ((\ell - 1)( r-1))^{2k+1}  ~\le~ 2^{-k} \; . \label{eq:2k_bound}
%\end{align}

Let $A_{\tau, 2k+1} := \{A_{iu} : (i,u) \in E_{\tau, 2k+1}\}$.
%and $\EXP_{\tilde{\sigma}^2}$ 
%is taken with respect to $A_{\tau, 2k+1}$ given $\sigma^2 = \tilde{\sigma}^2$
The exactness of {\bayIter} on tree implies that if $G_{\tau, 2k+1}$ is tree,
$b^k_{\tau}(\sigma'^2_{M_{\tau}}) $ is the likelihood of $\sigma^2_{M_{\tau}} = \sigma'^2_{M_{\tau}}$ given $A_{\tau, 2k+1}$ and thus
\begin{align}  \label{eq:bp-rho-2k}
\!\!\!\!\! 1- \EXP_{\tilde{\sigma}^2} [ b^k_{\tau}(\tilde{\sigma}^2_{M_{\tau}}) 
 ] 
& = \EXP_{\tilde{\sigma}^2}[
  \Pr[\sigma^2_{M_{\tau}} \neq \tilde{\sigma}^2_{M_{\tau}} \mid A_{\tau, 2k+1}] 
  ]  \nonumber \\
&  \le
\!\!
\sum_{u \in M_\tau}\EXP_{\tilde{\sigma}^2}
[\Pr[\sigma^2_{u} \neq \tilde{\sigma}^2_{u} \mid A_{\tau, 2k+1}] ] \!\!\!
\end{align}
where the inequality is due to the union bound.
Hence, it suffices to show that if $G_{\tau, 2k+1}$ is tree,
for any $u \in M_\tau$, the marginal probability of $\sigma^2_{u}$
concentrated at $\tilde{\sigma}^2_u$.
\begin{lemma} \label{lem:concentration}
For given $\rho \in W$, suppose $G_{\rho, 2k} = (V_{\rho, 2k}, W_{\rho, 2k}, E_{\rho, 2k})$
  \footnote{We denote by $\tau \in V$ and $\rho \in W$ task and worker roots.}
  is a $(\ell, r)$-regular bipartite graph 
%  $G = (V, W, E)$
with $\ell \ge 2$ and $r \ge 1$
  and it is a tree rooted from worker $\rho$
  with depth $2k \ge 2$. 
For given  $\varepsilon, \sigma^2_{\min}, \sigma^2_{\max}  >0$, consider 
$\mathcal{S} = \{\sigma^2_1, ..., \sigma^2_S\}$ such that
  \textnormal{(i)}
$|\sigma^2_s - \sigma^2_{s'} | > \varepsilon $ 
and $\sigma^2_{\min} \le \sigma^2_s \le \sigma^2_{\max}$ for all $1 \le s \neq s' \le S$.
 Then, % and $u \in M_\rho$
\begin{align*}  %\label{eq:concentration}
\EXP \left[ \EXP_{\tilde{\sigma}^2 } \left[
\Pr \left[ \sigma^2_\rho \neq \tilde{\sigma}^2_\rho \mid A_{\rho, 2k} 
\right]
\right] \right]
 \le 
 %4  \exp\left( -\frac{\varepsilon^2 r}{8(8\varepsilon+ 1) \sigma^2_{\max}}\right) 
 4  e^{ -\frac{\varepsilon^2 r}{8(8\varepsilon+ 1) \sigma^2_{\max}} }
% \EXP \left[ b^k_\rho(\sigma^2_\rho) 
%\Pr\left[ \sigma^2_\rho \mid A_{\rho, 2k}\right]
%\longmid \sigma^2_\rho
%\right]
% \ge 1- \exp(-r)
\end{align*}
 where the inner expectation $\EXP_{\tilde{\sigma}^2}$ is taken w.r.t. $A_{\rho, 2k}$ from the crowdsourced regression model
  given $\sigma^2 = \tilde{\sigma}^2 \in \mathcal{S}^{W}$,
  and the outer expectation $\EXP$ is taken w.r.t. $\tilde{\sigma}^2 \in \mathcal{S}^{W}$ drawn uniformly at random.
\end{lemma}

The proof of Lemma~\ref{lem:concentration} is  in the supplementary material. % Appendix~\ref{sec:concentration-pf}.
Combining 
\eqref{eq:b_bdd},  \eqref{eq:bp-rho-2k}  and Lemma~\ref{lem:concentration} leads to
%it follows that
\begin{align*}
\EXP \left[
( 1- \EXP_{\tilde{\sigma}^2}[b^k_\tau (\tilde{\sigma}^2_{M_\tau})])^{1/4} \right]
& \le 
\left(1- 
\EXP [ \EXP_{\tilde{\sigma}^2}[
b^k_\tau (\tilde{\sigma}^2_{M_\tau})]] \right)^{\frac{1}{4}} \\
&\le \Big( 4 \ell  e^{ -\frac{\varepsilon^2 r}{8(8\varepsilon+ 1) \sigma^2_{\max}}} + 2^{-k} \Big)^{\frac{1}{4}}
\end{align*}
where the first inequality is from the fact that 
  $(1-x)^{1/4}$ is concave, i.e.,  
$\EXP [(1-X)^{1/4}] \le (1-\EXP[X])^{1/4}$.
This completes the proof of Theorem~\ref{thm:quantify}
with \eqref{eq:MSE_BP_bdd} because of the arbitrary choice of root task $\tau \in V$.

\subsection{Proof of Theorem~\ref{thm:optimality}}
\label{sec:proof-optimality}

%We s
%on the conditional expectation of MSE of $\hat{\mu}^{{\MF BP} (k)}_i (A)$ 
%conditioned on $\sigma^2 = \tilde{\sigma}^2 \in \mathcal{S}^{W}$. 
%Let $ \EXP_{\tilde{\sigma}^2}$ be the conditional expectation given 
%$\sigma^2 = \tilde{\sigma}^2$. 
%
Pick an arbitrary task $\tau \in V$.
Recalling the exactness of BI on tree, it is clear that 
in the case of $\ell = 1$, 
 $\hat{\mu}_\tau^{{\MF \bayIter}{(k)}}(A) $
 is identical to the optimal estimator $\hat{\mu}_\tau^*(A)$.
%since $G$ is the set of disjoint {one-level} trees each of which root
%corresponds to a worker
%and we have
% $b^k_i (\sigma^2_{M_\tau}) = \Pr[\sigma^2_{M_i} \mid A]$.
We so focus on $\ell \ge 2$.
Recall that the Bayesian optimal estimator
$\hat{\mu}_\tau^*(A)$ minimizes the MSE given $A$. However, its analysis is very challenging due to
loops in its corresponding graphical model. 
To overcome this issue, we use the locally tree like structure
of random $(\ell, r)$-regular bipartite graph, again.
Intuitively, we will first construct an artificial but (analytically) 
tractable estimator outperforming $\hat{\mu}_\tau^*(A)$ in terms of MSE
and then we will show
the diminishing gap between MSE's of {\bayIter} and the constructed estimator.

% and extract subgraph $G_{\tau, 2k+1} = (V_{\tau, 2k+1}, W_{\tau, 2k+1}, E_{\tau, 2k+1})$ from $G$.
Let $\partial W_{\tau, 2k+1}$ be the set of all workers at distance $2k+1$ from root $\tau$ in subgraph $G_{\tau, 2k+1}$.
Consider an oracle estimator  $\hat{\mu}_\tau^{{\MF ora}{(k)}}(A)$ of $\mu_\tau$ with 
 free access to true variances of leaf-workers ${\partial W_{\tau, 2k+1}}$, formally defined as
\begin{align*}
\hat{\mu}_\tau^{{\MF ora}{(k)}}(A) &:= \!\sum_{\sigma^2_{M_\tau} \in \mathcal{S}^{M_\tau}} \!\!
 \bar{\mu}_\tau \!
 (A_\tau, \sigma^2_{M_\tau} )  \Pr [\sigma^2_{M_\tau}\! \mid A, \sigma^2_{\partial W_{\tau, 2k+1}}]  \\
 & = \sum_{\sigma^2_{M_\tau} \in \mathcal{S}^{M_\tau}} 
 \bar{\mu}_i\left(A_\tau, \sigma^2_{M_\tau} \right)  \Pr [\sigma^2_{M_\tau} \mid A_{\tau, 2k+1}, \sigma^2_{\partial W_{\tau, 2k+1}}] 
\end{align*}
where for the last equality, we use the conditional independence between $A_{\tau, 2k+1}$ 
and $A \setminus A_{\tau, 2k+1}$ given the additional information $\sigma^2_{\partial W_{\tau, 2k+1}}$.
Using the equality
 $(
 \hat{\mu}^*_\tau(A)   - \mu_\tau) = ( \hat{\mu}^*_\tau(A)  - \EXP[\mu_\tau \mid A, \sigma^2_{\partial  W_{\tau, 2k+1}}] ) +
( \EXP[\mu_\tau \mid A, \sigma^2_{\partial W_{\tau, 2k+1}}] - \mu_\tau )$,
it is not hard to check that 
$\hat{\mu}_\tau^{{\MF ora}{(k)}}$
has smaller expected MSE than $\hat{\mu}^*_\tau(A)$, i.e.,
$
\EXP [  \text{MSE} (\hat{\mu}^{{\MF {ora}} (k)}_\tau (A)) ] 
\le \EXP [\text{MSE} (\hat{\mu}^*_\tau (A))] 
 \le \EXP [\text{MSE} (\hat{\mu}^{{\MF {\bayIter}} (k)}_\tau (A))]$.
Thus, it is enough to show
that as $n \to \infty$,
\begin{align}
&\EXP\left[ \left|  \text{MSE} (\hat{\mu}^{{\MF {ora}} (k)}_\tau (A)) - \text{MSE} (\hat{\mu}^{{\MF {\bayIter}} (k)}_\tau (A))  \right| \right]  \to 0 \;. \label{eq:ora_BP}
\end{align}
Since the only difference between $\hat{\mu}^{{\MF {ora}} (k)}_\tau (A) $ and $\hat{\mu}^{{\MF {\bayIter}} (k)}_\tau (A)$ is the estimation on $\sigma^2_{M_\tau}$, i.e., {\bayIter} uses $b^k_\tau(\sigma^2_{M_\tau})$ instead of $\Pr[\sigma^2_{M_\tau} \mid A, \sigma^2_{\partial W_{\tau, 2k+1}}]$.
%Although 
%the gap between MSE's of $\mu_i^*(A)$ and $\mu_i^{\MF{BP}(k)}(A)$
%can be arbitrary large with positive measure,
Using Cauchy-Schwarz inequality and some calculus, similarly as \eqref{eq:MSE_BP_bdd},
we derive an upper bound on the expected difference
between MSE's of $\mu_\tau^{\MF{{ora}}(k)}(A)$ and $\mu_\tau^{\MF{{\bayIter}}(k)}(A)$
as follows:
\begin{align}
\EXP\left[  \big| \text{MSE} (\hat{\mu}^{{\MF ora}(k)}_\tau (A)) - \text{MSE} (\hat{\mu}^{{\MF {\bayIter}} (k)}_\tau (A))  \big| \right] 
 \le
\mathcal{E}_{\ell, \mathcal{S}} 
\sum_{ \sigma'^2_{M_\tau}, \sigma''^2_{M_\tau} \in \mathcal{S}^\ell}
\sqrt{
 \EXP \Big[ \big(D_{\tau, k}(\sigma'^2_{M_\tau}, \sigma''^2_{M_\tau}) \big)^2 \Big]
}
 \label{eq:MSE_diff_bound}
\end{align}
where 
$D_{\tau, k}(\sigma'^2_{M_\tau}, \sigma''^2_{M_\tau}): =
b^k_\tau (\sigma'^2_{ M_{\tau}}) 
b^k_\tau (\sigma''^2_{ M_{\tau}})
-   \Pr[\sigma^2_{M_\tau} = \sigma'^2_{ M_{\tau}} \mid A, \sigma^2_{\partial W_{\tau, 2k+1}}] 
 \Pr[\sigma^2_{M_\tau} = \sigma''^2_{ M_{\tau}} \mid A, \sigma^2_{\partial W_{\tau, 2k+1}}] 
$.
We provide the detailed steps for \eqref{eq:MSE_diff_bound} in the supplementary material. % Appendix~\ref{sec:MSE_diff_bound_pf}.
% $\rho \in W$ and tasks assigned to the task $\rho$.
Then, from the same decomposition in \eqref{eq:b_bdd},
% into two cases whether $G_{\tau, 2k+1}$ is a tree or not,
it follows that for each 
$\sigma'^2_{M_{\tau}}, \sigma''^2_{M_{\tau}}  \in \mathcal{S}^{\ell}$
and sufficiently large $n$,
\begin{align}
\EXP \left[\left(D_{\tau, k}(\sigma'^2_{M_\tau}, \sigma''^2_{M_\tau}\right)^2 \right]
~\le~ \EXP \left[D_{\tau, k}(\sigma'^2_{M_\tau}, \sigma''^2_{M_\tau} \right]
~\le~
\EXP 
\big[ 
\left| D_{\tau, k}(\sigma'^2_{M_\tau}, \sigma''^2_{M_\tau})
\right|
 \mid 
 \text{$G_{\tau, 2k+1}$ is tree}
\big] 
+ 2^{-k}
%+  \frac{3 \ell r }{n} ((\ell - 1)( r-1))^{2k+1} \; ,
\label{eq:decomp_partial}
\end{align}
where the first inequality follows from that $0 \le D_{\tau, k}(\sigma'^2_{M_\tau}, \sigma''^2_{M_\tau}) \le 1$.
% and
%let $G_{\tau, 2k}=(V_{\tau, 2k}, W_{\tau, 2k}, E_{\tau, 2k}) $ denote
%the subgraph of $G$ induced by all the nodes within (graph) distance $2k$ from root $\rho$.

Suppose $G_{\tau, 2k+1}$ is tree. Then, the graph subtracted from 
$G_{\tau, 2k+1}$ to task $\tau$
and edges between task $\tau$ and workers in $M_{\tau}$
is partitioned into $r$ sub-trees denoted by $\{ G_{\rho, 2k} = (V_{\rho, 2k}, W_{\rho, 2k}, E_{\rho, 2k}): \rho \in M_{\tau} \}$
each of which is rooted from worker $\rho \in M_\tau$
with depth $2k$ in the subtracted graph.
From the exactness of {\bayIter} on tree, it follows that
\begin{align*}
b^k_{\tau}(\sigma'^2_{M_{\tau}})  
&= \Pr[\sigma^2_{M_\tau} = \sigma'^2_{M_\tau} \mid A_{\tau, 2k+1}]  \\
&\propto \mathcal{C}_\tau (A_\tau, \sigma^2_{M_\tau}) 
 \Pr[\sigma^2_{M_\tau} = \sigma'^2_{M_\tau} \mid A_{\tau, 2k+1} \setminus A_{\tau}]  \\
& = \mathcal{C}_\tau (A_\tau, \sigma^2_{M_\tau}) 
\prod_{\rho \in M_\tau} 
\Pr[\sigma^2_{\rho} = \sigma'^2_{\rho} \mid A_{\rho, 2k}] 
\end{align*}
where
$A_{\rho, 2k} : = \{A_{iu} : (i,u) \in E_{\rho, 2k}\}$ and 
 for the last equality, we use 
the conditional independence
among $\sigma^2_{M_\tau}$ given 
$A_{\tau, 2k+1} \setminus A_{\tau}$
decomposed into $A_{\rho, 2k}$. 
Similarly, we also obtain
\begin{align*}
\Pr[\sigma^2_{M_\tau} = \sigma'^2_{ M_{\tau}} \mid A, \sigma^2_{\partial W_{\tau, 2k+1}}] 
&\propto \mathcal{C}_\tau (A_\tau, \sigma^2_{M_\tau}) 
 \Pr[\sigma^2_{M_\tau} = \sigma'^2_{M_\tau} \mid A_{\tau, 2k+1} \setminus A_{\tau}, \sigma^2_{\partial W_{\tau, 2k+1}}]  \\
& = \mathcal{C}_\tau (A_\tau, \sigma^2_{M_\tau}) 
\prod_{\rho \in M_\tau} 
\Pr[\sigma^2_{\rho} = \sigma'^2_{\rho} \mid A_{\rho, 2k}, \sigma^2_{\partial W_{\tau, 2k+1}}]  \;.
\end{align*}
Hence it is enough to show the vanishing correlation of true variances of workers at leaves 
to inferring the root worker's variance.
Formally, we provide Lemma~\ref{lem:correlation-decay} that captures 
a decreasing rate of the correlation. 
\begin{lemma} \label{lem:correlation-decay}
  Suppose $G_{\rho, 2k} = (V_{\rho, 2k}, W_{\rho, 2k}, E_{\rho, 2k})$
  is induced from $(\ell, r)$-regular bipartite graph $G = (V, W, E)$ 
%  with $\ell \ge 2$ and $r \ge 1$
  and it is a tree with depth $2k \ge 2$. 
  Let $\partial W_{\rho, 2k}$ be the set of workers at the leaves in $G_{\rho, 2k}$.
For given  $\varepsilon, \sigma^2_{\min}, \sigma^2_{\max}  >0$, consider 
$\mathcal{S} = \{\sigma^2_{\min}, \sigma^2_{\max}\}$ such that
$\sigma^2_{\min}+ \varepsilon   \le \sigma^2_{\max} \le 2\sigma^2_{\min}$. 
 Then, for any given $\tilde{\sigma}^2 \in \mathcal{S}^{W}$,
there exists a constant $C_{\ell, \varepsilon}$ such that if
  $r \ge C_{\ell, \varepsilon}$, then 
\begin{align} 
\EXP_{\tilde{\sigma}^2} \! \left[ \left|
\Pr[ \sigma^2_\rho =  \tilde{\sigma}^2_\rho  \mid A_{\rho, 2k}, \sigma^2_{\partial W_{\rho, 2k}}]  - \Pr[ \sigma^2_\rho =  \tilde{\sigma}^2_\rho \mid A_{\rho, 2k}] 
 \right| \right] 
 \le 2^{-k} \label{eq:correlation-decay}
\end{align}
 where the expectation is taken w.r.t. $A$ from the crowdsourced regression model
  given $\sigma^2 = \tilde{\sigma}^2$ and $G$.
\end{lemma}
The proof of Lemma~\ref{lem:correlation-decay} is given in the supplementary material.
This lemma completes the proof of Theorem~\ref{thm:optimality} with \eqref{eq:MSE_diff_bound} and \eqref{eq:decomp_partial}.

{
\captionsetup[subfloat]{captionskip=5.5pt}
\begin{figure*}[t!]
   %\begin{center}
%\vspace{-0.2cm}
%\begin{minipage}{.245\textwidth}
   \begin{center}
 %  \end{center}
     \subfloat[$\mathcal{S}_{\text{small}}$ with  $\ell=5$\label{fig:small_l5}]{\includegraphics[width=0.25\columnwidth]{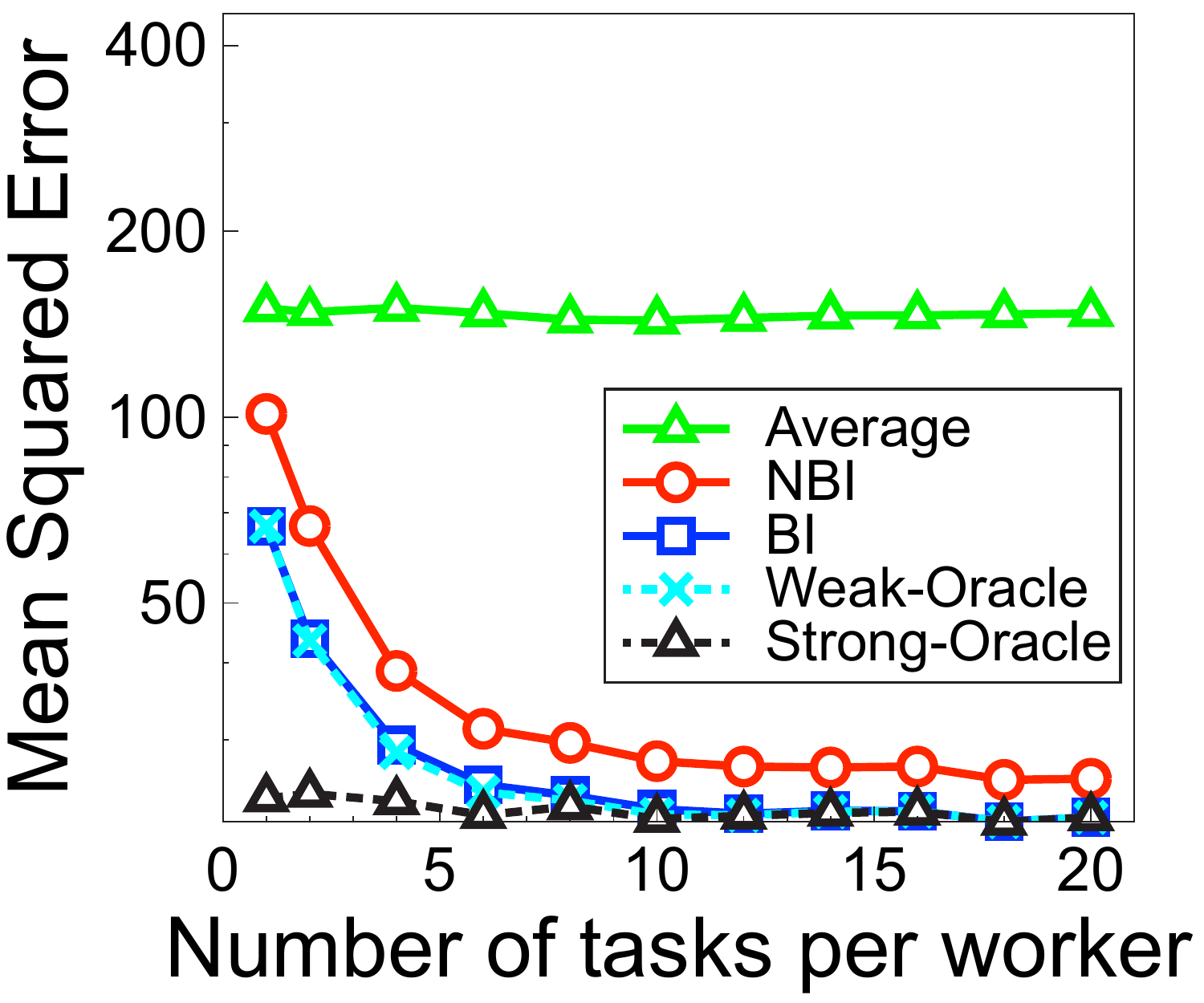}}
     \subfloat[$\mathcal{S}_{\text{large}}$ with  $\ell=5$\label{fig:large_l5}]{\includegraphics[width=0.25\columnwidth]{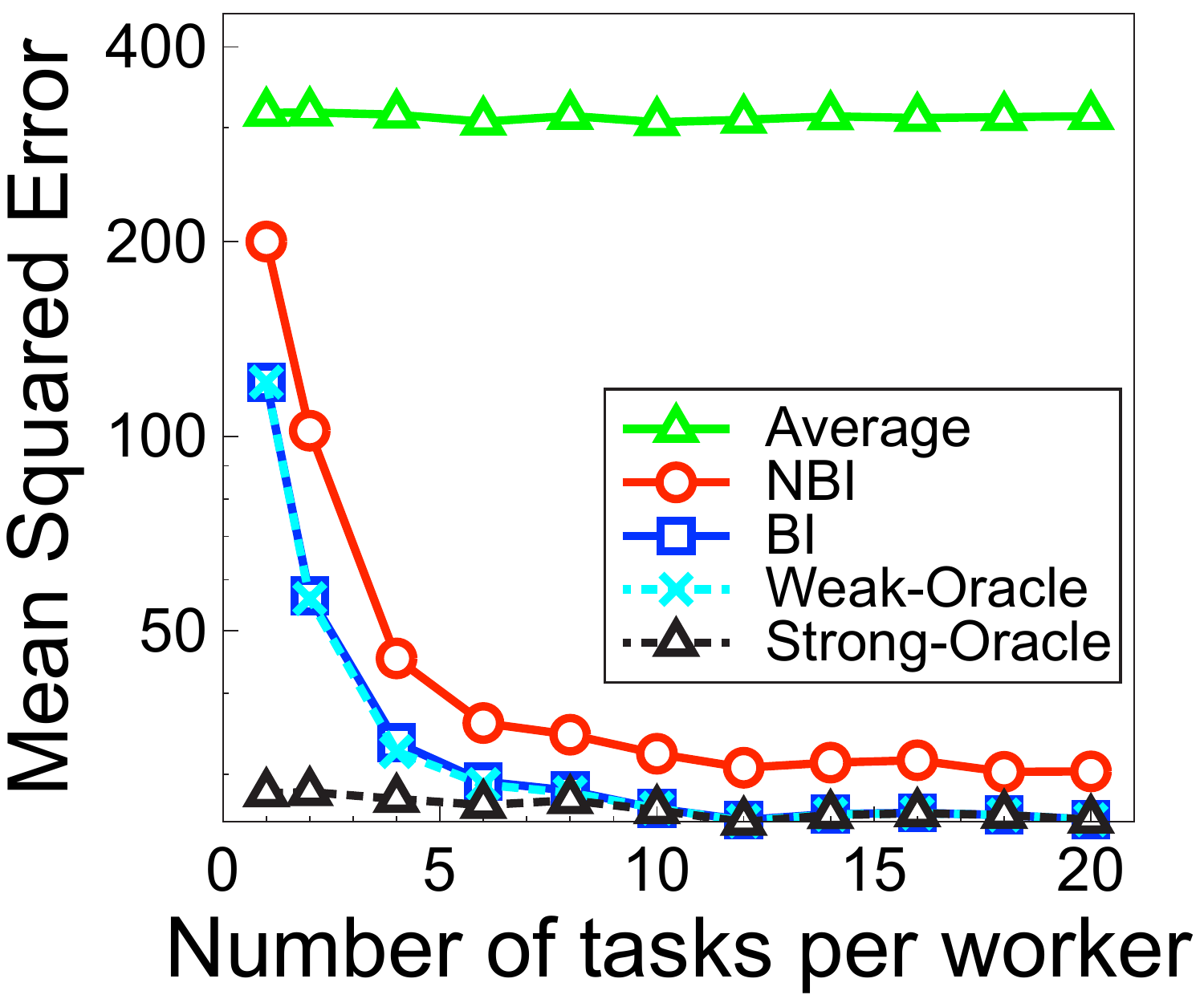}} 
%\end{minipage}
%\begin{minipage}{.245\textwidth}
     \subfloat[$\mathcal{S}_{\text{small}}$ with $r=5$\label{fig:small_r5}]{\includegraphics[width=0.25\columnwidth]{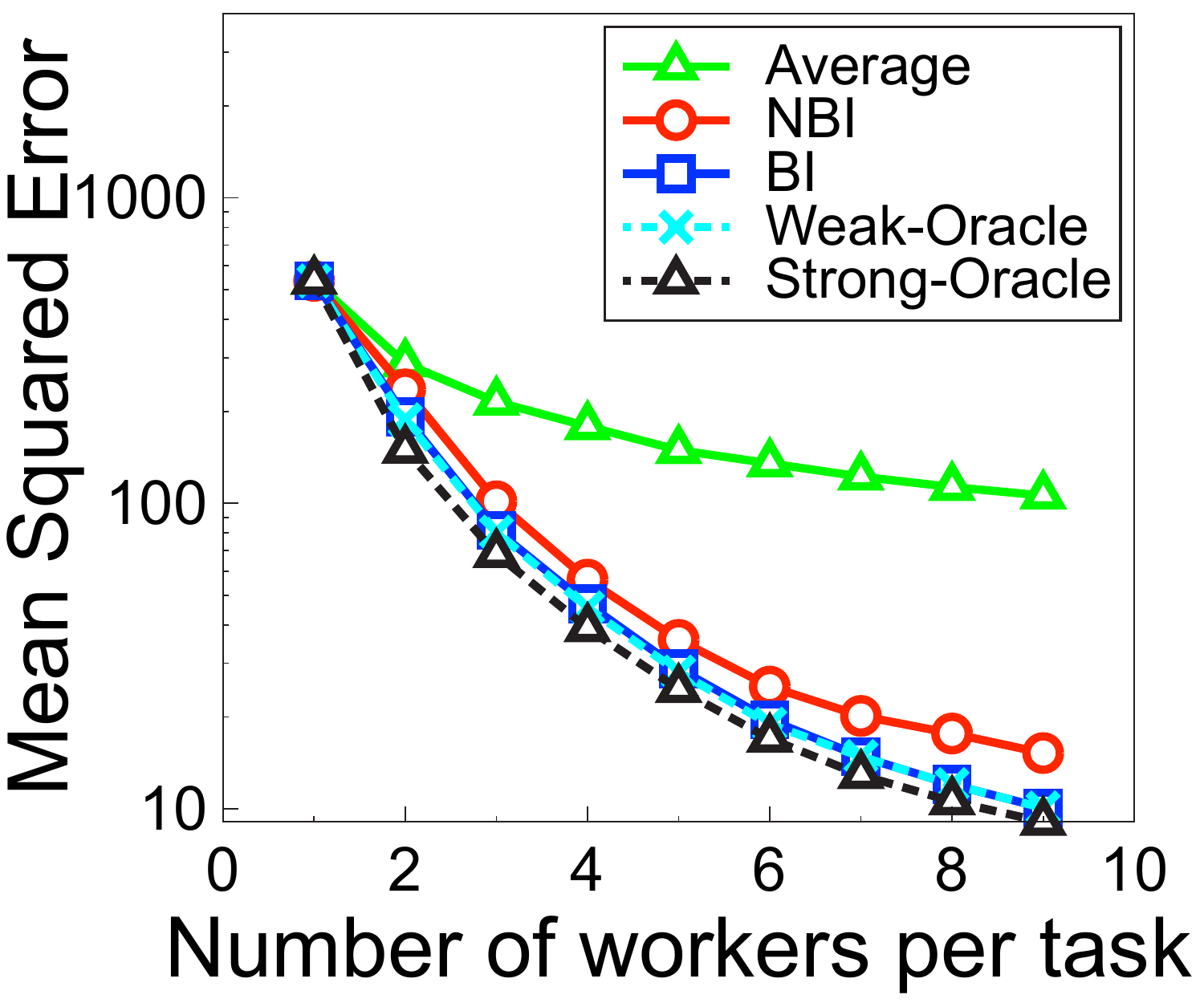}} 
          \subfloat[$\mathcal{S}_{\text{large}}$ with  $r=5$\label{fig:large_r5}]{\includegraphics[width=0.25\columnwidth]{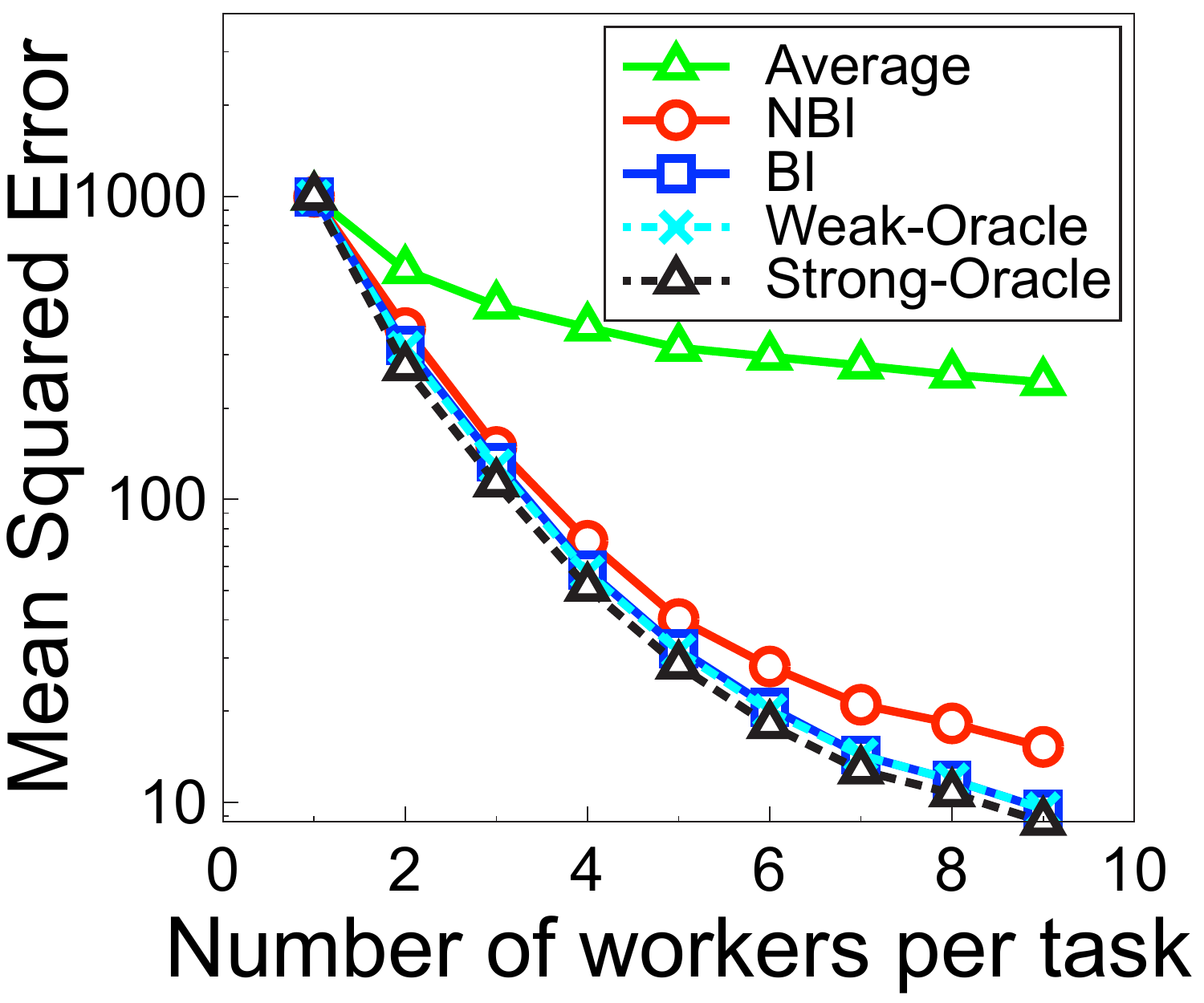}} 
  %   \end{minipage}
   \end{center}
   \vspace{-0.3cm}
   \caption{The average MSE of various algorithms on the
     synthetic datasets consisting of $200$
     tasks and workers with $\mathcal{S}_{\text{small}} \!=\! \{10,100,1000\}$ and
     $\mathcal{S}_{\text{large}} \!=\! \{10, 100, 5000\}$;
 (a)-(b)\! $\ell\!=\!5$ with varying $r$; (c)-(d)\!  $r\!=\!5$ with varying $\ell$.}
\label{fig:all}
\end{figure*}
}
%%%%%%%%%%%%%%%%%%%%%%%%%%%%%%%%%%%%%%%%%%%%%%%%%%%%%

\section{Experimental Results}
\label{sec:exp}

%In this section, we present experimental results supporting 
%our analytical findings and the superiority of 
%{\bayIter} on synthetic and real-world datasets.
%experiments  and  on simulated tasks of locating objects of interest in images. 
%\subsection{Tested Algorithms}
We experiment the following five algorithms:
%: {\MF \bayIter}, {\MF \nonIter}, {\MF Average}, and {\MF Oracle}
% as implemented 
%which are described in what follows:
\begin{compactitem}[$\circ$]
\item {\MF \bayIter} is implemented
%It is an implementation of 
%{\MF \bayIter} 
without any prior information on true positions
by taking the limit $\tau \to \infty$, i.e., it outputs 
$\mu^{{\MF \bayIter}}(A)$ in \cref{eq:ui-message,eq:iu-message,eq:belief} with
$\lim_{\tau \to \infty} \mathcal{C}_i (A_i, \sigma^2_{M_i})$. Note that our theoretical guarantees on
{\bayIter} still hold in this regime. 
%We terminate {\MF \bayIter} at the maximum of $100$ iterations or after
%checking convergence of messages. % tolerance.
% In our model, for
%mathematical rigor, we assumed that we have some information of
%the true position $\mu_i$ in advance as specified by the density of
%$\mu_i$ as the spherical Gaussian with mean $\nu_i$ and variance $\tau$.
%Since the exact knowledge of such statistical information might not be easy
%to obtain in practice, we implement {\MF BP} with no
%prior information on true positions by taking the limit of BP as $\tau \to \infty.$
%%$\lim_{\tau \to \infty} \mathcal{C}_i (A_i, \sigma^2_{M_i})$, i.e.,
%%where large $\tau$ dilutes the information.
%Instead of {\MF Oracle}, whereas it might not be tight we need to simulate a weaker
%oracle to tighten the lower bound in this regime, which is
%computationally intractable.
\vspace{0.1cm}
 \item {\MF \nonIter} is
an iterative algorithm of non-Bayesian type.
% designed without assuming
%the prior distribution.
%, hence named  (NBI).
It recursively 
updating 
workers' variances and tasks' answers based on workers' 
consensus.
%non-Bayesian algorithm
Formally, 
initialized with $\hat{\sigma}^2_{0, u} = 1$,
it estimates 
%\begin{align*}
%\mu^{{\MF con}(t)}_i(A) 
%&:= \frac{\sum_{ u \in M_i} \tfrac{1}{ \hat{\sigma}^{2}_{t, u} } A_{iu} }{\sum_{ u \in M_i} { \tfrac{1}{\hat{\sigma}^{2}_{t, u}}  }} \;, \\
%{\hat{\sigma}^2_{t+1, u}} &: = \frac{1}{|N_u|} \sum_{i \in N_u} \|A_{iu} - \mu^{{{\MF con}(t)}}_i(A)\|^2  \;.
%\end{align*}
$
\mu^{{\MF NBI}(t)}_i(A) 
:= \lim_{\tau \to \infty} \bar{\mu}_i(A_i, \hat{\sigma}^2_{t, M_i})$
and 
${\hat{\sigma}^2_{t+1, u}} : = \sum_{i \in N_u}  \|A_{iu} - \mu^{{{\MF NBI}(t)}}_i(A)\|^2/|N_u|.$
\vspace{0.1cm}
\item {\MF Average}
%This algorithm 
just takes the average of workers' observations without learning workers' variances,
 i.e., $\mu^{\MF avg}_i (A) := \frac{1}{|M_i|} \sum_{u \in M_i} A_{iu}.$
\vspace{0.1cm}
\item {\MF Strong/Weak-Oracle} are two artificial estimators
which have free access to workers' variances $\sigma^2.$ They would outperform all existing algorithms, even including the optimal estimator. 
For each task $i$,
{\MF Strong-Oracle} uses every worker's true variance, i.e., 
$\mu^{\MF strong}_i (A, \sigma^2) := \lim_{\tau \to \infty} \EXP[\mu_i \mid  A, \sigma^2].$
{\MF Weak-Oracle} uses just the true variances of leaf workers, denoted by $\partial T_i$, in Breadth-first search tree, denoted by $T_i$, of root $i$, i.e., $\mu^{\MF weak}_i (A, \sigma^2) := \lim_{\tau \to \infty} \EXP[\mu_i \mid  A, \sigma^2_{\partial T_i}]$.
%comparing the worker's answers and the others' answers,
%  and then it aggregates the workers' answers based on the estimated variances.
%  Formally, it outputs
%  $\mu^{\MF \nonIter}_i(A) := \mu^{\MF ora}_i (A, {\hat{\sigma}^2_{{\MF sim}}})$
%  where ${\hat{\sigma}^2_{{\MF sim},u}} : = \frac{1}{|N_u|}\sum_{i \in N_u} \|A_{iu} - \mu^{\MF avg}_i(A)\|^2$.
%  
  
%   \item {\MF \nonIter}: This has a simple mechanism to estimate each worker's variance
%by comparing the worker's answers and the others' answers,
%  and then it aggregates the workers' answers based on the estimated variances.
%  Formally, it outputs
%  $\mu^{\MF nav}_i(A) := \mu^{\MF ora}_i (A, {\hat{\sigma}^2_{{\MF sim}}})$
%  where ${\hat{\sigma}^2_{{\MF sim},u}} : = \frac{1}{|N_u|}\sum_{i \in N_u} \|A_{iu} - \mu^{\MF avg}_i(A)\|^2$.

% Then it aggregates $A_{iu}$'s based on the estimated variances %${\sigma^2_{\MF sim}}$
%  and outputs
\end{compactitem}

Recall that 
the true positions $\mu_i$'s are assumed to be drawn from 
the spherical Gaussian with mean $\nu_i$ and variance $\tau$.
As $\nu_i$'s and $\tau$ are hard to obtain in practice, 
in all experiments, 
we implement {\MF {\bayIter}} with no knowledge on 
{\em the prior distribution of the true positions of the tasks} by taking the limit of {\MF {\bayIter}} as $\tau \to \infty.$ 
%In our model, 
%For 
%mathematical rigorousness, %we assumed that we have 
%we assumed to have some information of 
%the true position $\mu_i$ in advance as specified by the density of
%$\mu_i$ as the spherical Gaussian with mean $\nu_i$ and variance $\tau$.
%Since such statistical information might not be easy
%to obtain in practice, 
%we implement {\MF {\bayIter}} with no
%prior knowledge on {\em true positions of the tasks} by taking the limit of {\MF {\bayIter}} as $\tau \to \infty.$
%$\lim_{\tau \to \infty} \mathcal{C}_i (A_i, \sigma^2_{M_i})$, i.e.,
%where large $\tau$ dilutes the information.
Note that our theoretical guarantees on {\MF {\bayIter}} still hold in this regime.

Our baseline comparisons include the ones with 
a simple approach of {\MF Average} and the fundamental lower bounds
from {\MF Strong/Weak-Oracle}.
{\MF Strong-Oracle} and {\MF Weak-Oracle} correspond to the fundamental limits that we compare with {\MF BI} in Theorems~\ref{thm:quantify}~and~\ref{thm:optimality}, respectively.
As we use the prior distribution in the synthetic experiments, 
one may ask about what is the gain of this side information. 
To address this, 
we test a non-Bayesian algorithm ({\MF {\nonIter}}) that 
does not assume such prior information, iteratively estimating
task answers and worker variances based on consensus.
Note that a similar idea has been also investigated in \cite{RYZ10}.

\subsection{Synthetic Datasets}

\label{sec:synthetic}
%\noindent {\bf Synthetic datasets.} 

Since our theoretical results cover a large $n$ regime, 
we test a more challenging regime of modest size $n = 200$. 
Synthetic datasets are generated by
the set of random $(\ell, r)$-regular bipartite graphs of $200$
object detection tasks where each task $i$ is associated with the true
position $\mu_i$ chosen uniformly at random in a $100 \times 100$ image.
We randomly choose each worker's variance using
$\mathcal{S}_{\text{small}} = \{10, 100, 1000\}$ or
$\mathcal{S}_{\text{large}} = \{10, 100, 5000\}$.
The simulation results with
varying $r$ and $\ell$ are plotted in
Figures~\ref{fig:all}\subref{fig:small_l5}-\subref{fig:large_l5} and
\ref{fig:all}\subref{fig:small_r5}-\subref{fig:large_r5}, where we
take the average of $50$ random instances.

\noindent {\bf Optimality of {\bayIter}.} 
%As discussed in Section~\ref{sec:bp_quanti},
%As discussed in Section~\ref{sec:bp_opt},
As discussed in Section~\ref{sec:result},
Figures~\ref{fig:all}\subref{fig:small_l5}-\subref{fig:large_r5} show that 
for {\it all} $(\ell, r)$, {\MF {\bayIter}} closely achieves the fundamental limit of {\MF Weak-Oracle},
whereas {\MF Average} and {\MF {\nonIter}} have the suboptimal performance.
We also observe that 
{\MF Weak-Oracle} with the challenge of identifying reliable workers 
indeed provides tighter fundamental limit than {\MF Strong-Oracle}, as discussed in Section~\ref{sec:result}.
Overall, {\MF {\nonIter}} has a small constant gap to {\MF \bayIter}, 
which quantifies the gain of {\MF{\bayIter}} using the matched prior distribution.
{\MF Average} shows the significant performance loss, compared to the optimal {\MF{\bayIter}}
or even {\MF{\nonIter}}.
For example, in  Figures~\ref{fig:all}\subref{fig:small_r5}-\subref{fig:large_r5},
in order to make MSE less than $100$ with $\mathcal{S}_{\text{small}}$, 
{\MF \bayIter} and {\MF \nonIter} require only $\ell \ge 3$,
but {\MF Average} requires $\ell \ge 9$, implying that 
{\MF Average} needs to hire three times more workers per task than others.

\noindent {\bf Importance of worker identification.} 
Comparing Figures~\ref{fig:all}\subref{fig:small_r5}-\subref{fig:large_r5},
we observe that under the minimum of workers' variances fixed,
both {\MF {\bayIter}} and {\MF {\nonIter}}, which
identify reliable workers and adaptively weight their answers,
sustain good performance for both small and large maximum
worker variances. However, the performance of {\MF Average}, which does not distinguish workers, 
is significantly degenerated by spammers with large variance from $\mathcal{S}_{\rm large}$.
This shows the importance of classifying workers in making the
estimator robust to spammers or adversary.

%It is interesting to see that MSE of {\MF Simple} estimating workers' variance 
%decreases as $r$ increases, similarly as {\MF \bayIter}.
 %doesn't have such ability.
%, even though
% BP doesn't have the free access to the workers variance
%which the oracle estimator has. 

\noindent {\bf Impact of $(\ell, r)$.}
As Figures~\ref{fig:all}\subref{fig:small_r5}-\subref{fig:large_r5} show, 
increasing $\ell$ (or budget) exponentially reduces MSE's of all algorithms, while 
the value of exponent varies for each algorithm. 
In Figures~\ref{fig:all}\subref{fig:small_l5}-\subref{fig:large_l5}, 
the gap between the optimal {\MF {\bayIter}} (or {\MF Weak-Oracle}) and {\MF Strong-Oracle}
quantifies the difficulty in identifying reliable workers.
As studied in Theorem~\ref{thm:quantify}, the gap is diminishing exponentially fast with increasing $r,$
whereas the fundamental limit of {\MF Strong-Oracle} does not change with $r$.
Hence, for efficiency from the worker identification, 
the task requester needs to assign each worker so as to answer a certain number of tasks at least,
while letting a worker solve too many tasks may be impractical but also unhelpful to increase accuracy.

%Since {\MF Simple} estimates workers' variances in a simple way, % at least,
% % although the estimation is not sophisticated as BP.
% it can obtain the certain level of the tolerance to spammers.

%In addition, the performance degeneration by increasing the maximum variance 
%(compare Figures~\cref{fig:small_l5,fig:large_l5}) 
%is not that severe as the one of {\MF Average}, while 
%{\MF BP} has a negligible performance degeneration when $r$ is large.
%
%We note that the arithmetic means of $\mathcal{S}_{\text{small}}$ and $\mathcal{S}_{\text{large}}$
%are $370$ and $1703$, respectively, while the harmonic means of them are values near $27$.
%Hence, as discussed in Section~\ref{sec:bp_quanti}, {\MF Average} 
% significantly different while
%Hence

 \begin{table}[t!]
\caption{ 
Estimation quality of {\MF Average}, {\MF \nonIter}, {\MF \bayIter} with $\mathcal{S}_{\textnormal{est}}$, {\MF PG1} 
and {\MF Weak/Strong-Oracle}'s
on crowdsourced {\MF FG-NET} datasets from Amazon MTurk workers.
}
\vspace{-0.3cm}
\label{tab:fg-net-table}
\begin{center}
{  \fontsize{9}{9}
\begin{sc}
\begin{tabular}{ccccccc}
\hline
\multirow{2}{*}{\!\!{Estimator}\!\!}	& \multirow{2}{*}{\!\!{\MF Average}\!\!} & \multirow{2}{*}{\!\!{\MF \nonIter}\!\!} & \multirow{2}{*}{\!\!\MF\bayIter\!\!} &  \multirow{2}{*}{\!\!{\MF PG1}\!\!} & {\!\!{\MF Weak}\!\!} & {\!\!{\MF Strong}\!\!} \\
& & & & & \!\!{\MF Oracle}\!\! & \!\!{\MF Oracle}\!\! \\
\hline 
\!\!{Data noise}\!\!	& \multirow{2}{*}{\!\!34.99\!\!}& \multirow{2}{*}{\!\!32.80\!\!} & \multirow{2}{*}{\!\!28.72\!\!} &\multirow{2}{*}{\!\!33.54\!\!}	&\multirow{2}{*}{\!\!28.68\!\!}	 &\multirow{2}{*}{\!\!28.45\!\!}	 \\
  ({MSE})  & {} & {} & {} & {} & {} & {} \\
  \hline
\end{tabular}
\end{sc}
}
\end{center}
\end{table}

%%%%%%%%%%%%%%
% Previous table
%%%%%%%%%%%%%%
\iffalse
 \begin{table*}[t]
\begin{minipage}{1\columnwidth}
\caption{ 
Estimation quality of {\MF Average}, {\MF \nonIter}, {\MF \bayIter} with $\mathcal{S}_{\textnormal{est}}$, {\MF PG1} and {\MF Weak/Strong-Oracle}
on crowdsourced {\MF FG-NET} datasets from Amazon MTurk workers
in terms of MSE.
}
\label{tab:fg-net-table}
\begin{center}
%\begin{small}
{ % \fontsize{9}{9}
\begin{sc}
  %   \begin{minipage}{1\columnwidth}   
\begin{tabular}{cccc}
\hline
% \multirow{2}{*}{Estimator}  & \multirow{2}{*}{Estimator} & \multirow{2}{*}{MSE} 	\\
  \multirow{2}{*}{Estimator}  & Data noise  \\ %\multicolumn{2}{c}{Testing error} \\
{} & ({MSE})  %&  (MAE)	
 \\
\hline
%\multirow{4}{*}{\minitab[c]{VOC 07/12}}
					{\MF Average}		&34.99	%& 5.006 	
							\\
					{\MF \nonIter}		&32.80	%&4.687 	 
							 \\
					 {\MF\bayIter}		&28.72   %& 4.773 	 
					 		\\
					{\MF PG1} 		&33.54	%& 4.623 	
							 \\
					{\MF Weak-Oracle} 		&28.68	%& 4.623 	
							 \\
					{\MF Strong-Oracle} 		&28.45	%& 4.623 	
							 \\
\hline
%AVG: 5.005607e+00 
%AVG: 3.226521e+00 
%BP1: 4.773906e+00 
%BP1: 3.100259e+00 
%BP2: 4.386465e+00 
%BP2: 2.890493e+00 
%Iter: 4.687298e+00 
%Iter: 3.135448e+00 
%Ora: 4.623180e+00 
%Ora: 3.002691e+00
\end{tabular}
\end{sc}
%\end{small}
}
%\vspace{cm}
\end{center}
\end{minipage}
\end{table*}
\fi
%%%%%%%%%%%%%%

  \subsection{Human Age Prediction}

\label{sec:HAP}
%To evaluate performance of our methods in real-world scenario, 
%we use
We also present experiment results on datasets from a {\em real-world} crowdsourcing system.
We use {\MF FG-NET} datasets
which has been widely used as a benchmark for facial age estimation \cite{lanitis2008comparative}.
The dataset contains $1,002$ photos of $82$ individuals' faces, in which
each photo is labeled with a biological age %of individual when the photo was taken
as the ground truth.
Furthermore, \citet{han2015demographic} provide 
crowdsourced labels on {\MF FG-NET} datasets,
in which $165$ workers in Amazon Mechanical Turk (MTurk) answer 
their own age estimation on given subset of $1,002$ photos,
so that each photo has 
$10$ answers from workers, while each worker provides a different number (from $1$ to $457$) of
 answers, and $60.73$ answers on average.
%Examples of workers' age estimation is presented in Figure~\ref{fig:face_crowd}.
% {\MF FG-NET} datasets on human age prediction from \cite{han2015demographic},
%where

\noindent{\bf Prior estimation.} 
In processing the real-world dataset, the prior distribution on noise level is not provided in advance, while it is required to run
%while we need 
%it for 
{\MF {\bayIter}}.
%Hence, to obtain support for {\MF {\bayIter}},
To infer the prior distribution, we first study workers' answer patterns.
We often observe two extreme classes of answers for a task: 
a few outliers and consensus among majority. 
For example, in Figures~\ref{fig:face_crowd}\subref{fig:easy} and \ref{fig:face_crowd}\subref{fig:hard},
there exist noisy answers $5$ and $7$, respectively, which are far
from the majority's answers, $1$ and $55$, respectively.
%The easiest and hardest samples of {\MF FG-NET} datasets in terms of
%average absolute error of crowd workers' answers:
%as age estimations on each photo of (a) 1-year-old (b) 35-year-old individuals,
%Amazon M-turk workers provide answers (1, 1, 1, 1, 1, 1, 1, 1, 2, 5)
%and (7, 51, 52, 55, 55, 55, 59, 63, 66, 67), respectively.
Such observations suggest to choose a simple support, e.g.~$\mathcal{S}=\{  \sigma^2_{\rm  good}, \sigma^2_{\rm bad}\}$.
%In particular, the examples in Figure~\ref{fig:face_crowd} suggest to use a simple support 
In particular, without any use of ground truth, we first run  {\MF {\nonIter}} 
%for estimating workers' reliabilities 
and use
the top $10\%$ and bottom $10\%$ workers' reliabilities as the binary support, which is
$\mathcal{S}_{\textnormal{est}} = \{6.687, 62.56\}$. % in our experiment.
%For {\MF FG-NET}, we use {\MF {\bayIter}} with the estimated support $\mathcal{S} = [6.687, 62.56]$.

\noindent{\bf Validation on the estimated prior.}
For {\MF {FG-NET}}, we additionally test {\MF PG1} which is Gibbs sampling algorithm\footnote{As suggested in \cite{piech2013tuned}, we use $800$ iterations of Gibbs sampling after discarding the initial $80$ burn-in samples.} relying on a sophisticated worker model, called $\text{PG}_1$ model in \cite{piech2013tuned}, with biased Gaussian noises where worker variance and bias are drawn from continuous supports.
In Table~\ref{tab:fg-net-table}, we compare the estimation of {\MF {\bayIter}} to other algorithms.
Observe that MSE of {\MF {\bayIter}} with the binary support $\mathcal{S}_{\textnormal{est}}$ 
is close to those of {\MF Weak/Strong-Oracle}, while the other algorithms have some gaps.
Despite using a sophisticated model, $\text{PG}_1$, {\MF PG1} is
observed to perform worse than {\MF {\bayIter}}.
%We conjecture that in order to exploit the advantage in modeling, 
This is because {\MF PG1} needs non-trivial parameter optimization based on {\it training dataset}, % as in \citet{piech2013tuned} 
which incurs overfitting.
%while the use of training dataset can also improve {\MF {\bayIter}}.
This result from the real workers supports the value of
our simplified modeling on workers' noise. % as in our model.
For interested readers,
we also report the regression accuracy of 
deep neural networks trained using the pruned dataset produced by
different algorithm in the supplementary material.

\begin{figure}[t!]
{
\captionsetup[subfloat]{captionskip=3pt}
%\begin{figure}[t!]
\begin{minipage}{1\columnwidth}
\begin{center}
     \subfloat[1-year-old \label{fig:easy}]
     {	\includegraphics[height=0.2\columnwidth, trim={0cm 0.1cm 0.1cm 0cm},clip]{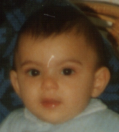}} 
     \hspace{0.8cm}
     \subfloat[35-year-old\label{fig:hard}]
     {	\includegraphics[height=0.2\columnwidth, trim={0cm 0.1cm 0cm 0.1cm},clip]{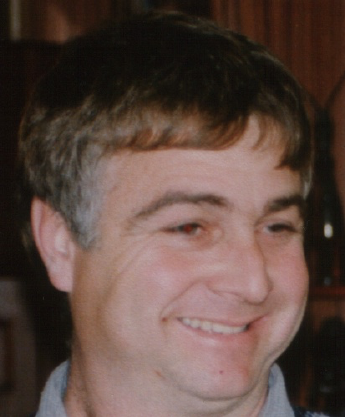}} 
     \end{center}
     \vspace{-0.3cm}
   \captionof{figure}{
Easy and hard samples from {\MF FG-NET} in terms of
average absolute error of workers' answers:
(1,1,1,1,1,1,1,1,2,5) and (7,51,52,55,55,55,59,63,66,67)
on photo of (a) 1-year-old, and (b) 35-year-old, resp.}
\label{fig:face_crowd}
%\label{fig:all}
% easiest: 046A04.JPG
% AVG: 7.761955e-02
% BP1: 3.440957e-01
% BP2: 1.571412e-01
% Iter: 2.708888e-01
% Ora: 4.885221e-01
%     hardest: 032A15.JPG
% AVG: 2.245510e+01
% BP1: 2.125741e+01
% BP2: 2.334494e+01
% Iter: 2.023362e+01
% Ora: 2.572180e+01
\end{minipage}
%\end{figure}
}
\end{figure}

  \vspace{-0.1cm}
\section{Conclusion}
\vspace{-0.2cm}

\label{sec:conclusion}
%\vspace{-0.1cm}

We study a model to address the problem of aggregating 
real-valued responses from a crowd of workers with 
heterogeneous noise level. In particular, inspired by the observation on answer pattern of Amazon MTurk workers,
we use a canonical noise model.
This modeling allows us to pose this crowdsourced regression problem as an inference problem over a graphical model, naturally motivating the proposed {\bayIter} algorithm based on BP.
Typically, the analysis of such iterative algorithms is not tractable even
for estimating discrete labels.
However, 
our theoretical framework,
inspired by recent advances in BP, e.g. \cite{mossel2014},
provides
sharp guarantees on {\bayIter}
and shows its optimality for a broad range of parameters. 

\iffalse
A promising future research direction with significant practical interest is 
the question of how to adaptively assign tasks to make more efficient use of the budget. 
As workers typically arrive in an online fashion, 
such heuristics are used widely in practice with little theoretical understanding. 
Efficient and principled schemes have  given significant gain in, for instance, voting in social media \cite{JJN16}. 
There are recent advances for adaptive crowdsourced classification \cite{HJV13,KO16}, but 
these approaches rely on the discrete nature of the problem
while crowdsourced regression requires to handle real-valued responses.
\fi
%For crowdsourced regression, it requires innovative ideas to characterize confidence intervals for 
%real-valued responses. 
%Our model seems a promising approach for this since
%it converts analysis of discrete random variables (noise level) to 
%that of real-valued ones (ground truth).

An important research direction is in generalizing the proposed noise model. 
First natural generalization is to allow differing task difficulties, by adding an additional independent Gaussian with variance $\sigma^2_i$ 
for answers on task $i$. Larger variance represents more difficult tasks. 
Second natural generalization is to allow worker biases, by adding a constant shift of value $\mu_u$ for answers given by worker $u$ as \citet{piech2013tuned} considered.
Our observations on the crowdsourced {\MF FG-NET} datasets also suggest 
heterogeneous task difficulty and boundary effect on tasks. 
%suggests several  improving noise models to capture 
%In fact, we observe such characteristics in the crowdsourced {\MF FG-NET} datasets.
%It is also independent of research interest to develop more realistic model capturing
As in the examples in Figure~\ref{fig:face_crowd}, we often observe less estimation error
for photos of younger individuals: MSE $15.00$ for $233$ photos whose ages are below $5$,
and MSE $85.39$ for $769$ photos whose ages are above $5$.
This shows heterogeneous task difficulty and boundary effect at the same time:
age prediction on older individual is more challenging due to more variations,
and human age cannot be negative.
% as its object biological age.

%
%   24.2983
%
%   65.3440
%
%num_young =
%
%   233
%
%num_old =
%
%   769

%%% Local Variables:
%%% mode: latex
%%% TeX-master: "main"
%%% End:

{\small
\bibliographystyle{plainnat}
\bibliography{ref}

\begin{thebibliography}{41}
\providecommand{\natexlab}[1]{#1}
\providecommand{\url}[1]{\texttt{#1}}
\expandafter\ifx\csname urlstyle\endcsname\relax
  \providecommand{\doi}[1]{doi: #1}\else
  \providecommand{\doi}{doi: \begingroup \urlstyle{rm}\Url}\fi

\bibitem[Bernstein et~al.(2011)Bernstein, Brandt, Miller, and Karger]{BBMK11}
M.~S. Bernstein, J.~Brandt, R.~C. Miller, and D.~R. Karger.
\newblock Crowds in two seconds: enabling realtime crowd-powered interfaces.
\newblock In \emph{Proc. of ACM UIST}, 2011.

\bibitem[Dalvi et~al.(2013)Dalvi, Dasgupta, Kumar, and Rastogi]{DDKR13}
N.~Dalvi, A.~Dasgupta, R.~Kumar, and V.~Rastogi.
\newblock Aggregating crowdsourced binary ratings.
\newblock In \emph{Proc. of ACM WWW}, 2013.

\bibitem[Dawid and Skene(1979)]{dawid1979}
A.~P. Dawid and A.~M. Skene.
\newblock Maximum likelihood estimation of observer error-rates using the {EM}
  algorithm.
\newblock \emph{Journal of the Royal Statistical Society. Series C (Applied
  Statistics)}, 28\penalty0 (1):\penalty0 20--28, 1979.

\bibitem[{De Alfaro} and Shavlovsky(2014)]{de2014crowdgrader}
L.~{De Alfaro} and M.~Shavlovsky.
\newblock {CrowdGrader}: {A} tool for crowdsourcing the evaluation of homework
  assignments.
\newblock In \emph{Proc. of ACM SIGCSE}, 2014.

\bibitem[Deng et~al.(2009)Deng, Dong, Socher, Li, Li, and
  Fei-Fei]{deng2009imagenet}
J.~Deng, W.~Dong, R.~Socher, L.-J. Li, K.~Li, and L.~Fei-Fei.
\newblock Imagenet: A large-scale hierarchical image database.
\newblock In \emph{Proc. of IEEE CVPR}, 2009.

\bibitem[Everingham et~al.(2015)Everingham, Eslami, Gool, Williams, Winn, and
  Zisserman]{pascal15}
M.~Everingham, SM.~A. Eslami, L.~Van Gool, C.~KI Williams, J.~Winn, and
  A.~Zisserman.
\newblock The pascal visual object classes challenge: A retrospective.
\newblock \emph{International Journal of Computer Vision}, 111\penalty0
  (1):\penalty0 98--136, 2015.

\bibitem[Ghosh et~al.(2011)Ghosh, Kale, and McAfee]{GKM11}
A.~Ghosh, S.~Kale, and P.~McAfee.
\newblock Who moderates the moderators?: Crowdsourcing abuse detection in
  user-generated content.
\newblock In \emph{Proc. of ACM EC}, 2011.

\bibitem[Girshick(2015)]{girshick2015fast}
R.~Girshick.
\newblock Fast r-cnn.
\newblock In \emph{Proc. of the IEEE ICCV}, 2015.

\bibitem[Goldin and Ashley(2011)]{goldin2011peering}
I.~M. Goldin and K.~D. Ashley.
\newblock Peering inside peer review with bayesian models.
\newblock In \emph{Proc. of AIED}, pages 90--97. Springer, 2011.

\bibitem[Han et~al.(2015)Han, Otto, Liu, and Jain]{han2015demographic}
H.~Han, C.~Otto, X.~Liu, and A.~K. Jain.
\newblock Demographic estimation from face images: Human vs. machine
  performance.
\newblock \emph{IEEE transactions on pattern analysis and machine
  intelligence}, 37\penalty0 (6):\penalty0 1148--1161, 2015.

\bibitem[Jordan(1998)]{jordan1998learning}
M.~I. Jordan.
\newblock \emph{Learning in graphical models}, volume~89.
\newblock Springer Science \& Business Media, 1998.

\bibitem[Jordan(2004)]{jordan2004graphical}
M.~I. Jordan.
\newblock Graphical models.
\newblock \emph{Statistical Science}, 19\penalty0 (1):\penalty0 140--155, 2004.

\bibitem[Karger et~al.(2011)Karger, Oh, and Shah]{kos2011}
D.~R. Karger, S.~Oh, and D.~Shah.
\newblock Iterative learning for reliable crowdsourcing systems.
\newblock In \emph{Proc. of NIPS}, 2011.

\bibitem[Karger et~al.(2013)Karger, Oh, and Shah]{KOS13SIGMETRICS}
D.~R. Karger, S.~Oh, and D.~Shah.
\newblock Efficient crowdsourcing for multi-class labeling.
\newblock In \emph{Proc. of ACM SIGMETRICS}, 2013.

\bibitem[Karger et~al.(2014)Karger, Oh, and Shah]{kos2014}
D.~R. Karger, S.~Oh, and D.~Shah.
\newblock Budget-optimal task allocation for reliable crowdsourcing systems.
\newblock \emph{Operations Research}, 62\penalty0 (1):\penalty0 1--24, 2014.

\bibitem[Khetan and Oh(2016)]{KO16}
A.~Khetan and S.~Oh.
\newblock Achieving budget-optimality with adaptive schemes in crowdsourcing.
\newblock In \emph{Proc. of NIPS}, pages 4844--4852, 2016.

\bibitem[Kudekar et~al.(2013)Kudekar, Richardson, and
  Urbanke]{kudekar2013spatially}
S.~Kudekar, T.~Richardson, and R.~L. Urbanke.
\newblock Spatially coupled ensembles universally achieve capacity under belief
  propagation.
\newblock \emph{IEEE Transactions on Information Theory}, 59\penalty0
  (12):\penalty0 7761--7813, 2013.

\bibitem[Lanitis(2008)]{lanitis2008comparative}
A.~Lanitis.
\newblock Comparative evaluation of automatic age-progression methodologies.
\newblock \emph{EURASIP Journal on Advances in Signal Processing},
  2008:\penalty0 101, 2008.

\bibitem[Lee et~al.(2012)Lee, Steyvers, Young, and Miller]{LSD12}
M.~D. Lee, M.~Steyvers, M.~De Young, and B.~Miller.
\newblock Inferring expertise in knowledge and prediction ranking tasks.
\newblock \emph{Topics in cognitive science}, 4\penalty0 (1):\penalty0
  151--163, 2012.

\bibitem[Liu et~al.(2012)Liu, Peng, and Ihler]{liu2012}
Q.~Liu, J.~Peng, and A.~T. Ihler.
\newblock Variational inference for crowdsourcing.
\newblock In \emph{Proc. of NIPS}, 2012.

\bibitem[Liu et~al.(2013)Liu, Ihler, and Steyvers]{liu2013scoring}
Q.~Liu, A.~T. Ihler, and M.~Steyvers.
\newblock Scoring workers in crowdsourcing: How many control questions are
  enough?
\newblock In \emph{Proc. of NIPS}, pages 1914--1922, 2013.

\bibitem[Liu et~al.(2016)Liu, Anguelov, Erhan, Szegedy, Reed, Fu, and
  Berg]{liu2016ssd}
W.~Liu, D.~Anguelov, D.~Erhan, C.~Szegedy, S.~Reed, C.-Y. Fu, and A.~C. Berg.
\newblock {SSD}: Single shot multibox detector.
\newblock In \emph{Proc. of ECCV}, 2016.

\bibitem[Mossel et~al.(2014)Mossel, Neeman, and Sly]{mossel2014}
E.~Mossel, J.~Neeman, and A.~Sly.
\newblock Belief propagation, robust reconstruction and optimal recovery of
  block models.
\newblock In \emph{Proc. of COLT}, 2014.

\bibitem[Murphy et~al.(1999)Murphy, Weiss, and Jordan]{murphy1999loopy}
K.~P. Murphy, Y.~Weiss, and M.~I. Jordan.
\newblock Loopy belief propagation for approximate inference: An empirical
  study.
\newblock In \emph{Proc. of UAI}, 1999.

\bibitem[Ok et~al.(2016)Ok, Oh, Shin, and Yi]{ok2016icml}
J.~Ok, S.~Oh, J.~Shin, and Y.~Yi.
\newblock Optimality of belief propagtion for crowdsourced classification.
\newblock In \emph{Proc. of ICML}, 2016.

\bibitem[Panis et~al.(2016)Panis, Lanitis, Tsapatsoulis, and
  Cootes]{panis2016overview}
G.~Panis, A.~Lanitis, N.~Tsapatsoulis, and T.~F. Cootes.
\newblock Overview of research on facial ageing using the {FG-NET} ageing
  database.
\newblock \emph{IET Biometrics}, 5\penalty0 (2):\penalty0 37--46, 2016.

\bibitem[Pearl(1982)]{pearl1982}
J.~Pearl.
\newblock Reverend bayes on inference engines: A distributed hierarchical
  approach.
\newblock In \emph{Proc. of AAAI}, 1982.

\bibitem[Peng et~al.(2013)Peng, Qiang, Ihler, and Berger]{PQIB13}
J.~Peng, L.~Qiang, A.~Ihler, and B.~Berger.
\newblock Crowdsourcing for structured labeling with applications to protein
  folding.
\newblock In \emph{Proc. of ICML Machine Learning Meets Crowdsourcing
  Workshop}, 2013.

\bibitem[Piech et~al.(2013)Piech, Huang, Chen, Do, Ng, and
  Koller]{piech2013tuned}
C.~Piech, J.~Huang, Z.~Chen, C.~Do, A.~Ng, and D.~Koller.
\newblock Tuned models of peer assessment in {MOOC}s.
\newblock In \emph{Proc. of EDM}, 2013.

\bibitem[Raykar et~al.(2010)Raykar, Yu, Zhao, Valadez, Florin, Bogoni, Moy, and
  Blei]{RYZ10}
V.~C. Raykar, S.~Yu, L.~H. Zhao, G.~H. Valadez, C.~Florin, L.~Bogoni, L.~Moy,
  and D.~Blei.
\newblock Learning from crowds.
\newblock \emph{Journal of Machine Learning Research}, 11:\penalty0 1297--1322,
  2010.

\bibitem[Ren et~al.(2015)Ren, He, Girshick, and Sun]{ren2015faster}
S.~Ren, K.~He, R.~Girshick, and J.~Sun.
\newblock Faster r-cnn: Towards real-time object detection with region proposal
  networks.
\newblock In \emph{Proc. of NIPS}, 2015.

\bibitem[Rothe et~al.(2015)Rothe, Timofte, and {Van Gool}]{rothe2015dex}
R.~Rothe, R.~Timofte, and L.~{Van Gool}.
\newblock Dex: Deep expectation of apparent age from a single image.
\newblock In \emph{Proc. of ICCV}, pages 10--15, 2015.

\bibitem[Salvo et~al.(2013)Salvo, Giordano, and Kavasidis]{DGK13}
R.~Di Salvo, D.~Giordano, and I.~Kavasidis.
\newblock A crowdsourcing approach to support video annotation.
\newblock In \emph{Proc. of the International Workshop on Video and Image
  Ground Truth in Computer Vision Applications}, page~8. ACM, 2013.

\bibitem[Shah et~al.(2016)Shah, Balakrishnan, and
  Wainwright]{shah2016permutation}
N.~B. Shah, S.~Balakrishnan, and M.~J. Wainwright.
\newblock A permutation-based model for crowd labeling: Optimal estimation and
  robustness.
\newblock \emph{arXiv preprint arXiv:1606.09632}, 2016.

\bibitem[Simonyan and Zisserman(2014)]{simonyan2014very}
K.~Simonyan and A.~Zisserman.
\newblock Very deep convolutional networks for large-scale image recognition.
\newblock \emph{arXiv preprint arXiv:1409.1556}, 2014.

\bibitem[Su et~al.(2012)Su, Deng, and Fei-Fei]{su2012crowdsourcing}
H.~Su, J.~Deng, and L.~Fei-Fei.
\newblock Crowdsourcing annotations for visual object detection.
\newblock In \emph{Proc. of Workshops at AAAI Conference on Artificial
  Intelligence}, 2012.

\bibitem[Wu et~al.(2012)Wu, Fan, and Yu]{WFY12}
X.~Wu, W.~Fan, and Y.~Yu.
\newblock Sembler: Ensembling crowd sequential labeling for improved quality.
\newblock In \emph{Proc. of AAAI}, 2012.

\bibitem[Yanover et~al.(2006)Yanover, Meltzer, and Weiss]{yanover2006linear}
C.~Yanover, T.~Meltzer, and Y.~Weiss.
\newblock Linear programming relaxations and belief propagation--an empirical
  study.
\newblock \emph{Journal of Machine Learning Research}, 7\penalty0
  (Sep):\penalty0 1887--1907, 2006.

\bibitem[Zhang et~al.(2014)Zhang, Chen, Zhou, and Jordan]{zhang2014}
Y.~Zhang, X.~Chen, D.~Zhou, and M.~I. Jordan.
\newblock Spectral methods meet em: A provably optimal algorithm for
  crowdsourcing.
\newblock In \emph{Proc. of NIPS}, 2014.

\bibitem[Zhou et~al.(2012)Zhou, Platt, Basu, and Mao]{ZPBM12}
D.~Zhou, J.~Platt, S.~Basu, and Y.~Mao.
\newblock Learning from the wisdom of crowds by minimax entropy.
\newblock In \emph{Proc. of NIPS}, 2012.

\bibitem[Zhou et~al.(2015)Zhou, Liu, Platt, Meek, and Shah]{ZLPCS15}
D.~Zhou, Q.~Liu, J.~C. Platt, C.~Meek, and N.~B. Shah.
\newblock Regularized minimax conditional entropy for crowdsourcing.
\newblock \emph{arXiv preprint arXiv:1503.07240}, 2015.

\end{thebibliography}
}

\appendix

% !TEX root =  main.tex
%\clearpage

\section{Importance of Crowdsourcing System on Machine Learning Applications}
	%to Demonstrate the Importance of Estimator}
\vspace{-0.1cm}

Crowdsourcing is a primary marketplace to get labels on training datasets, 
to be used to train machine learning models. 
In this section, using both semi-synthetic and real datasets, 
we investigate the impact of having higher quality labels 
 on real-world machine learning tasks. 
We show that sophisticated regression algorithms like {\MF \bayIter}
can produce high quality labels on 
 the crowdsourced training datasets, 
 improving the end-to-end performance of convolutional neural network (CNN) 
 on visual object detection or human age prediction.
This highlights the importance of estimator but also
justifies the use of the proposed {\MF \bayIter}, in a real world system. 

%We also provide the experimental results demonstrating
%the impact of crowdsourced regression on real-world machine learning tasks.
%To do so, we investigate how much an efficient estimator, refining
% the crowdsourced training dataset,
% improves the performance of convolutional neural network (CNN) 
% for the object detection problem.

\subsection{Visual Object Detection} 

\label{sec:VOD}

%which is a fundamental problem in the area of computer vision
%The vision task requires a huge amount of the training datasets often obtained by 
 %the crowdsourcing system, e.g., Amazon's mechanical turk \cite{deng2009imagenet}.
\noindent{\bf Emulating a crowdsourcing system.} 
To do so, we use PASCAL
 visual object classes (VOC) datasets from \cite{pascal15}:
%two real-world object detection datasets:
{\MF VOC-07/12} 
consisting of $40,058$ annotated objects in $16,551$ images. % in total.
%consisting of $12,608$ and $27,450$ annotated objects in $5,011$ and $11,540$ images in total.
Each object is annotated by a %rectangular
 bounding box expressed
 by two opposite corner points.
% We consider each corner point as the true position of a task in our model.
We emulate the crowdsourcing system with 
a random $(\ell = 3, r = 10)$-regular bipartite graph % between images and virtual workers. 
where each image is assigned to $3$ workers and each worker
is assigned $10$ images ($\simeq 24.2$ objects on average) to draw the bounding boxes of every object in the assigned images.
% the set of images.  
Each worker has variance drawn uniformly at random from support $\mathcal{S} = \{10, 1000\}$.
 and generates noisy responses of which examples are shown in 
Figure~\ref{fig:annotation}.

\begin{figure*}[h!]
\begin{center}
     \subfloat[When $\sigma^2_u  = 10$ \label{fig:annotation10}]
     {\includegraphics[height=2.5cm]{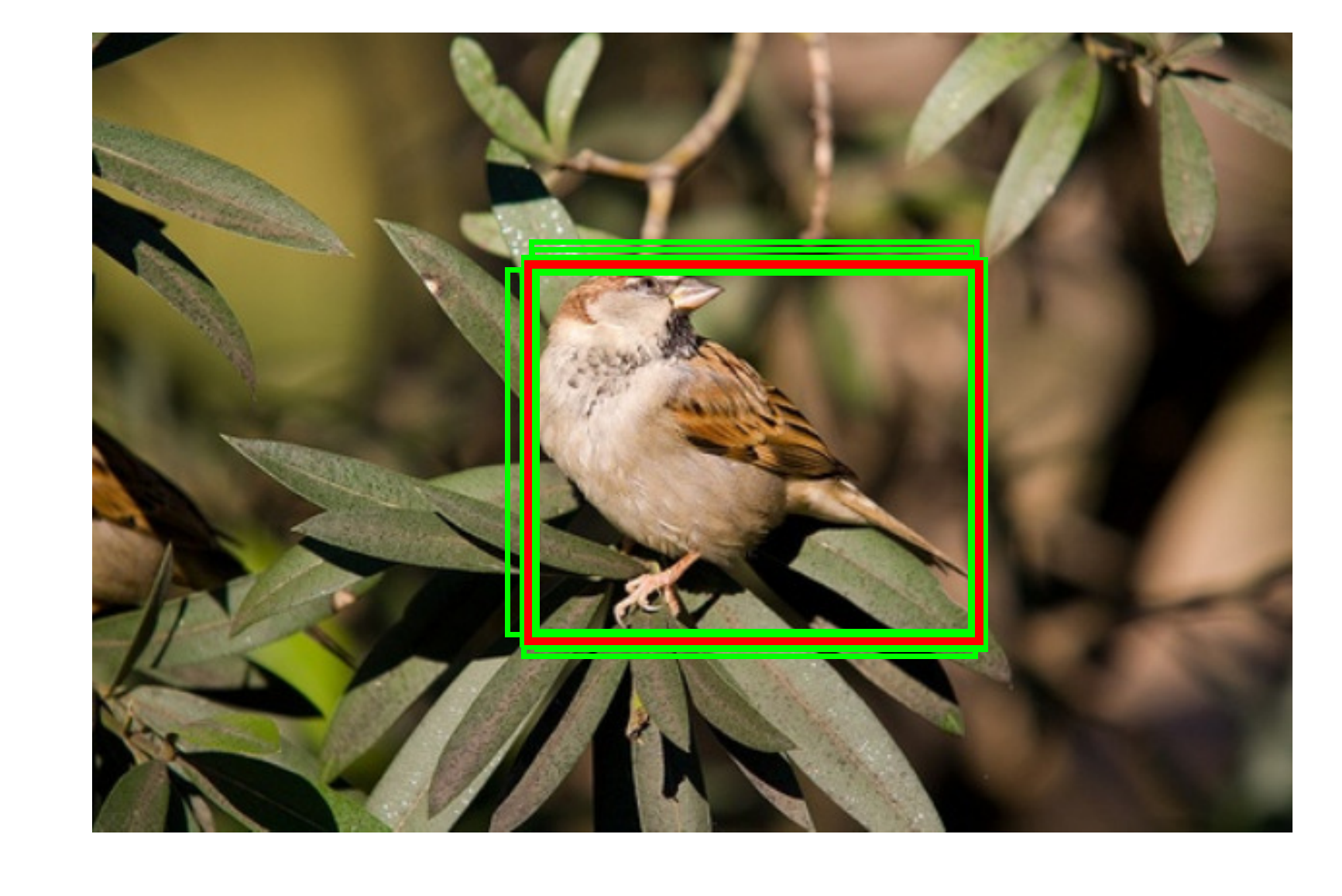}
	\includegraphics[height=2.5cm]{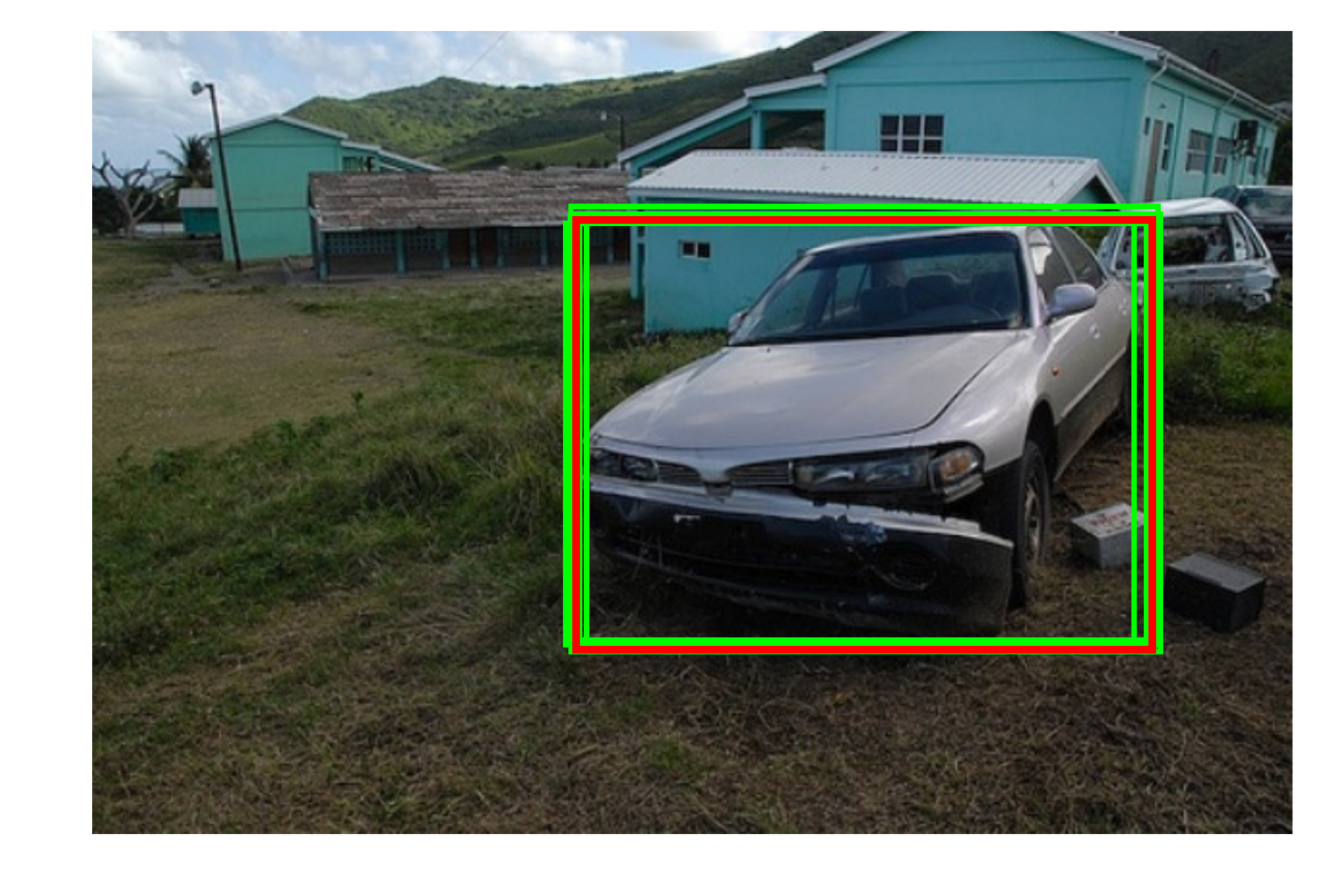}} 
     \subfloat[When $\sigma^2_u  = 1000$ \label{fig:annotation1000}]
     {\includegraphics[height=2.5cm]{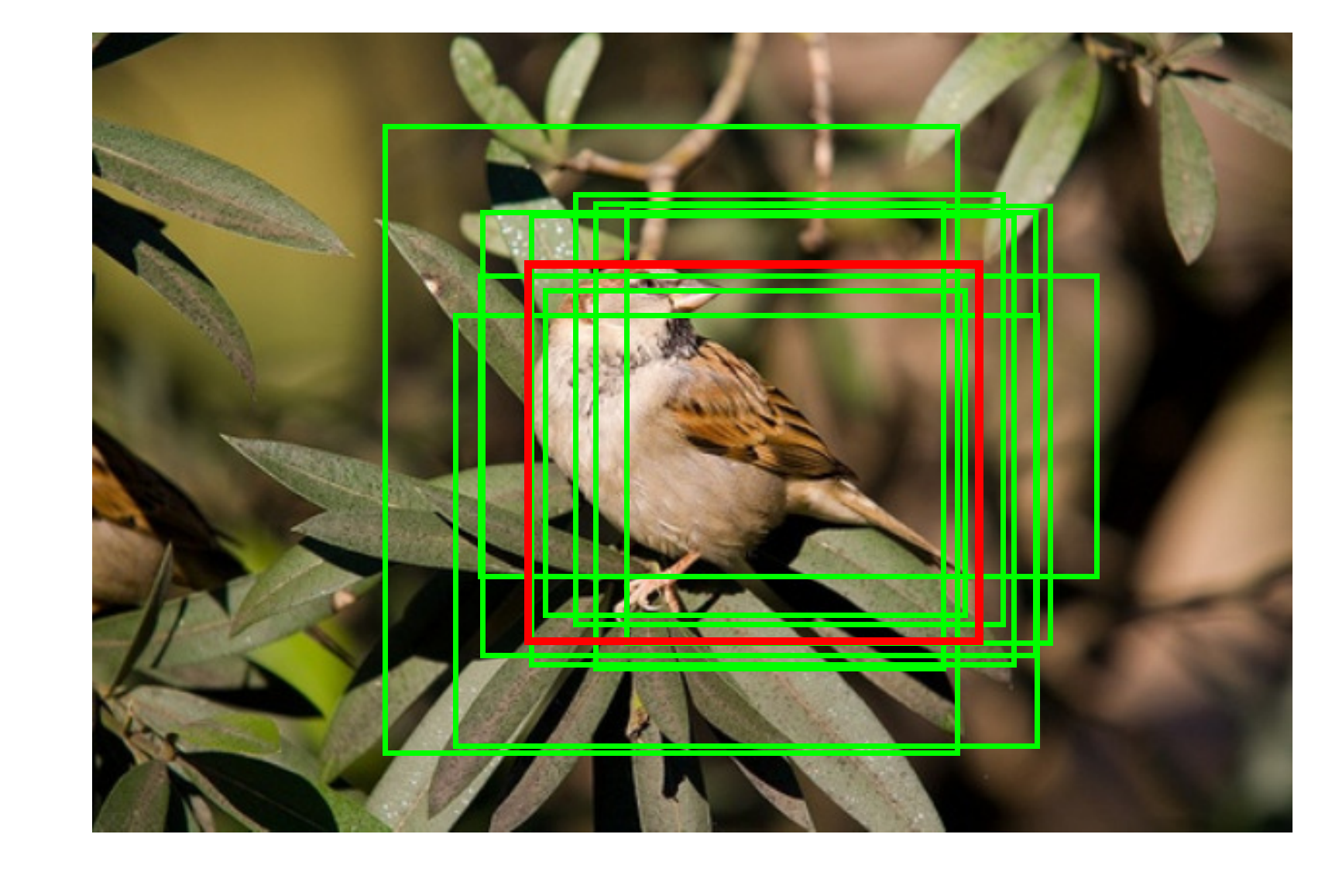}
	\includegraphics[height=2.5cm]{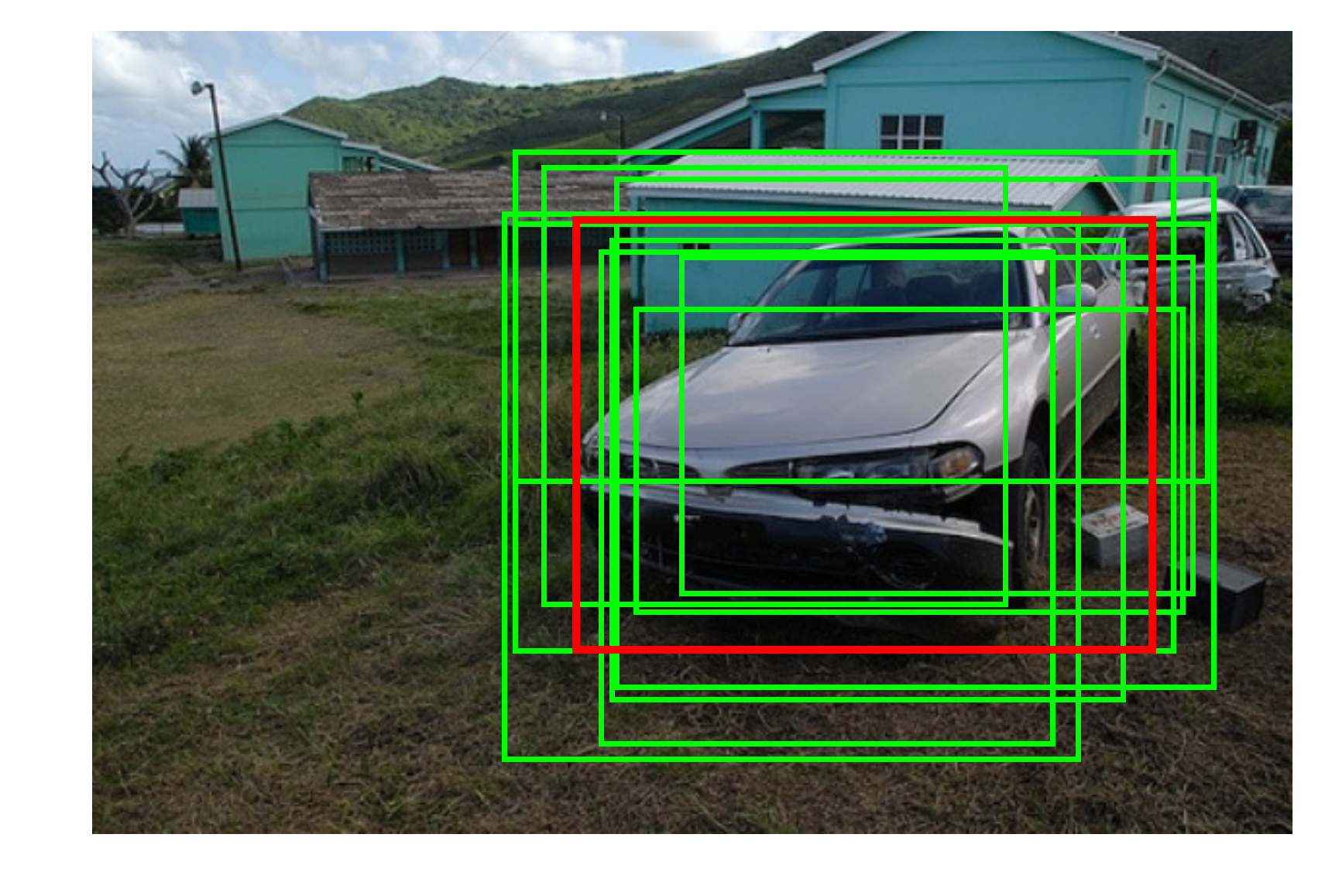}} 
     \end{center}
   \caption{
  Examples of object annotations by a worker $u$ with $\sigma^2_u = 10$ or $1000$.
    %{\MF PASCAL VOC2007} test images.
  }
\label{fig:annotation}%\label{fig:all}
\end{figure*}
%%%%%%%%%%%%%%%%%%%%%%%
% Three figures
%%%%%%%%%%%%%%%%%%%%%%%
\iffalse
  \begin{figure*}[b!]
\begin{center}
     \subfloat[When $\sigma^2_u  = 10$ \label{fig:annotation10}]
     {\includegraphics[height=3cm]{../fig/vision/10_is_not_large1}
	\includegraphics[height=3cm]{../fig/vision/10_is_not_large2}
	\includegraphics[height=3cm]{../fig/vision/10_is_not_large3}} \\
     \subfloat[When $\sigma^2_u  = 1000$ \label{fig:annotation1000}]
     {\includegraphics[height=3cm]{../fig/vision/1000_is_not_large1}
	\includegraphics[height=3cm]{../fig/vision/1000_is_not_large2}
	\includegraphics[height=3cm]{../fig/vision/1000_is_not_large3}} 
     \end{center}
   \caption{
  Examples of object annotations by a worker $u$ with $\sigma^2_u = 10$ or $1000$.
    %{\MF PASCAL VOC2007} test images.
  }
\label{fig:annotation}%\label{fig:all}
\vspace{-0.08in}
\end{figure*}
\fi
%%%%%%%%%%%%%%%%%%%%%%%%%%%

\noindent{\bf Evaluation on visual object detection task.}  %\noindent{\bf Impact of de-noising process.} 
Using each training dataset from four different estimators ({\MF Average}, {\MF \nonIter}, {\MF \bayIter}, {\MF Strong-Oracle}),
we train\footnote{As suggested by \citet{liu2016ssd}, we train SSD using $120,000$ iterations where the learning rate is initialized at $4\times10^{-5}$, and is decreased by factor $0.1$ at $80,000$-th and $100,000$-th iterations.} CNN of single shot multibox detector (SSD) model \cite{liu2016ssd}, which shows the state-of-the-art performance.
Then we evaluate the trained SSD's in terms of the mean average precision (mAP)
which is a popular benchmarking metric for the %PASCAL VOC 
datasets (see Table~\ref{tab:ssd-table}).
Intuitively, a high mAP means more true positive and less false positive detections.

Comparing mAP of {\MF Average},
mAP's of {\MF \bayIter} and  {\MF \nonIter} 
are $4\%$ mAP higher as Figure~\ref{fig:ex1_app} also visually shows the improvement. % in the experiment with {\MF VOC-07/12} datasets. 
Note that achieving a similar amount of improvement is highly
challenging. %, as evidenced in recent extensive research efforts in the literature.
%on
%smarter machine learning algorithms. 
Indeed, Faster-RCNN in
\cite{ren2015faster} is proposed to improve the mAP of Fast-RCNN in
\cite{girshick2015fast} from $70.0\%$ to $73.2\%$.  Later, SSD in
\cite{liu2016ssd} is proposed to achieve $4\%$ mAP improvement over
Faster-RCNN. % $77.2$ mAP).
%%%%%%%%%%%%%%%
%%In addition to the mAP improvement, more accurate training dataset with less MSE leads to more
%%qualified detection with higher overlap
%%ratio, % than others from {\MF Simple} or {\MF Average}.
%%whose examples are provided
%%in Figure~\ref{fig:ex1}.
%in the supplementary material.
%We also present examples SSD detect objects {in the supplementary material.}
%Figure
%\ref{fig:ex1}, where %,\ref{fig:ex2}, and \ref{fig:ex3}.
%we observe that the training dataset from {\MF \nonIter} or {\MF \bayIter}
%enables SSD to not only detect more objects but also draw tighter
%bounding boxes than {\MF Average}.

%\else %
%In addition to the mAP improvement, more accurate training dataset with less MSE leads to more
%qualified detection with higher overlap
%ratio. % than others from {\MF Simple} or {\MF Average}.
%We present how SSD detect objects in Figure
%\ref{fig:ex1}, where %,\ref{fig:ex2}, and \ref{fig:ex3}.
%we observe that the training dataset from {\MF \nonIter} or {\MF \bayIter}
%enables SSD to not only detect more objects but also draw tighter
%bounding boxes than {\MF Average}.
%\fi
%%%%%%%%%%%%%%%

%We further investigate the impact of crowdsourced regression algorithms on real-world machine learning tasks.
%To do so, 

%which is a fundamental problem in the area of computer vision
%The vision task requires a huge amount of the training datasets often obtained by 
 %the crowdsourcing system, e.g., Amazon's mechanical turk \cite{deng2009imagenet}.
 %\noindent{\bf Emulating a crowdsourcing system.} 

%\subsection{Real Datasets} % FG-NET Datasets} 

  \begin{table}[t!]
%    \begin{minipage}{0.6\columnwidth}
\caption{ 
Estimation quality of {\MF Average}, {\MF \nonIter}, {\MF \bayIter}, and {\MF Strong-Oracle}
on crowdsourced {\MF VOC-07/12} datasets from virtual workers in terms of MSE, and 
performance of SSD's trained with
the estimated dataset and ground truth ({\MF VOC-07/12})
in terms of mean average precision (mAP); mean portion of the output bounding box overlapped on the ground truth (Overlap).
% For a reference performance, we also provide mAP when we train SSD using ground truth.
}
      \vspace{-0.2cm}
\label{tab:ssd-table}
\begin{center}
%\begin{small}
{ % \fontsize{9}{7}
\begin{sc}
  %   \begin{minipage}{1\columnwidth}   
\begin{tabular}{cccc}
\hline
% \multirow{2}{*}{Dataset}  & \multirow{2}{*}{Estimator} & \multirow{2}{*}{MSE}  &{Mean Average }	 & \multirow{2}{*}{Overlap} \\
% & & & { Precision} &\\
  \multirow{2}{*}{Estimator}  & 
Data noise  & \multicolumn{2}{c}{Testing accuracy}  \\
%  \cmidrule(lr){3-4}
 & (MSE)  & (mAP)	 & ({Overlap}) \\
\hline
%\multirow{4}{*}{\minitab[c]{VOC 07/12}}
					{\MF Average}	&355.6	& 71.80 	& 0.741\\%/0.122 \\
					{\MF \nonIter}		&116.1	&75.62 	& 0.767 \\ %0.122 \\
					{\MF \bayIter}		&109.8 	& 75.94 	& 0.772 \\ %/0.123\\
					{\MF Strong-Oracle} 		&109.8	& 76.05 	& 0.774 \\ %/0.123\\
					\hline
					\!\!Ground truth\!\!\!\!\! & -		& 77.79  	& 0.784 \\ %/0.123        

\hline
%\multirow{5}{*}{\minitab[c]{VOC \\ 07}}
%					&Average 	&353.1	& 62.70 	& 0.711 	\\ %/0.117 \\
%					& \nonIter 		&170.6	&  66.383 		&0.726 			\\ %/0.117 \\
%					&\bayIter 		&111.2	& 67.52 	& 0.736 	\\ %/0.120\\
%					&Oracle		&111.2	&  67.55 	& 0.736 	\\ %/0.120\\
%					&Ground Truth   		&0		& 69.47 	& 0.746 		\\ %/0.120       \\
%\hline
\end{tabular}
\end{sc}
%\end{small}
}
\end{center}
%\end{minipage}
%\begin{minipage}{0.4\columnwidth}
%{
%\captionsetup[subfloat]{captionskip=2pt}
%%\begin{figure}[t!]
%%\vspace{-0.6cm}
%\begin{center}
%     \subfloat[$\sigma^2_u  = 10$ \label{fig:annotation10_intext}]
%     {	\includegraphics[width=0.45\columnwidth, trim={5cm 1cm 0.5cm 1cm},clip]{../fig/vision/10_is_not_large3}}
%     \subfloat[$\sigma^2_u  = 1000$ \label{fig:annotation1000_intext}]
%     {	\includegraphics[width=0.45\columnwidth, trim={5cm 1cm 0.5cm 1cm},clip]{../fig/vision/1000_is_not_large3}} 
%     \end{center}
%\captionof{figure}{
% An example of object annotations by workers' with $\sigma^2_u = 10$ or $1000$.
%  }
%\label{fig:annotation_intext_app}
%}
%\end{minipage}
\end{table}

{
\begin{figure*}[t]
\captionsetup[subfloat]{captionskip=1pt}
%\vspace{-0.3cm}
   %\begin{center}
%\begin{minipage}{0.5\columnwidth}
     \begin{minipage}{0.245\columnwidth}   
     \subfloat[{\MF Average} \label{fig:ex1-avg}]{\includegraphics[width=1\columnwidth]{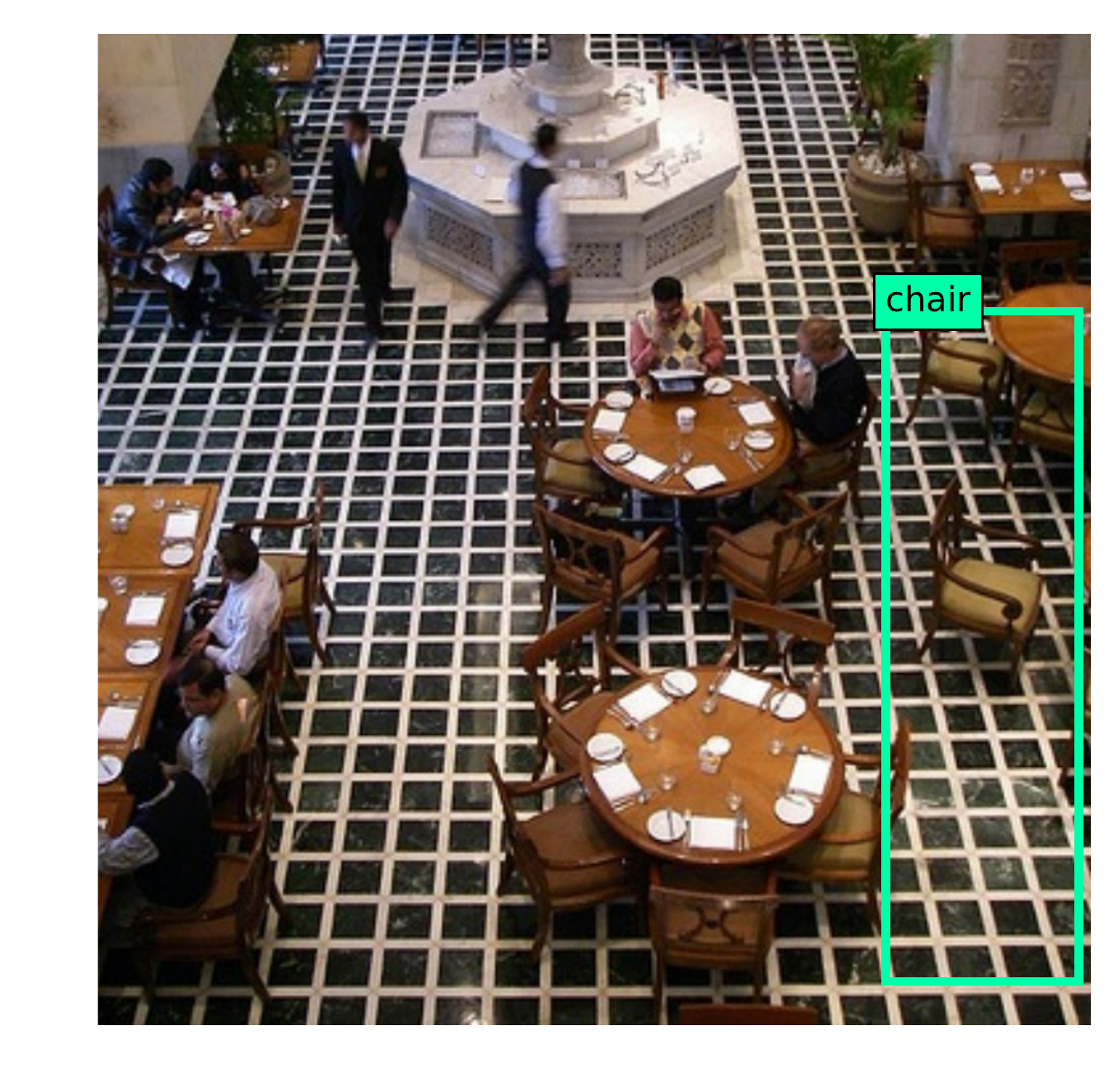}}
               \end{minipage}
          \begin{minipage}{0.245\columnwidth}   
     \subfloat[{\MF \nonIter} \label{fig:ex1-nav}]{\includegraphics[width=1\columnwidth, trim={0.5cm 0.15cm 0cm 0cm},clip]{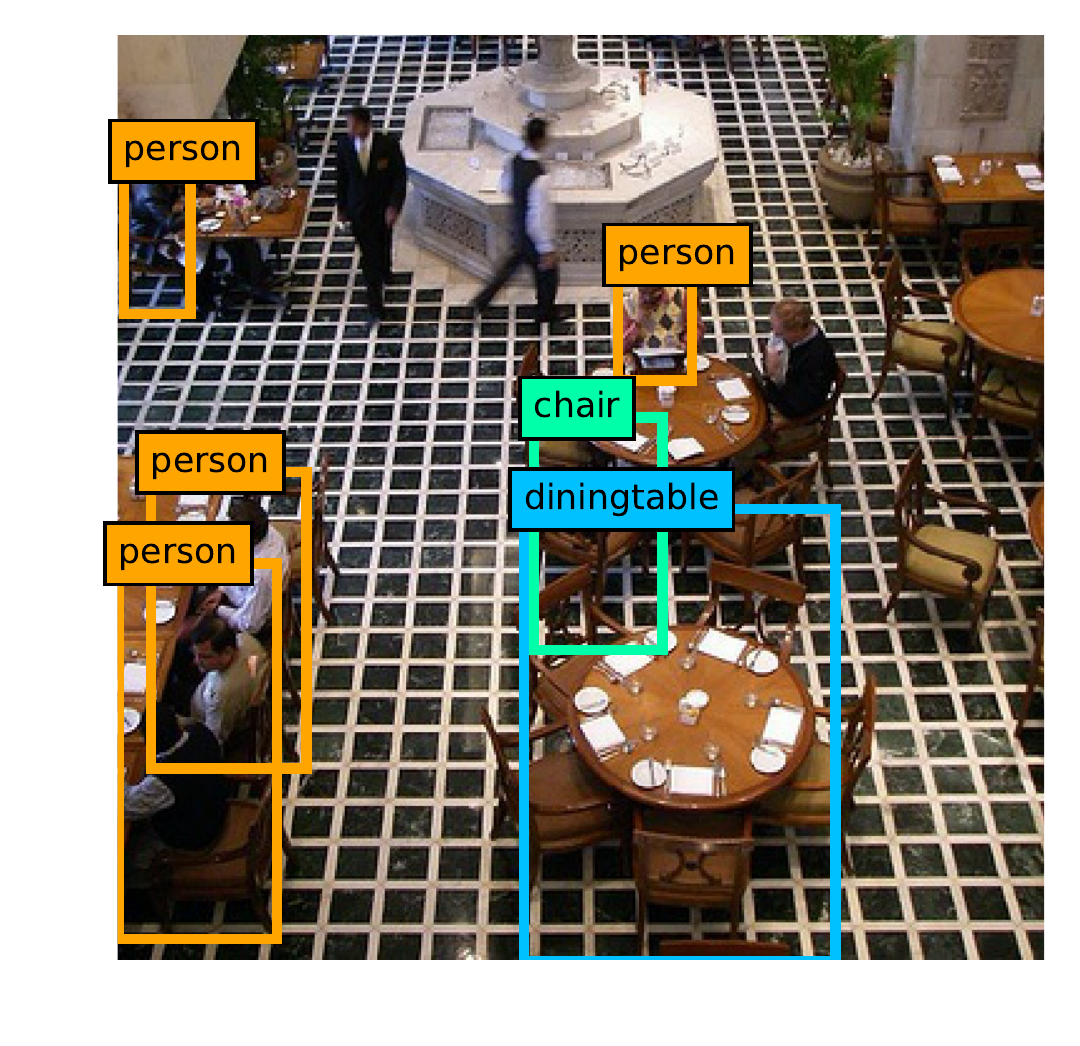}} 
          \end{minipage} 
          \begin{minipage}{0.245\columnwidth}   
     \subfloat[{\MF \bayIter}\label{fig:ex1-bp}]{\includegraphics[width=1\columnwidth]{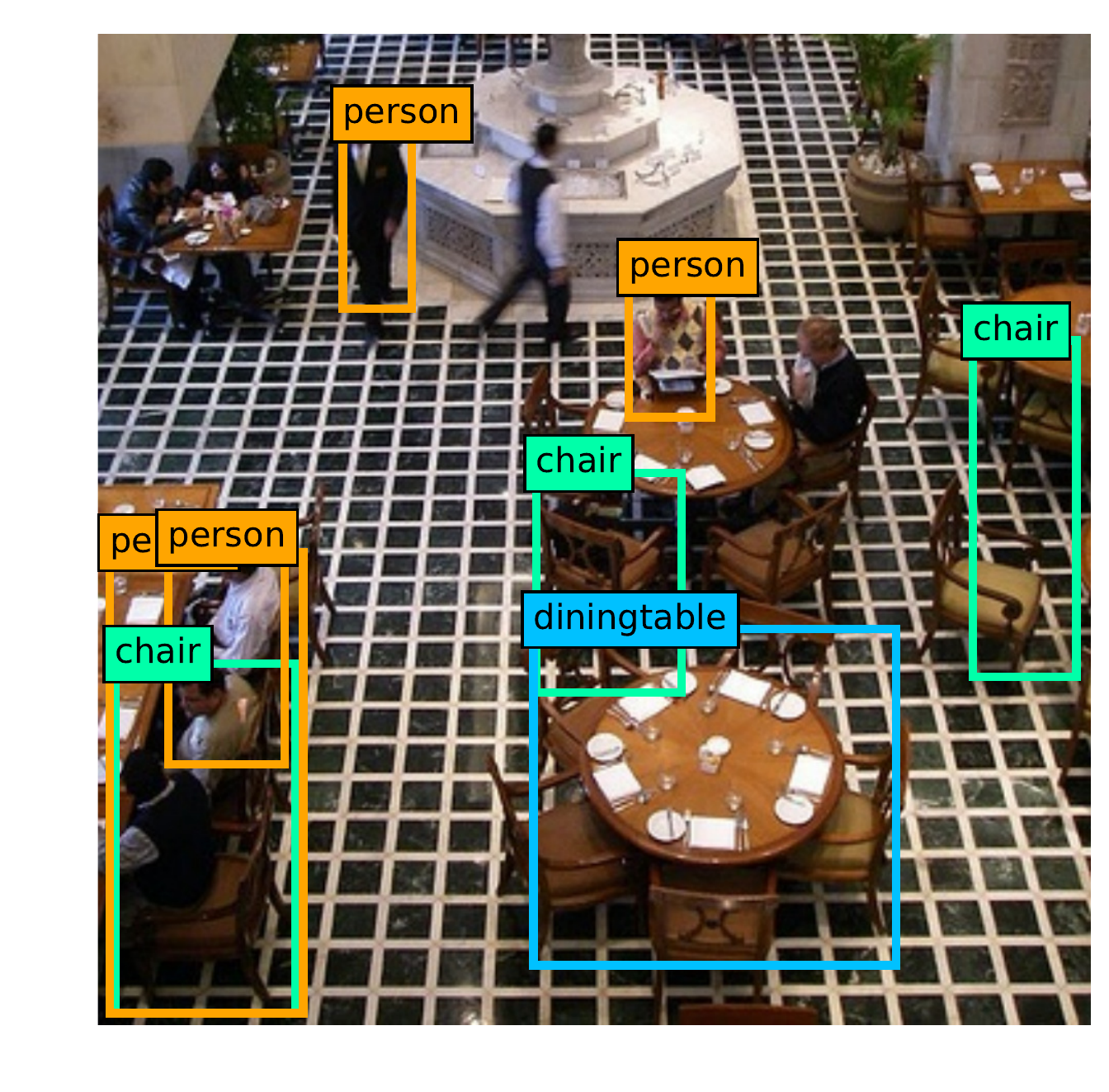}} 
     \end{minipage}       
     \begin{minipage}{0.245\columnwidth}   
     \subfloat[{\MF Strong-Oracle}\label{fig:ex1-ora}]{\includegraphics[width=1\columnwidth]{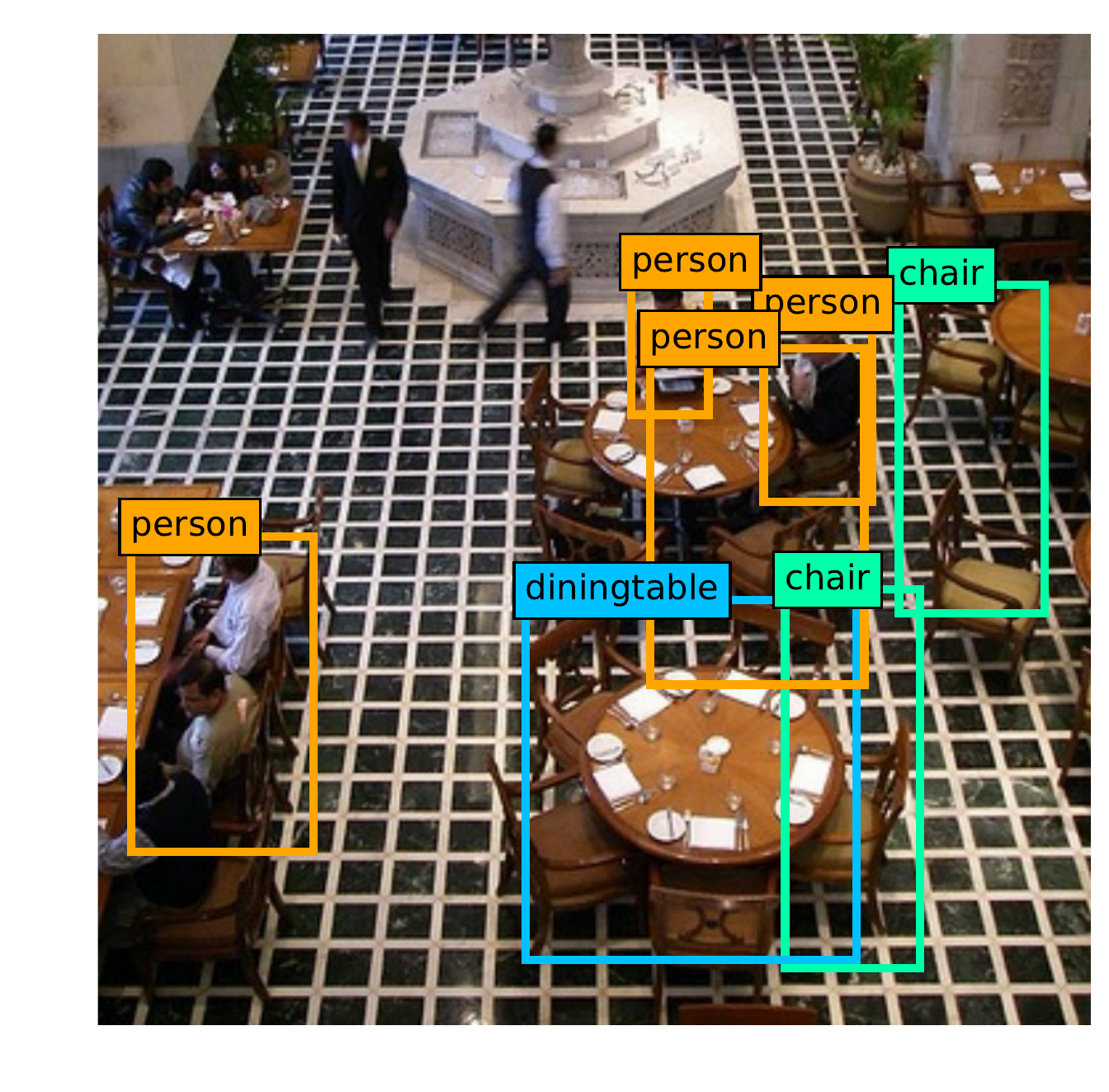}} 
 \end{minipage}  
   \caption{
  Examples of detections of SSD trained by the crowdsourced 
  {\MF VOC-07/12} 
  datasets  by {\MF Average}, {\MF \nonIter}, {\MF \bayIter}, and {\MF Strong-Oracle}.
  }
\label{fig:ex1_app}
\end{figure*}
}

%\noindent {\bf Human age prediction.} 
\subsection{Human Age Prediction} 

\label{sec:HAP}
%To evaluate performance of our methods in real-world scenario, 
%we use
\noindent{\bf Real-world dataset.} 
We also perform similar experiment using datasets from a {\em real-world} crowdsourcing system.
We use {\MF FG-NET} datasets
which has been widely used as a benchmark dataset for facial age estimation \cite{lanitis2008comparative}.
The dataset contains $1,002$ photos of $82$ individuals' faces, in which
each photo has biological age %of individual when the photo was taken
as ground truth.
% and $10$ answers from workers, while a worker provides 
Furthermore, \cite{han2015demographic} provide 
crowdsourced labels on {\MF FG-NET} datasets,
in which $165$ workers in Amazon Mechanical Turk (MTurk) answer 
their own age estimation on given subset of $1,002$ photos
so that each photo has 
%biological age %of individual when the photo was taken
%as ground truth and 
$10$ answers from workers, while each worker provides a different number (from $1$ to $457$) of
 answers, and $60.73$ answers in average.

In the dataset, we often observe two extreme classes of answers for a task: 
a few outliers and consensus among majority. 
For example, for Figures~\ref{fig:face_crowd_app}\subref{fig:easy} and \ref{fig:face_crowd_app}\subref{fig:hard},
there exist noisy answers $5$ and $7$, respectively, which are far from majority $1$ and $55$, respectively.
Such observations suggest to choose a simple support, e.g.~$\mathcal{S}=\{  \sigma^2_{\rm  good}, \sigma^2_{\rm bad}\}$.
In particular, without any use of ground truth, we first run  {\MF {\nonIter}} 
and use the top $10\%$ and bottom $10\%$ workers' reliabilities as the binary support, which is
$\mathcal{S}_{\textnormal{est}} = \{6.687, 62.56\}$ in our experiment.

\noindent{\bf Evaluation on human age prediction task.}
We first compare the estimation of {\MF {\bayIter}} to other algorithms as reported in Table~\ref{tab:fg-net-table_app}.
Observe that MSE of {\MF {\bayIter}} with the binary support $\mathcal{S}_{\textnormal{est}}$ 
is close to that of {\MF Strong-Oracle}, while the other algorithms have some gaps.
This result from real workers supports the idea of simplified workers' noise level in our model.
We also evaluate the impact on de-noising process for human age prediction.
To this end, using the pruned datasets from different estimators,
we train\footnote{ We train VGG-16 using batch normalization with standard hyper parameter setting, where we initialize based on the imagenet pre-trained model. To regressed the estimated age of the given face images, we replaced final layer of VGG-16 with one dimensional linear output layer, and fine-tuned all the layers with initial learning rate $0.01$ (and divided by $10$ after $30, 60, 90$ epoch). Protocol of measuring model performance is standard Leave One Person Out (LOPO) which uses images of $81$ subjects for training and use remaining subject for test, and the final result is averaged over the total $82$ model training \cite{panis2016overview}.
} one of the state-of-art CNN models, called VGG-16 \cite{simonyan2014very},
under some modification proposed by \cite{rothe2015dex} for human age prediction.
Although the crowdsourced dataset {\MF {FG-NET}} is 
not large-scale in order to see performance difference,
models trained by both {\MF {\bayIter}} and {\MF {\nonIter}} show superiority to that of {\MF {Average}} 
(which is widely used in practical crowdsourcing systems),
in terms of %mean and 
median absolute errors (MDAE), %(MAE and MDAE, respectively),
as reported in Table~\ref{tab:fg-net-table_app}.

 \begin{table}[t]
      \vspace{-0.3cm}
\begin{minipage}{0.6\columnwidth}
\caption{ 
Estimation quality of {\MF Average}, {\MF \nonIter}, {\MF \bayIter} with $\mathcal{S}_{\textnormal{est}}$, and {\MF Strong-Oracle}
on crowdsourced {\MF FG-NET} datasets from Amazon MTurk workers
in terms of MSE,
and 
performance of VGG-16's trained with the estimated datasets 
and the ground truth dataset ({\MF FG-NET})
in terms of
%mean absolute error (MAE); 
median absolute error (MDAE).
%Performance of VGG-16's trained by
%the ground truth datasets ({\MF FG-NET}) 
%and the crowdsourced datasets using
%{\MF Average}, {\MF \nonIter}, {\MF \bayIter}, and {\MF Oracle} 
%in terms of mean squared errors (MSE); mean absolute error (MAE).
%We also provide the result using ground truth for a reference performance.
%For a reference performance, we also provide MAE when we train VGG-16 using ground truth.
}
      \vspace{-0.2cm}
\label{tab:fg-net-table_app}
\begin{center}
%\begin{small}
{  % \fontsize{9}{9}
\begin{sc}
  %   \begin{minipage}{1\columnwidth}   
\begin{tabular}{cccc}
\hline
% \multirow{2}{*}{Estimator}  & \multirow{2}{*}{Estimator} & \multirow{2}{*}{MSE} 	\\
  \multirow{2}{*}{Estimator}  & Data noise  & \!\!Testing error\!\! \\ %\multicolumn{2}{c}{Testing error} \\
{} & ({MSE})  %&  (MAE)	
& (MDAE) \\
\hline
%\multirow{4}{*}{\minitab[c]{VOC 07/12}}
					{\MF Average}		&34.99	%& 5.006 	
							& 3.227 \\
					{\MF \nonIter}		&32.80	%&4.687 	 
							& 3.135 \\
					 {\MF \bayIter}		&28.72   %& 4.773 	 
					 		& 3.100\\
					{\MF Strong-Oracle} 		&28.45	%& 4.623 	
							 & 3.003\\
					\hline
					\!\!Ground truth\!\!\!\!\! & -	 %&	XX 
							&  1.822	 \\ %/0.123        

\hline
%AVG: 5.005607e+00 
%AVG: 3.226521e+00 
%BP1: 4.773906e+00 
%BP1: 3.100259e+00 
%BP2: 4.386465e+00 
%BP2: 2.890493e+00 
%Iter: 4.687298e+00 
%Iter: 3.135448e+00 
%Ora: 4.623180e+00 
%Ora: 3.002691e+00
\end{tabular}
\end{sc}
%\end{small}
}
%\vspace{cm}
\end{center}
\end{minipage}
\begin{minipage}{0.4\columnwidth}
{
\captionsetup[subfloat]{captionskip=3pt}
%\begin{figure}[t!]
%\begin{minipage}{0.4\columnwidth}
\begin{center}
     \subfloat[1-year-old \label{fig:easy}]
     {	\includegraphics[height=0.4\columnwidth, trim={0cm 0.1cm 0.1cm 0cm},clip]{fig/vision/010A01_easy}} 
     \hspace{0.8cm}
     \subfloat[35-year-old\label{fig:hard}]
     {	\includegraphics[height=0.4\columnwidth, trim={0cm 0.1cm 0cm 0.1cm},clip]{fig/vision/028A35_hard}} 
     \end{center}
   \captionof{figure}{
Easy and hard samples from {\MF FG-NET} in terms of
average absolute error of crowd workers' answers:
(1,1,1,1,1,1,1,1,2,5) and (7,51,52,55,55,55,59,63,66,67)
on photo of (a) 1-year-old, and (b) 35-year-old, resp. 
  }
\label{fig:face_crowd_app}%\label{fig:all}
% easiest: 046A04.JPG
% AVG: 7.761955e-02
% BP1: 3.440957e-01
% BP2: 1.571412e-01
% Iter: 2.708888e-01
% Ora: 4.885221e-01
%     hardest: 032A15.JPG
% AVG: 2.245510e+01
% BP1: 2.125741e+01
% BP2: 2.334494e+01
% Iter: 2.023362e+01
% Ora: 2.572180e+01
%\end{minipage}
%\end{figure}
}
\end{minipage}
     \vspace{-0.2cm}
\end{table}

\section{Model Derivations}
\label{sec:model_calc}
\vspace{-0.2cm}

\subsection{Calculation of $\bar{\mu}$ }%\eqref{eq:optimal_mu_given_sigma}}
\label{sec:model_calc1}
\vspace{-0.1cm}

We first show that the posterior 
%Using Bayes' theorem and the independence between $\mu_i$ and $\sigma^2$,
%the conditional marginal 
density of $\mu_i$
given $A_i = y_i := \{y_{iu} \in \mathbb{R}^{d } : u \in M_i\}$ and $\sigma^2_{M_i}$
 is a Gaussian density in the following: 
\begin{align}
&    f_{\mu_i}[x   \mid  A_i = y_{i}, \sigma^2_{M_i}]  =  
 \frac{f_{\mu_i}[x] \, f_{A_i} \left[y_i  \mid  \mu_i = x_i, \sigma^2_{M_i}\right] }
 { f_{A_i} \left[y_i  \mid   \sigma^2_{M_i}\right] } \label{eq:marginal} \\
 %{\int  \Pr[\mu_i] \cdot \Pr \left[A_i \longmid \mu_i, \sigma^2_{M_i}\right] d\mu_i} \label{eq:marginal} \\
% & =
%\mathcal{N}(\mu_{i}  \mid \nu_{i}, \tau^2_{u} )
% \prod_{u \in M_i}  \mathcal{N}(A_{iu}  \mid \mu_{i}, \sigma^2_{u} ) \\
 &\hspace{2.6cm} = \phi \left(x   \mid  \bar{\mu}_i\left( y_i, \sigma^2_{M_i}\right), \bar{\sigma}^2_i\left(\sigma^2_{M_i}\right) \right) \label{eq:marginal_calc} 
\end{align}
where we define $\bar{\sigma}^2_i : \mathcal{S}^{M_i}\to \mathbb{R}$ and  $\bar{\mu}_i : \mathbb{R}^{d \times M_i} \times \mathcal{S}^{M_i} \to \mathbb{R}^d$ as follows
\begin{align*}
\bar{\sigma}^2_i \left(\sigma^2_{M_i} \right)
:= \frac{1}{\frac{1}{\tau^2} + \sum_{u \in M_i}  \frac{1}{\sigma^2_{u}}} 
\;,\quad \text{and} \quad
\bar{\mu}_i\left(A_i, \sigma^2_{M_i} \right)
 := \bar{\sigma}^2_i \left(\sigma^2_{M_i}\right)
  \bigg( \frac{\nu_i}{\tau^2} + \sum_{u \in M_i} \frac{A_{iu}}{\sigma^2_{u}} \bigg) \;.
\end{align*}
The Gaussian posterior density \eqref{eq:marginal_calc} follows from:
\begin{align*}
f_{\mu_i}[x]  f_{A_i} \left[y_i \longmid \mu_i = x, \sigma^2_{M_i}\right] 
&=
\phi (x  \mid \nu_{i}, \tau^2 )
 \prod_{u \in M_i}  \phi(y_{iu}  \mid \mu_{i}, \sigma^2_{u} ) \\
&= \mathcal{C}_i \left(y_i, \sigma^2_{M_i} \right)
 \phi \left(x  \longmid \bar{\mu}_i\left(y_i, \sigma^2_{M_i}\right), \bar{\sigma}^2_i\left(\sigma^2_{M_i}\right) \right) 
\end{align*}
where we have $f_{A_i} [y_i \mid \sigma^2_{M_i}]  = \mathcal{C}_i (y_i, \sigma^2_{M_i} )$
with 
%\begin{align*}
%f_{A_i} [y_i \mid \sigma^2_{M_i}]  = \mathcal{C}_i (y_i, \sigma^2_{M_i} ) \; . 
% \label{eq:marC}
%\end{align*}
%with 
%we define $\mathcal{C}_i : \mathbb{R}^{d \times M_i} \times \mathcal{S}^{M_i} \to \mathbb{R}$ and
%$\mathcal{D}_i : \mathbb{R}^{d \times M_i} \times \mathcal{S}^{M_i} \to \mathbb{R}$ as follows
\begin{align*}
\mathcal{C}_i (A_i, \sigma^2_{M_i} ) 
&~:=~  {\bigg(\frac{ 2\pi  \bar{\sigma}^2_i (\sigma^2_{M_i})}{ 2\pi \tau^2_i  \prod_{u \in M_i} (2\pi \sigma^2_u)}\bigg)^{\frac{d}{2}}}  e^{-\mathcal{D}_i (A_i, \sigma^2_{M_i} ) }, ~\text{and}
\end{align*}
\begin{align*}
 \mathcal{D}_i (A_i, \sigma^2_{M_i} ) 
&~:=~ \frac{1}{2} \bar{\sigma}^2_i (\sigma^2_{M_i}) 
\Bigg(
\sum_{u \in M_i}  \frac{\|A_{iu} - \nu_{i}\|_2^2}{\sigma_u^2 \tau^2} 
+
\sum_{v \subset M_i \setminus \{u\}} \frac{\|A_{iu} - A_{iv}\|_2^2}{\sigma_u^2 \sigma_v^2}
 \Bigg) \; .
\end{align*}
%\note{Please give me an idea for writing $\mathcal{C}_i \left(A_i, \sigma^2_{M_i} \right)$... }

The Gaussian density in \eqref{eq:marginal_calc} leads to the posterior mean, which is weighted average of the prior mean 
and the worker responses, each weighted by the inverse of its variance:  
%we note that
%$\mu_i$ in \eqref{eq:marginal_calc} is
%%$\Pr[\mu_i \mid A_i, \sigma^2_{M_i}]$ 
%the Gaussian with mean $\bar{\mu}_i (A_i, \sigma^2_{M_i})$, i.e., 
\begin{align*}
\EXP[\mu_i \mid A_i, \sigma^2_{M_i}] = \bar{\mu}_i (A_i, \sigma^2_{M_i}) \; . %\label{eq:conditional_mu}
\end{align*}
Thus, the optimal estimator $\hat{\mu}_i^*(A)$ is given 
as \eqref{eq:MAP}.
%by 
%\begin{align*}
%\hat{\mu}^*_i (A)  = \sum_{\sigma^2_{M_i} \in \mathcal{S}^{M_i}} \bar{\mu}_i (A_i, \sigma^2_{M_i}) \cdot \Pr [\sigma^2_{M_i} \mid A] \;.
%\end{align*}

%\subsection{Calculation of \eqref{eq:sub_factor_graph}}
\subsection{Factorization of Joint Probability}
\label{sec:model_calc2}
\vspace{-0.1cm}

Using Bayes' theorem, it is not hard to write 
the joint probability of $\sigma^2$ given $A = y = \{y_{iu} \in \mathbb{R}^d : (i,u) \in E\}$,
\begin{align*}
\Pr[ \sigma^2 \mid A = y] 
 \propto f_{A}[y \mid \sigma^2]  =  \prod_{i \in V}  f_{A_i}[y_i \mid \sigma^2_{M_i}]   = \prod_{i \in V} \mathcal{C}_i (y_i, \sigma^2_{M_i} ) \; .
 \end{align*}
 
% , it follows that 
%$\Pr[ \sigma^2 \mid A] %& \propto \Pr[A\mid \sigma^2] =  \prod_{i \in V}  \Pr[A_i \mid \sigma^2_{M_i}] \nonumber \\
%\propto \prod_{i\in V}
% \mathcal{C}_i \left(A_i, \sigma^2_{M_i} \right).  %\label{eq:sub_factor_graph}
%$
%\begin{align*}
%\Pr[ \sigma^2 \mid A] %& \propto \Pr[A\mid \sigma^2] =  \prod_{i \in V}  \Pr[A_i \mid \sigma^2_{M_i}] \nonumber \\
%\propto \prod_{i\in V}
% \mathcal{C}_i \left(A_i, \sigma^2_{M_i} \right) \;.  %\label{eq:sub_factor_graph}
%\end{align*}
%Hence, the conditional probability of $\sigma^2_{M_i}$ given $A$
%%$\Pr[\sigma^2_u \mid A]$ 
%in \eqref{eq:MAP} 
%can be calculated by
% marginalizing out
%$\sigma^2_{-i} := \{\sigma^2_v : v \in W \setminus M_i \}$ 
%from the joint probability of $\sigma^2$,
%i.e.,
%\begin{align}
%  \Pr[\sigma^2_{M_i} \mid A] &= \sum_{\sigma^2_{-i} \in  \mathcal{S}^{W \setminus M_i} }  \Pr[\sigma^2 \mid A] ,  \label{eq:marginal_sum}
%\end{align}
%where 
%The summation in \eqref{eq:marginal_sum} is taken over exponentially
%many $\sigma^2_{-i} \in \mathcal{S}^{m-|M_i|}$ with respect to $m$. Thus,
%in general, the optimal estimator $\hat{\mu}^*(A)$ in
%\eqref{eq:MAP_calc}, requiring the probability of $\sigma^2_{M_i}$ given $A$ in \eqref{eq:marginal_sum},
% is {\it  computationally intractable}.

%%% Local Variables:
%%% mode: latex
%%% TeX-master: "main"
%%% End:

 % !TEX root =  main.tex

%\section{Proofs of Lemmas}
%\label{sec:lem_pf}

\section{Proofs of Lemmas}
\subsection{Proof of Lemma~\ref{lem:concentration}}
\label{sec:concentration-pf}

We first
introduce an inference problem
and connect its error rate to the expectation of likelihood of
worker $\rho$'s $\sigma^2_{\rho}$ given $A$.
Let $s_\rho \in \{1, \ldots, S\}$ be the index of $\tilde{\sigma}^2_{\rho}$,
i.e., $\tilde{\sigma}^2_{\rho} = \sigma^2_{s_\rho}$.
 Consider the classification problem recovering given but latent $s$
 %$\sigma^2_\rho = \sigma^2_{s} \in \mathcal{S}$
from $A_{\rho, 2k}$, where $A_{\rho, 2k}$ is generated from the crowdsourced regression model with fixed but hidden $\sigma^2 = \tilde{\sigma}^2$.
More formally, the problem is formulated as the following optimization problem:
\begin{align} \label{eq:classification_minimize}
\underset{\hat{s}_\rho:\text{estimator}}{\text{minimize}}~ \Pr[s_\rho \neq \hat{s}_\rho(A_{\rho, 2k})]
\end{align}
where the optimal estimator, denoted by $\hat{s}^*_\rho$, minimizes the classification error rate.
From the standard Bayesian argument,
the optimal estimator $\hat{s}^*_\rho$
is given $A_{\rho, 2k}$ as
\begin{align}\label{eq:estimating_index}
 \hat{s}^*_\rho(A_{\rho, 2k})  &: = \underset{s'_\rho = 1, \ldots, S}{\arg\max}~ \Pr[ s_\rho =  s'_\rho \mid A_{\rho, 2k} ] \;.
 \end{align}
%  \\
%  &= \underset{s' = 1, \ldots, S}{\arg\min}~ \Pr[ \sigma^2_{\rho}\neq \sigma^2_{s'} \mid A_{\rho, 2k} ] \;.
%\end{align*}
%Since the optimal estimator always maximizes the likelihood of $s_\rho$
%only given $A_{\rho, 2k}$,
From the construction of the optimal estimator in \eqref{eq:estimating_index},
it is not hard to check
%we can write
%the error rate of $\hat{s}^*$ as follows: 
\begin{align}
%\min_{\hat{s}_\rho} 
%\EXP[\Pr{[s_\rho = \hat{s}_\rho(A_{\rho, 2k})]} \mid ]  &=
 \Pr_{ \tilde{\sigma}^2}[s_\rho = \hat{s}^*_\rho(A_{\rho, 2k})]  
~: =~ \Pr[s_\rho = \hat{s}^*_\rho(A_{\rho, 2k}) \mid {\sigma}^2 = \tilde{\sigma}^2]
 ~\le~ \EXP_{\tilde{\sigma}^2} \left[ \Pr[\sigma^2_\rho = \tilde{\sigma}^2_\rho \mid  A_{\rho, 2k}] \right] \;. 
%\nonumber \\
%& = \EXP[b^k_{\rho}(\sigma^2_{\rho}) \mid \sigma^2_\rho]\;. 
\label{eq:connection_b_s} 
\end{align}
%where the conditional expectation is taken over $A_{\rho, 2k}$ conditioned on $\sigma^2_\rho = \sigma^2_s$
%and for the last equality, we use \eqref{eq:bp-rho-2k}.
Thus 
 an upper bound of the average error rate of an estimator for \eqref{eq:classification_minimize}
 will provide an upper bound of 
$\EXP_{\tilde{\sigma}^2} \left[ \Pr[\sigma^2_\rho \neq \tilde{\sigma}^2_\rho \mid  A_{\rho, 2k}] \right]$ since the optimal estimator minimizes 
the average error rate. Indeed, we have
\begin{align*}
\EXP_{\tilde{\sigma}^2} \left[ \Pr[\sigma^2_\rho \neq \tilde{\sigma}^2_\rho \mid  A_{\rho, 2k}] \right] 
& \le 1- \EXP_{\tilde{\sigma}^2} \left[ \Pr[\sigma^2_\rho = \tilde{\sigma}^2_\rho \mid  A_{\rho, 2k}] \right]\\
&=
\EXP [ \Pr_{ \tilde{\sigma}^2}[s_\rho \neq \hat{s}^*_\rho(A_{\rho, 2k})]  ] \\
& = 
 \min_{\hat{s}_\rho : \text{estimator}} \Pr[s_{\rho} \neq \hat{s}_\rho (A_\rho, 2k) ] \;.
\end{align*}
%$\EXP[b^k_{\rho}(\sigma^2_{\rho}) \mid \sigma^2_\rho]$.
%To prove Lemma~\ref{lem:concentration}, we will analyze the accuracy of
Consider a simple estimator for \eqref{eq:classification_minimize}, denoted by $\hat{s}^{\dag}_\rho$, which uses
only $A_{\rho, 2} \subset A_{\rho, 2k}$ as follows:
\begin{align}
 \hat{s}^{\dag}_\rho (A_{\rho, 2}) 
 = \underset{s'_\rho = 1,\ldots, S}{\arg \min} \left| (\sigma^2_{s'_\rho} + \sigma^2_{\MF avg}(\mathcal{S}) ) - \hat{\sigma}^2(A_{\rho, 2k}) \right|
\end{align}
where we define
\begin{align*}
\sigma^2_{\MF avg}(\mathcal{S}) := \frac{\sum_{s' = 1, \ldots, S} \sigma^2_{s'}}{S(\ell -1)}  \;,
\hat{\sigma}^2(A_{\rho, 2}) := \frac{1}{r} \sum_{i \in N_\rho} \hat{\sigma}^2_i(A_{i})  \;,  \text{and}
\;
\hat{\sigma}^2_i(A_{i}) := \left\|\frac{\sum_{u \in M_{i} \setminus \{\rho\}} A_{iu}}{\ell-1}  - A_{i\rho} \right\|^2_2 \;.
\end{align*}

From now on, we condition $\sigma^2_{\partial^2 \rho}$ additionally to $\sigma^2_\rho$ 
where $\partial^2 \rho$ is the set of $\rho$'s 
grandchildren in $G_{\rho, 2}$. 
For every $i \in N_\rho$, we define
\begin{align*}
a_i := \sum_{u\in M_i \setminus \{\rho\}} \frac{\tilde{\sigma}^2_{u}}{(\ell- 1)^2} + \tilde{\sigma}^2_\rho \; , ~\text{and}~
Z_i := \frac{\sum_{u \in M_{i} \setminus \{\rho\}} A_{iu}}{\ell-1}  - A_{i\rho} \;.
\end{align*}
Since the conditional density of $Z_i$ given $\sigma^2 = \tilde{\sigma}^2$ is $\phi (Z_i \mid 0, a_i)$,
the conditional density of $\|Z_i\|^2_2 /a_i$ is $\chi^2$-distribution
 with degree of freedom $d$. 
In addition, it is not hard to check that $\|Z_i\|^2_2$ is sub-exponential with parameters
$((2a_i \sqrt{d})^2, 2a_i)$ such that for all $|\lambda| < \frac{1}{2a_i}$,
\begin{align*}
  \EXP_{\tilde{\sigma}^2} \left[ \exp\left({\lambda \left(\|Z_i\|^2_2 - d a_i\right)} \right) \right] 
= \left( \frac{e^{-a_i \lambda}}{\sqrt{1-2 a_i \lambda}} \right)^d 
 \le \exp\left( \frac{(2a_i \sqrt{d})^2 \lambda^2 }{2} \right)\;.
\end{align*}
%where we let $\EXP_\rho$ denote the conditional expectation given
 %$\sigma^2_{\rho}$ and $\sigma^2_{\partial^2 \rho}.$
 Thus it follows that for all $|\lambda| \le  \min_{i \in N_\rho}\frac{1}{2a_i}$,
  \begin{align*}
  \EXP_{\tilde{\sigma}^2} \left[\exp\left(\lambda \sum_{i \in N_\rho} \left(\|Z_i\|^2_2 - da_i\right)\right) 
\right] 
  &= \prod_{i\in N_\rho}
    \EXP_{\tilde{\sigma}^2} \left[\exp\left(\lambda\left(\|Z_i\|^2_2 - da_i\right)\right) \right] \\
&    \le \prod_{i\in N_\rho} \exp\left( \frac{(2a_i \sqrt{d})^2 \lambda^2 }{2} \right).
  \end{align*}
From this, it is straightforward to check that $r \hat{\sigma}^2(A_{\rho, 2}) = \sum_{i \in N_\rho} \|Z_i\|^2_2$ is
sub-exponential with parameters $( (6\sigma^2_{\max} \sqrt{d})^2  , 6 \sigma^2_{\max} )$ since
 \begin{align} \label{eq:a_i_bdd}
0 \le a_i \le \sigma^2_{\max}  \left(\frac{\ell +1}{\ell -1} \right)  \le 3 \sigma^2_{\max} \;. 
  \end{align}
 Using Bernstein bound, we have
 \begin{align} 
 \Pr_{\tilde{\sigma}^2} \left[\left|
  \hat{\sigma}^2(A_{\rho, 2})  - \frac{\sum_{i \in N_\rho} a_i}{r} 
  \right| \ge \frac{\varepsilon}{4}  \right] 
  %\nonumber \\
  ~\le~ 2 \exp \left(- \frac{\varepsilon r}{48 \sigma^2_{\max}}\right)  \label{eq:bernstein_sim}
\end{align} 
where we let $\Pr_{\tilde{\sigma}^2}$ denote the conditional probability given ${\sigma}^2 = \tilde{\sigma}^2$.
%where the upper bound in \eqref{eq:bernstein} doesn't depend on $\sigma^2_{\rho}$ and $\sigma^2_{\partial^2 \rho}$.
%Noting that $0 \le \frac{\sum_{i \in N_u} a_i}{r} \le 3 \sigma^2_{\max}$ and $\EXP[a_i]$
Using Hoeffding bound with \eqref{eq:a_i_bdd}, it follows that
\begin{align}
\Pr_{\tilde{\sigma}^2} \left[ \left| \frac{\sum_{i \in N_\rho} a_i }{r} 
- \left( \sigma^2_{\MF avg}(\mathcal{S}) + \sigma^2_\rho \right) \right| \ge \frac{\varepsilon}{4} \right]
 %\nonumber \\&
 ~\le~ 2 \exp\left( -\frac{\varepsilon^2 r}{8 \sigma^2_{\max}}\right) \;.\label{eq:hoeffding_sim}
\end{align}
%Since $\sigma^2_u$ is independently
% drawn from the discrete uniform distribution with support $\mathcal{S}$, 
Combining \eqref{eq:bernstein_sim} and \eqref{eq:hoeffding_sim} and using the union bound,
 it follows that
\begin{align}
\Pr_{\tilde{\sigma}^2} \left[ s_\rho \neq \hat{s}^{\dag}_\rho (A_{\rho, 2}) \right] 
%\nonumber \\
&\le \Pr_{\tilde{\sigma}^2} \left[ \left|   \hat{\sigma}^2(A_{\rho, 2})   -  (\sigma^2_{\MF avg} (\mathcal{S}) + \sigma^2_\rho)  \right| > \frac{\varepsilon}{2}   \right] \nonumber \\
&\le 2 
\left( \exp\left( -\frac{\varepsilon r}{48 \sigma^2_{\max}}\right) +  \exp\left( -\frac{\varepsilon^2 r}{8 \sigma^2_{\max}}\right) \right) \nonumber \\
& \le
4  \exp\left( -\frac{\varepsilon^2 r}{8(8\varepsilon+ 1) \sigma^2_{\max}}\right) \label{eq:final_bdd_sim}
\end{align}
where for the first inequality we use $|\sigma^2_{s'} - \sigma^2_{s''} |  \ge \varepsilon$ for all $1 \le s', s'' \le S$ such that $s' \neq s'$.
% it is straightforward to check that
%\begin{align*}
%\EXP[b^k_\rho (\sigma^2_\rho) \mid \sigma^2_\rho] & \ge 
%\EXP [\Pr[ s = \hat{s}^{\MF sim}(A_{\rho, 2} ) ]  \mid  \sigma^2_{\rho} ] \\
%&\ge1 - 4  \exp\left( -\frac{\varepsilon^2 r}{8(8\varepsilon+ 1) \sigma^2_{\max}}\right) 
%\end{align*}
Hence, noting that $\hat{s}^{\dag}$ cannot outperform the optimal one $\hat{s}^*$ in \eqref{eq:connection_b_s}, 
 this performance guarantee on $\hat{s}^\dag$ in \eqref{eq:final_bdd_sim}
 %implies \eqref{eq:concentration} and 
 completes the proof of Lemma~\ref{lem:concentration}.

\subsection{Proof of Lemma~\ref{lem:correlation-decay}}
\label{sec:correlation-decay-pf}

We begin with the underlying intuition on the proof. As Lemma~\ref{lem:concentration} states,
if there is the strictly positive gap $\varepsilon >0$ between $\sigma^2_{\min}$ and $\sigma^2_{\max}$,
%and $r$ is sufficiently large, 
one can recover $\sigma^2_\rho \in \{\sigma^2_{\min}, \sigma^2_{\max}\}$
with small error using only the local information, i.e., $A_{\rho, 2k}$. % $\{A_{i \rho} : i \in N_{\rho}\}$.
On the other hand, $A \setminus A_{\rho, 2k}$ is far from $\rho$
and is less useful on estimating $\sigma^2_\rho$.
In the proof of Lemma~\ref{lem:correlation-decay}, we quantify the decaying rate of information
w.r.t. $k$.

  We first introduce several notations for convenience. 
  For $u \in W_{\rho, 2k}$, let $T_{u} = (V_{u}, W_{u}, E_{u})$ 
be the subtree rooted from $u$ including all the offsprings of $u$ in tree $G_{\rho,
  2k}$. Note that $T_\rho  =G_{\rho, 2k}$. We let $\partial W_{u} \subset W_{\rho, 2k}$ 
   denote the subset of worker on the leaves in $T_u$
    and let $A_{u}:= \{A_{iv} : (i, v) \in E_{u}\}$.
%     where $\Pr[\sigma^2_{u} = \sigma^2_1] = \Pr[\sigma^2_{u} = \sigma^2_2] = 1$,
%     we pick an arbitrary $\tilde{\sigma}^2 =\{ \tilde{\sigma}^2_u \in \mathcal{S} : u \in W_{\rho, 2k}\}$ and 
%    
     Since each worker $u$'s $\sigma^2_u$ is a binary random variable, we define
     a function $s_u: \mathcal{S} \to \{+1, -1\}$ for the given $\tilde{\sigma}^2$
 as follows:
\begin{align*}
s_u(\sigma^2_u) = 
\begin{cases}
+1 & \text{if}~\sigma^2_u = \tilde{\sigma}^2_u \\
-1 & \text{if}~\sigma^2_u \neq \tilde{\sigma}^2_u \;.
\end{cases}
\end{align*} 
%    We will show \eqref{eq:correlation-decay} condition on $\sigma^2 = \tilde{\sigma}^2$.
%To show \eqref{eq:correlation-decay},
%we will show % that %for $s \in \{ 1, 2\}$, 
It is enough to show
    \begin{align} 
%\EXP_{\tilde{\sigma}^2} \left[ \left|
%\Pr[ \sigma^2_\rho = \tilde{\sigma}^2_\rho \mid A]  - \Pr[  \sigma^2_\rho = \tilde{\sigma}^2_\rho \mid A_{\rho, 2k}] 
% \right| \right]  
\EXP_{\tilde{\sigma}^2} \left[ \left|
\Pr[ s_\rho(\sigma^2_\rho) = +1 \mid A_{\rho, 2k}, \sigma^2_{\partial W_{\rho}}]  - \Pr[ s_\rho(\sigma^2_\rho) = +1  \mid A_{\rho, 2k}] 
 \right| \right]  ~\le~ 2^{-k} \label{eq:wts_correlation}
\end{align}
since for each $u \in W$, $\Pr[\sigma^2_{u} = \sigma^2_1] = \Pr[\sigma^2_{u} = \sigma^2_2] = \frac{1}{2}$.

%where 
%%we let $\EXP_{\tilde{\sigma}^2}$ denote the conditional expectation given $\sigma^2 = \tilde{\sigma}^2$
%and 

To do so, we first define
\begin{align*}
X_{u}  := 2 \Pr[s_u(\sigma^2_{u}) = +1 \mid A_{u}] - 1\;, \quad \text{and} \quad  %\Pr[s_u = -1 \mid A_{u}]  \\
%1 %\Pr[ \sigma^2_u = \tilde{\sigma}^2_u \mid  A_{u}] 
 %-  2\Pr[  \sigma^2_u \neq \tilde{\sigma}^2_u \mid  A_{u}] \\
Y_{u} 
:= 2 \Pr[s_u(\sigma^2_{u}) = +1 \mid A_{u},  , \sigma^2_{\partial W_{\rho}}] - 1 %\Pr[s_u = -1 \mid A_{u}]  \\
%&:= \Pr[s_u = +1 \mid A_{u}, A_{-\rho}] - \Pr[s_u = -1 \mid A_{u}, A_{-\rho}] \;.
%&:=1 % \Pr[ \sigma^2_u = \tilde{\sigma}^2_u \mid  A_{u}, A_{-\rho}] 
% - 2 \Pr[  \sigma^2_u \neq \tilde{\sigma}^2_u \mid  A_{u}, A_{-\rho}] \;.
\end{align*}
%where we denote
%$A_{-\rho}:= A \setminus A_\rho$ so that 
so that we have
\begin{align*}
\left| \Pr[ s_\rho(\sigma^2_\rho) = +1 \mid A_{\rho, 2k},  \sigma^2_{\partial W_{\rho}}]
  - \Pr[ s_\rho(\sigma^2_\rho) = +1 \mid A_{\rho, 2k}] \right| 
 ~=~ \frac{1}{2}\left| X_\rho - Y_\rho \right|.
\end{align*}
 %we define $Y_i$ as follows:
%where we let $\Pr_{\tilde{\sigma}^2}$ denote the conditional probability
% given $\sigma^2 = \tilde{\sigma}^2$.
% Noting that 
%Next, for $0 \le t \le k$, we define $i(t) \in \partial V_{\rho,
%  2k-2t}$ to be a random node chosen uniformly at random so that
%$i(0)$ is a leaf node in $G_{\rho, 2k}$, i.e., $X_{i(0)} = 0$ thus
%$|X_{i(0)} - Y_{i(0)}| \le 1$, and $i(k)$ is the root $\rho$, i.e.,
%$\Delta(\hat{z}_\rho^*(A_{\rho})) - \Delta(\hat{s}_\rho^*(A_{\rho})) =
%\frac{1}{2}\EXP \big[ |Y_{\rho}| - |X_\rho| \big]$. Therefore it is
 %$Y_\rho = 1- 2\Pr[\sigma^2_{\rho} \mid  ]$
Using the above definitions of $X_u$ and $Y_u$
and noting $|X_u - Y_u| \le 2$,
it is enough to show that for given non-leaf worker 
$u \in W_\rho \setminus \partial W_\rho$,
\begin{align} \label{eq:recurrence-wts}
\EXP_{\tilde{\sigma}^2} \left[ {\big|X_{u} - Y_{u} \big|}  \right]
 \le \frac{1}{2 |\partial^2 u|} \sum_{v \in \partial^2 u} \EXP_{\tilde{\sigma}^2} \left[ { \big|X_{v} - Y_{v} \big|}  \right]
\end{align}
where  we let $\partial^2 u $ denote the set of grandchildren of $u$ in $T_u$.
%since using \eqref{eq:recurrence-wts} and  $|X_u - Y_u| \le 2$, we have \eqref{eq:wts_correlation}.
% that 
%$$\EXP_{\tilde{\sigma}^2} [|X_\rho - Y_\rho|] \le 2^{-k} \cdot 2 $$
%and completes the proof with 
%In the above, $\EXP_{\tilde{\sigma}^2}$ quantifies
%the correlation between the root $\rho$ and the information $A_{-\rho}$ of the outside of tree $T_\rho$.

To do so, we study certain recursions describing relations among $X$
and $Y$.
For notational convenience, 
%Recalling the factor form of joint probability of $\sigma^2$ in \eqref{eq:sub_factor_graph}
%and using Bayes' theorem and some calculus,
we define $g^{+}_{iu}$ and $g^{-}_{iu}$ as follows:
\begin{align*}
g^{+}_{iu} (X_{\partial_{u} i}; A_{i}) 
%:= \Pr \left[s_u = +1  \mid A_{i}, X_{\partial_{u} i} \right] \\
&:= \sum_{\sigma'^2_{M_i} \in \mathcal{S}^{M_i} : {\sigma}'^2_u = \tilde{\sigma}^2_u
}  \mathcal{C}_i(A_i, \sigma'^2_{M_i}) \prod_{ v \in \partial_{u} i} \frac{1+ s_v (\sigma'^2_v)  X_{v}}{2}
\\% \label{eq:g_plus_def} \\
g^{-}_{iu} (X_{\partial_{i} u}; A_{i}) 
%:= \Pr \left[s_u = -1   \mid A_{i}, X_{\partial_{u} i} \right] \\
&:= \sum_{\sigma'^2_{M_i} \in \mathcal{S}^{M_i} : {\sigma}'^2_u \neq \tilde{\sigma}^2_u
}  \mathcal{C}_i(A_i, \sigma'^2_{M_i}) \prod_{ v \in \partial_{u} i} \frac{1+ s_v (\sigma'^2_v)  X_{v}}{2}
\;. %\label{eq:g_minus_def}
%&:=
%\EXP_{\mu} \left[\frac{1-A_{iu} \mu_u}{2}
% \prod_{j\in \partial_{i} u}  \frac{1+ A_{ju} \mu_u X_j }{2}  \right]
\end{align*}
where we may omit $A_i$ in the argument of $g^{+}_{iu}$ and $g^{-}_{iu}$
if $A_i$ is clear from the context. Recalling 
the factor form of the joint probability of $\sigma^2$, i.e., 
and using Bayes' theorem with the fact that $\Pr[s_u(\sigma^2_u) = +1 \mid A_u] = \frac{1+X_u}{2}$
and some calculus, it is not hard to check
\begin{align}
g^{+}_{iu} (X_{\partial_{u} i}; A_{i}) 
&\propto \Pr \left[s_u(\sigma^2_u) = +1  \mid A_{i}, X_{\partial_{u} i} \right]  \label{eq:g_plus_prob} \\ %\;, ~\text{and} \lable\\
g^{-}_{iu} (X_{\partial_{i} u}; A_{i}) 
&\propto \Pr \left[s_u(\sigma^2_u) = -1   \mid A_{i}, X_{\partial_{u} i} \right] \;.\label{eq:g_minus_prob}
\end{align}
From the above, it is straightforward to check that
\begin{align} \label{eq:recurrence-X}
X_{u} &= h_u(X_{\partial^2 u}) \cr
&:= \frac{\prod_{i \in \partial u} g^{+}_{iu} (X_{\partial_u i}) -\prod_{i \in \partial u} g^{-}_{iu} (X_{\partial_u i})}
{\prod_{i \in \partial u} g^{+}_{iu} (X_{\partial_u i}) +\prod_{i \in \partial u} g^{-}_{iu} (X_{\partial_u i})}
\end{align}
where we let $\partial u$ be the task set of all the children of
worker $u$ and $\partial_u i$ be the worker set of all the children of $i$ in tree $T_u$.
%i.e., $\partial u : = \{i \in M_{u} : (i,u) \in E_{u}\}$ and
%$\partial_u i : = \{v \in N_{i} : (i,u) \in E_u \}$.
Similarly, we also have
\begin{align*} %\label{eq:recurrence-X}
Y_{u} &= h_u(Y_{\partial^2 u}) \;.
%&:= \frac{\prod_{i \in \partial u} g^{+}_{iu} (X_{\partial_u i}) -\prod_{i \in \partial u} g^{-}_{iu} (X_{\partial_u i})}
%{\prod_{i \in \partial u} g^{+}_{iu} (X_{\partial_u i}) +\prod_{i \in \partial u} g^{-}_{iu} (X_{\partial_u i})}
\end{align*}

For simplicity, we now pick an arbitrary worker $u \in W_{\rho}$ which is neither
the root nor a leaf, i.e., $ u \notin \partial W_{\rho}$ and $u \neq \rho$, so that
$\left|\partial^2 u \right| = (\ell - 1)  (r - 1)$. It is enough to show \eqref{eq:recurrence-wts} for only $u$.
%Then, to
%prove \eqref{eq:recurrence-wts}, it is enough to show that
%\begin{align}
%\EXP^+ \left[ {|X_i - Y_i|}  \right]  
%\le \frac{1}{2 (\ell - 1) (r - 1)} \sum_{j \in \partial^2 i}\EXP^+ \left[ {|X_{j} - Y_j|}  \right]
%\label{eq:final-wts}
%\end{align}
%where we let $\EXP^+$ denote the conditional expectation given $s_j =
%+1$ for all $j$.
To do so, we will use the mean value theorem. We
first obtain a bound on the gradient of $h_u(x)$ for $x \in
[-1,1]^{\partial^2 u}$. Define $g^+_u(x): = \prod_{i \in \partial u}
g^+_{iu}(x_{\partial_{u} i})$
 and $g^-_u(x): = \prod_{i \in \partial u}
g^-_{iu}(x_{\partial_{u} i})$.  Using basic calculus, we
obtain that for $v \in \partial_{u} i$,
\begin{align*}
  \frac{\partial h_u}{\partial x_v}
  &= \frac{\partial }{\partial x_v} \frac{g^+_u -  g^-_u}{g^+_u+ g^-_u}\\
 &= \frac{2}{(g^+_u + g^-_u)^2} \left(
g^-_u \frac{\partial g^+_u}{\partial x_v}  - g^+_u  \frac{\partial g^-_u}{\partial x_v}
\right)
 \\
 &= \frac{2 g^+_u  g^-_u}
{(g^+_u + g^-_u)^2}
\left(
\frac{1}{g^+_{iu}} \frac{\partial g^+_{iu}}{\partial x_v} %
- \frac{1}{g^-_{iu}} \frac{\partial g^-_{iu}}{\partial x_v} %
\right).
% & =
% \begin{cases}
% \frac{2 g^+_i g^-_i }{(g^+_i + g^-_i)^2}
% \left(\frac{\delta}{1+ \delta x_j} +\frac{\delta'}{1-\delta' x_j}  \right)  &\text{if~} \sigma_{ij}  = ++ \\
% -\frac{2 g^+_i g^-_i }{(g^+_i + g^-_i)^2}
% \left( \frac{\delta}{1- \delta x_j}+\frac{\delta'}{1+\delta' x_j} \right)  &\text{if~} \sigma_{ij} = +- \\
% -\frac{2 g^+_i g^-_i }{(g^+_i + g^-_i)^2}
% \left(\frac{\delta}{1+ \delta x_j}+\frac{\delta'}{1-\delta' x_j} \right)  &\text{if~} \sigma_{ij} = -+ \\
% \frac{2 g^+_i g^-_i }{(g^+_i + g^-_i)^2}
% \left(\frac{\delta}{1- \delta x_j}+\frac{\delta'}{1+\delta' x_j} \right)  &\text{if~} \sigma_{ij} = -- \\
% \end{cases} \\
%& =
%\begin{cases}
%~~~\frac{ g^+_i g^-_i }{(g^+_i + g^-_i)^2}
%\frac{4\theta}{(1+ \delta x_j)(1-\delta' x_j)}   &\text{if~} \sigma_{ij}  = ++ \\
%-\frac{ g^+_i g^-_i }{(g^+_i + g^-_i)^2}
%\frac{4\theta}{(1- \delta x_j)(1+\delta' x_j)}  &\text{if~} \sigma_{ij} = +- \\
%-\frac{ g^+_i g^-_i }{(g^+_i + g^-_i)^2}
%\frac{4\theta}{(1+ \delta x_j)(1-\delta' x_j)}  &\text{if~} \sigma_{ij} = -+ \\
%~~~\frac{ g^+_i g^-_i }{(g^+_i + g^-_i)^2}
%\frac{4\theta}{(1- \delta x_j)(1+\delta' x_j)}  &\text{if~} \sigma_{ij} = --.
%\end{cases}
\end{align*}
% In the last equality, we avoid division-by-zero by assuming 
% constant $\mu < 1$ which implies that 
% there exist  constants $f_{\max}$ and $f_{\min}$ with respect to $\ell$
% such that for all $s_{N(u)} \in \{-1, +1\}^{N(u)}$ 
% \begin{align} \label{eq:f_bound}
% 0 < f_{\min} \le f_u(s_{N(u)}) \le  f_{\max}< 1
% \end{align}
%  and thus by simple algebra, it follows that
% $\min \{ g^+_{iu}, g^-_{iu} \} \ge f_{\min} > 0.$
% % \quad \text{and} \quad \max \{ g^+_{iu}, g^-_{iu} \} \le f_{\max} \cdot 2^{2(r-1)}.

Using the fact that for $x \in [-1, 1]^{\partial^2 u}$, both $g^+_u$ and $g^-_u$ are
positive, it is not hard to show that
\begin{align} \label{eq:diff-upper}
 \frac{g^+_u g^-_u}{(g^+_u + g^-_u)^2} \le \sqrt{\frac{g^-_u}{g^+_u}}.
\end{align}
We note here that one can replace ${g^-_u}/{g^+_u}$ with ${g^+_u}/{g^-_u}$
 in the upper bound.
However, in our analysis, we use \eqref{eq:diff-upper} since we will take 
the conditional expectation $\EXP_{\tilde{\sigma}^2}$ 
which takes  the randomness of $A$ generated by the condition $\sigma^2 = \tilde{\sigma}^2$.
Hence $X_u$ and $Y_u$ will be closer to $1$ than $-1$
 thus ${g^-_u}/{g^+_u}$ will be a tighter upper
bound than ${g^+_u}/{g^-_u}$.
% Our analysis covers all the general
%cases because the same analysis with ${g^+_u}/{g^-_i}$ will work with
%$s_i = -1$ conversely.

From \eqref{eq:diff-upper},
it follows that for $x \in [-1, 1]^{\partial^2 u}$ and $v \in \partial_u i$,
\begin{align*}
\left| \frac{\partial h_u}{\partial x_v} (x) \right| 
&\le \left|g'_{uv}(x_{\partial_{u}i})  \right|   \prod_{j \in \partial u \,:\, j \neq i} \sqrt{ \frac{g^-_{ju} (x_{\partial_{u}j})}{g^+_{ju} (x_{\partial_{u}j})} } 
\end{align*}
where we define
\begin{align*}
g'_{uv}(x_{\partial_{u}i}) 
~ :=~
2  \sqrt{ \frac{g^-_{iu} (x_{\partial_{u}i})}{g^+_{iu} (x_{\partial_{u}i})} }
\left(
\frac{1}{g^+_{iu} (x_{\partial_{u}i})} \frac{\partial g^+_{iu}(x_{\partial_{u}i})}{\partial x_v} %
- \frac{1}{g^-_{iu}(x_{\partial_{u}i})} \frac{\partial g^-_{iu}(x_{\partial_{u}i})}{\partial x_v} %
\right).
\end{align*}
Further, we
make the bound independent of 
$x_{\partial_{u} i} \in [-1,1]^{\partial_{u}i}$
 by taking
 the maximum of $|g'_{uv}(x_{\partial_{u}i})|$, i.e.,
\begin{align}
\left| \frac{\partial h_u}{\partial x_v} (x) \right| 
&\le \eta_i(A_i)   \prod_{j \in \partial u \,:\, j \neq i} \sqrt{ \frac{g^-_{ju} (x_{\partial_{u}j};A_j)}{g^+_{ju} (x_{\partial_{u}j}; A_j)} }  \label{eq:diff_bdd}
\end{align}
where we define
\begin{align*}
\eta_{i}(A_i)  := \max_{x_{\partial_{u} i} \in [-1,1]^{\partial_{u}i}}
g'_{uv}(x_{\partial_{u}i};A_i) \;.
\end{align*}
Now we apply the mean value theorem with \eqref{eq:diff_bdd} to bound 
$|X_u - Y_u| = |h_u(X_{\partial^2 u}) - h_u(Y_{\partial^2 u})| $ by
$|X_v - Y_v|$ of $v \in \partial^2 u$.
It follows that for given $X_{\partial^2 u}$ and $Y_{\partial^2 u}$, 
there exists $\lambda' \in [0,1]$
such that
\begin{align}
&|X_u - Y_u| = |h_u(X_{\partial^2 u}) - h_u(Y_{\partial^2 u})|
 \nonumber \\
&\le \sum_{i \in \partial {u}}\sum_{v \in \partial_{u} i} 
|X_{v} - Y_{v}| \nonumber   \left|  \frac{\partial h_u}{\partial x_v}  \left(\lambda' X_{\partial^2 u} + (1-\lambda') Y_{\partial^2 u} \right)\right| 
\nonumber  \\
&\le \sum_{i \in \partial {u}}\sum_{v \in \partial_{u} i} 
|X_{v} - Y_{v}|  \eta_i(A_i) 
\prod_{j  \in \partial u: j \neq i} \max_{\lambda \in [0,1]} \left\{ 
\sqrt{
\frac{g^-_{ju} (\lambda X_{\partial_{u}j}+(1-\lambda) Y_{\partial_{u}j}; A_j)}{g^+_{ju} (\lambda X_{\partial_{u}j}+(1-\lambda) Y_{\partial_{u}j}; A_j)}}
\right\}.
 \label{eq:mean-value}
\end{align}
where for the first and last inequalities, we use the mean value theorem and \eqref{eq:diff_bdd}, respectively. 
We note that each term in an element of the summation in the RHS
of \eqref{eq:mean-value} is independent to each other.
Thus, %from the symmetry among $\{X_{\partial_{u} i}\}_{i \in \partial u}$, 
it follows that 
\begin{align}  
&\EXP_{\tilde{\sigma}^2} \left[ {|X_u - Y_u|} \right]\nonumber \\
& \le  
\sum_{i \in \partial {u}}\sum_{v \in \partial_{u} i} 
\EXP_{\tilde{\sigma}^2}  \left[ {|X_{v} - Y_{v}|} \right]
 \EXP_{\tilde{\sigma}^2}  \left[ {\eta_i}(A_i) \right] 
\prod_{j \in \partial u :j \neq i}
\EXP_{\tilde{\sigma}^2}
\left[ 
\max_{\lambda \in [0,1]} 
\Gamma_{ju}(\lambda X_{\partial_{u}j}+(1-\lambda) Y_{\partial_{u}j})
\right] 
\label{eq:exp_bdd}
\end{align}
where we define function $\Gamma_{iu}(x_{\partial_{u} i}; A_i)$ for given $x_{\partial_{u} i} \in [-1,1]^{\partial_{u} i }$ 
as follows:
\begin{align*}
\Gamma_{iu}(x_{\partial_{u} i})  & := 
\sqrt{\frac{g^-_{iu}(x_{\partial_{u} i}; A_i)}{g^+_{iu}(x_{\partial_{u} i} ; A_i)}} \;.
%& = 
%\cdot 
%\sqrt{
%\frac{g^-_{iu} (x_{\partial_{i} u};A_u)}{g^+_{iu} (x_{\partial_{i} u};A_u)}
%}
\end{align*}
%where we let $\Pr{\text{$^+$}}$ denote the conditional probability measure given that $s_j$ for all $j$.
%where we use the fact that for given $A_u \in \{-1, +1\}^{N_u}$, 
%\begin{align*} % \label{eq:A_u_dist}
%\Pr{\text{$^+$}} [A_u ] = \EXP_{\mu} \left[ \prod_{j \in N_u} \frac{1+A_j \mu_u}{2}\right].
%\end{align*}
Note that the assumption on $\sigma^2_{\min}$ and $\sigma^2_{\max}$, i.e.,
 $\sigma^2_{\min} +\varepsilon \le \sigma^2_{\max}  < \frac{5}{2}  \sigma^2_{\min}$. This implies 
 $$\left(-\frac{1}{\sigma^2_{\max}}+ \frac{1}{\sigma^2_{\min}} \right) \frac{3}{2}
 -\frac{1}{\sigma^2_{\max}} ~<~ 0 \;. $$
 Hence, %we have $g'_{uv} \le \exp(- C \sum_{})$
 for constant $\ell$ and $\varepsilon > 0 $,
  it is not hard to check that there is a finite constant 
 $\eta$ with respect to $r$ such that
% depending on only $\varepsilon$ and $\sigma^2_{\min}$, $\sigma^2_{\max}$
% that bounds
%  $\EXP_{\tilde{\sigma}^2} [\eta_i(A_i)]$ regardless of $\tilde{\sigma}^2$ and $r$ but
%  and , i.e., 
\begin{align} \label{eq:constant_bound}
\max_{\tilde{\sigma}^2} \EXP_{\tilde{\sigma}^2} [\eta_i(A_i)]
~ \le ~
\eta ~<~ \infty
\end{align}
where $\eta$ may depend on only $\varepsilon$, $\sigma^2_{\min}$, and $\sigma^2_{\max}$.
%We note that $\eta$ is independent of $r$.

%In addition, we also find a constant $\psi$ bounding 
In addition, we also obtain a bound of the last term of \eqref{eq:exp_bdd},
when $r$ is sufficiently large, in the following lemma whose proof
is presented in Section~\ref{sec:one-lem-pf}.
\begin{lemma} \label{lem:one}
For given $\tilde{\sigma}^2_{M_i} \in \mathcal{S}^{M_i}$ and $u \in M_i$, 
let $\tilde{\sigma}'^2_{M_i} \in \mathcal{S}^{M_i}$ be 
the set of $\tilde{\sigma}'^2_v$ such that 
$\tilde{\sigma}^2_u \neq \tilde{\sigma}'^2_u$ and 
 $\tilde{\sigma}^2_v = \tilde{\sigma}'^2_v$ for all 
$v \in M_i \setminus \{u\}$. Then, there  exists a constant $C'_{\ell, \varepsilon}$
such that for any $r \ge C'_{\ell, \varepsilon}$,
\begin{align*}
\EXP_{\tilde{\sigma}^2}
\left[ 
\max_{\lambda \in [0,1]} 
\Gamma_{iu}(\lambda X_{\partial_{u}i}+(1-\lambda) Y_{\partial_{u}i})
\right] 
~\le~ 
1-\frac{\Delta_{\min}}{2}
~<~ 1,
\end{align*}
where we let $\Delta_{\min}$ be the square of the minimum Hellinger distance between the conditional densities of $A_i$ given two different $\sigma'^2_{M_i}$ and $\sigma''^2_{M_i}$, i.e., 
\begin{align}
\Delta_{\min} := \min_{\sigma^2_{M_i}, \sigma'^2_{M_i} \in \mathcal{S}^{M_i} : \sigma^2_{v} \neq \sigma'^2_{v} ~\exists v \in M_i} 
H^2(f_{A_i \mid \sigma'^2_{M_i} } , f_{A_i \mid \sigma''^2_{M_i} }) 
~>~ 0 \;.\nonumber
\end{align}
\end{lemma}
Using the above lemma, %and noting that the product in \eqref{eq:exp_bdd},
we can find a sufficiently large constant $C_{\ell, \varepsilon} \ge C'_{\ell, \varepsilon} $ such that if $|\partial u |= r \ge
C_{\ell, \varepsilon}$,
\begin{align*}
\prod_{j \in \partial u :j \neq i}
\EXP_{\tilde{\sigma}^2}
\left[ 
\max_{\lambda \in [0,1]} 
\Gamma_{ju}(\lambda X_{\partial_{u}j}+(1-\lambda) Y_{\partial_{u}j})
\right]
&\le
\eta \left(1-\psi_{\min} \right)^{\frac{{C_{\ell, \varepsilon} - 2 }}{2}}\\
& \le \frac{1}{2(\ell-1)(C_{\ell, \varepsilon}-1)}
  ~\le~ \frac{1}{2(\ell - 1)(r-1) }
\end{align*}
which implies \eqref{eq:recurrence-wts} with \eqref{eq:exp_bdd} and completes the proof of
Lemma~\ref{lem:correlation-decay}.

\subsection{Proof of Lemma~\ref{lem:one}}  
  \label{sec:one-lem-pf}

We first obtain a bound on $X_v$ and $Y_v$ for  $v \in \partial_u i$.
Noting that $v$ is a non-leaf node in $G_{\rho, 2k}$ and $|\partial v| = r -1$,
Lemma~\ref{lem:concentration}
 directly provides
  \begin{align*}
\EXP_{\tilde{\sigma}^2 } \left[ \Pr[\sigma^2_v \neq \tilde{\sigma}^2_v \mid A_{v, 2k}] 
\right] ~=~ \EXP_{\tilde{\sigma}^2 } \left[ \frac{1 - X_v}{2} \right] ~\le~
 4  \exp\left({ -\frac{\varepsilon^2}{8(8\varepsilon+ 1) \sigma^2_{\max}}  (r-1)} \right) \;.
     \end{align*}
Using Markov inequality for $\frac{1-X_v}{2} \ge 0$, 
it is easy to check that for any $\delta > 0$,
\begin{align} \label{eq:markov_X}
\Pr_{\tilde{\sigma}^2} \left[ X_v < 1- \delta \right]
 \le  \frac{8}{\delta}  \exp \left({ -\frac{\varepsilon^2}{8(8\varepsilon+ 1) \sigma^2_{\max}}  (r-1)} \right) \;.
\end{align}
Note that
\begin{align*}
 4  \exp\left({ -\frac{\varepsilon^2}{8(8\varepsilon+ 1) \sigma^2_{\max}}  (r-1)} \right)
&~ \ge ~
\EXP_{\tilde{\sigma}^2 } \left[ \Pr[\sigma^2_v \neq \tilde{\sigma}^2_v \mid A_{v}] 
\right] \\
& ~\ge~ \EXP_{\tilde{\sigma}^2 } \left[ \Pr[\sigma^2_v \neq \tilde{\sigma}^2_v \mid A_{v}, A_{-v}] 
\right] ~=~ \EXP_{\tilde{\sigma}^2 } \left[ \frac{1 - Y_v}{2} \right]
 \;.
\end{align*}
Hence, %similarly to \eqref{eq:markov_X}, 
we have the same bound in \eqref{eq:markov_X} for $Y_v$, i.e.,
\begin{align*} %\label{eq:markov_Y}
\Pr_{\tilde{\sigma}^2} \left[ Y_v < 1- \delta \right]
 ~\le~  \frac{8}{\delta}  \exp \left({ -\frac{\varepsilon^2}{8(8\varepsilon+ 1) \sigma^2_{\max}}  (r-1)} \right) \;.
\end{align*}

Using the assumption that $\sigma^2_{\min} +\varepsilon \le \sigma^2_{\max}  < \frac{5}{2}  \sigma^2_{\min}$,
similarly to \eqref{eq:constant_bound}, we can find finite constants
$\eta'$ and $\eta''$ with respect to $r$ such that for all $x \in [0,1]^{\partial_{u} i}$,
\begin{align*}
 \max_{\tilde{\sigma}'^2} \EXP_{\tilde{\sigma}'^2}\left[\left|  \Gamma_{iu}(x) \right| \right]
  ~\le~ \eta' \;, \quad \text{and} \quad
   \max_{\tilde{\sigma}'^2}
\EXP_{\tilde{\sigma}^2} \left[ \left| \frac{\partial \Gamma_{iu}(x)}{\partial x_v} \right| \right] 
~\le~ \eta''.
\end{align*}
%Let $\varepsilon(\ell) := \exp\left(- \frac{(\ell-1)\mu^2}{4} \right) \le \exp\left(- \frac{(|\partial j|-1)\mu^2}{4} \right).$
Then, it follows that for given $\delta > 0$, 
\begin{align}
&\EXP_{\tilde{\sigma}^2} \left[
 \max_{\lambda \in [0, 1]} 
\Gamma_{iu}(\lambda X_{\partial_{u}i}+(1-\lambda) Y_{\partial_{u}i})
\right] \nonumber \\
&\le
\left(1 - 
\Pr_{\tilde{\sigma}^2} 
\left[
X_{v} > 1-\delta \text{~and~} Y_v > 1-\delta \;, ~\forall v \in \partial_{u}i
\right]
 \right) 
 \max_{x \in [-1,1]^{\partial_{u} i}} \EXP_{\tilde{\sigma}^2} \left[\Gamma_{iu}(x)\right] + 
%\Pr_{\tilde{\sigma}^2} 
%\left[
%X_{v} > 1-\delta \text{~and~} Y_v > 1-\delta \;, ~\forall v \in \partial_{u}i
%\right] 
 \max_{x \in [1-\delta,1]^{\partial_{u} i}} \EXP_{\tilde{\sigma}^2}[ \Gamma_{iu}(x)] \nonumber \\
&{\le} \label{eq:gamma_a}
\left( \sum_{v \in \partial_{u}i}
\Pr_{\tilde{\sigma}^2}[X_{v} \le 1-\delta] + 
\Pr_{\tilde{\sigma}^2}[Y_{v} \le 1-\delta]
 \right) 
  \max_{x \in [-1,1]^{\partial_{i} u}}\EXP_{\tilde{\sigma}^2}[ \Gamma_{iu}(x)]
+  \max_{x \in [1-\varepsilon,1]^{\partial_{i} u}}\EXP_{\tilde{\sigma}^2}[ \Gamma_{iu}(x)]  \\
&{\le} \label{eq:gamma_b}
r  \eta'  \frac{8}{\delta}  \exp \left({ -\frac{\varepsilon^2}{8(8\varepsilon+ 1) \sigma^2_{\max}}  (r-1)} \right)
+ \max_{x \in [1-\delta,1]^{\partial_{i} u}} \EXP_{\tilde{\sigma}^2}[\Gamma_{iu}(x) ] \\ 
&{\le} \label{eq:gamma_c}
r  \eta'  \frac{8}{\delta}  \exp \left({ -\frac{\varepsilon^2}{8(8\varepsilon+ 1) \sigma^2_{\max}}  (r-1)} \right)
+ \delta \eta'' +  \EXP_{\tilde{\sigma}^2}[\Gamma_{iu}(1_{\partial_{u} i}) ] 
\end{align}
where  for \eqref{eq:gamma_a}, \eqref{eq:gamma_b}, and \eqref{eq:gamma_c}, we use the union bound, \eqref{eq:markov_X}, and the mean value theorem, respectively.
We will show there exists constant $\Delta$ such that
$\EXP_{\tilde{\sigma}^2} [\Gamma_{iu}(1_{\partial_{u} i})] \le {1-\Delta}$, since
the first term in \eqref{eq:gamma_c} is exponentially 
decreasing with respect to $r$ thus there exists a constant $C'_{\ell, \varepsilon}$ such that
for $r \ge C'_{\ell, \varepsilon}$,
\begin{align*}
\EXP_{\tilde{\sigma}^2} \left[
 \max_{\lambda \in [0, 1]} 
\Gamma_{iu}(\lambda X_{\partial_{u}i}+(1-\lambda) Y_{\partial_{u}i})
\right] 
~\le~ 1 - \frac{\Delta}{2} \;.
\end{align*}
Recalling the property of $g^+_{iu}$ and $g^-_{iu}$ in \eqref{eq:g_plus_prob} and \eqref{eq:g_minus_prob},
it directly follows that 
\begin{align*}
&\EXP_{\tilde{\sigma}^2}  [ \Gamma_{iu}(1_{\partial_{u} i}) ]\\
&= \int_{\mathbb{R}^{d \times M_i}} f_{A_i} [x_i \mid  \sigma^2_{M_i}  = \tilde{\sigma}^2_{M_i} ] 
\sqrt{\frac{g^-_{iu}(1_{\partial_{u}i} ; A_i = x_i)}{g^+_{iu}(1_{\partial_{u}i} ; A_i= x_i)}} ~d x_i
\\
&
= \int_{\mathbb{R}^{d \times M_i}} f_{A_i} [x_i \mid   \sigma^2_{M_i}  = \tilde{\sigma}^2_{M_i} ] 
\sqrt{\frac{f_{A_i} [x_i \mid   \sigma^2_{M_i \setminus \{u\}}  = \tilde{\sigma}^2_{M_i \setminus \{u\}},  \sigma^2_{u}  = \tilde{\sigma}'^2_{u} ] }
{f_{A_i} [x_i \mid   \sigma^2_{M_i}  = \tilde{\sigma}^2_{M_i} ] }} ~d x_i
\\
&
= \int_{\mathbb{R}^{d \times M_i}} \sqrt{f_{A_i} [x_i \mid   \sigma^2_{M_i}  = \tilde{\sigma}^2_{M_i} ] }
\sqrt{f_{A_i} [x_i \mid   \sigma^2_{M_i \setminus \{u\}}  = \tilde{\sigma}^2_{M_i \setminus \{u\}},  \sigma^2_{u}  = \tilde{\sigma}'^2_{u} ] } ~d x_i \;.
\end{align*}
 For notational simplicity, we define 
\begin{align*}
\Delta(\tilde{\sigma}^2_{M_i}, \tilde{\sigma}'^2_{M_i}) := \frac{1}{2} -
\frac{1}{2} 
\int_{\mathbb{R}^{d \times M_i}} \sqrt{f_{A_i} [x_i \mid   \sigma^2_{M_i}  = \tilde{\sigma}^2_{M_i} ] }
\sqrt{f_{A_i} [x_i \mid   \sigma^2_{M_i \setminus \{u\}}  = \tilde{\sigma}^2_{M_i \setminus \{u\}},  \sigma^2_{u}  = \tilde{\sigma}'^2_{u} ] } ~d x_i  \;.
\end{align*}
Then $2  \Delta(\tilde{\sigma}^2_{M_i}, \tilde{\sigma}'^2_{M_i}) $ is equal to 
the square of the Hellinger distance $H$ between the conditional densities of $A_i$ given 
$ \sigma^2_{M_i}  = \tilde{\sigma}^2_{M_i}$ and $ \sigma^2_{M_i}  = \tilde{\sigma}'^2_{M_i}$, i.e.,
$$\Delta(\tilde{\sigma}^2_{M_i}, \tilde{\sigma}'^2_{M_i}) ~=~ {H^2( f_{A_i \mid \tilde{\sigma}^2_{M_i}}, f_{A_i \mid \tilde{\sigma}'^2_{M_i}})}~ >~ 0 \;. $$
This implies $\Delta(\tilde{\sigma}^2_{M_i}, \tilde{\sigma}'^2_{M_i}) > 0$ and taking the minimum $\Delta$, we complete the proof of Lemma~\ref{lem:one}.

%\section{Other }

 \subsection{Proof of inequality \eqref{eq:MSE_BP_bdd}}%{Detailed Steps for \eqref{eq:MSE_BP_bdd}}
 \label{sec:MSE_BP_bdd_pf}
 Noting that $\hat{\mu}^{{\MF \bayIter} (k)}_i (A)$ is the weighted sum of $\bar{\mu}_i(A_i, \sigma'^2_{M_i} )$ as described in 
 \eqref{eq:BP-estimator}, we can rewrite $\| \hat{\mu}^{{\MF \bayIter} (k)}_i (A)  - \mu_i \|^2_2$ as follows:
\begin{align*}
\ \| \hat{\mu}^{{\MF \bayIter} (k)}_i (A)  - \mu_i \|^2_2  =\sum_{\sigma'^2_{M_i} } 
\sum_{\sigma''^2_{M_i} }
 \big( \bar{\mu}_i(A_i, \sigma'^2_{M_i} ) - \mu_i \big)^\top  
 \big( \bar{\mu}_i(A_i, \sigma''^2_{M_i} ) - \mu_i\big) 
% & \qquad \qquad \quad \times 
 b^k_{i} (\sigma'^2_{M_i})  b^k_{i} (\sigma''^2_{M_i}) \; .
% \prod_{u \in M_{i}} b^k_u (\sigma'^2_u) \cdot b^k_u(\sigma''^2_u)  \; .
\end{align*}
Hence, using Cauchy-Schwarz inequality 
for random variables for the summation over 
all $\sigma'^2_{M_i}, \sigma''^2_{M_i} \in \mathcal{S}^{\ell}$ except $\sigma'^2_{M_i} \neq \tilde{\sigma}^2_{M_i}$, it follows that
%for $\tilde{\sigma}^2 \in \mathcal{S}^W$,
\begin{align}
\EXP_{\tilde{\sigma}^2 }  \left[ \| \hat{\mu}^{{\MF \bayIter} (k)}_i (A)  - \mu_i \|^2_2  \right]
 \le & \EXP_{\tilde{\sigma}^2}  \left[\left\| \left( \bar{\mu}_i\left(A_i, \tilde{\sigma}^2_{M_i} \right) - \mu_i \right) \right\|^2_2 
\right] \nonumber  + \sum_{\sigma''^2_{M_i} }
\sum_{\sigma'^2_{M_i} \neq \tilde{\sigma}^2_{M_i}} 
 \sqrt{
\EXP_{\tilde{\sigma}^2}  \left[ \left(b^k_{i} (\sigma'^2_{M_i}) b^k_{i}(\sigma''^2_{M_i})  \right)^2 \right]} 
 \nonumber \\
&   \quad   \times 
\sqrt{   \EXP_{\tilde{\sigma}^2} \left[ 
\left(
\big( \bar{\mu}_i(A_i, \sigma'^2_{M_i} ) - \mu_i \big)^\top  
 \big( \bar{\mu}_i(A_i, \sigma''^2_{M_i} ) - \mu_i\big)
\right)^2
 \right] } \; . \label{eq:BP_minus_mu_i}
 \end{align}
% where for the last inequality, 
%we use the fact that $0 \le b^k_u (\sigma^2_u) \le 1$.
% it is straightforward to check
%that
%\begin{align*}
%&\sum_{\sigma'^2_{M_i} \neq \tilde{\sigma}^2_{M_i}} 
%%\sum_{\sigma''^2_{M_i} }
%\sqrt{
%\EXP_{\tilde{\sigma}^2} \left[  \prod_{u \in M_{i}} \left(b^k_u (\sigma'^2_u)  \right)^2 \right]} \\
%&\le
%\sum_{u \in M_i}  \sum_{\sigma'^2_{u} \neq \tilde{\sigma}^2_{u}} 
%\sqrt{
%\EXP_{\tilde{\sigma}^2} \left[   b^k_u (\sigma'^2_u)    \right]}
%\end{align*}
Noting that the conditional density of $X = ( \bar{\mu}_i(A_i, \tilde{\sigma}^2_{M_i} ) - \mu_i )$ 
given $\sigma^2 = \tilde{\sigma}^2$ is identical to
$\phi (X  \mid 0, \bar{\sigma}^2_i (\tilde{\sigma}^2_{M_i})),$
%the $d$-dimensional spherical Gaussian with zero mean and variance $\bar{\sigma}^2_i (\tilde{\sigma}^2_{M_i})$.
%the first term is obtained by $d\EXP[\bar{\sigma}^2_i (\tilde{\sigma}^2_{M_i})]$.
it follows that
\begin{align}
\EXP_{\tilde{\sigma}^2}  \left[\left\| \left( \bar{\mu}_i\left(A_i, \tilde{\sigma}^2_{M_i} \right) - \mu_i \right) \right\|^2_2 
\right]
 =d  \bar{\sigma}^2_i (\tilde{\sigma}^2_{M_i}) \; . 
%  d\EXP[\bar{\sigma}^2_i (\tilde{\sigma}^2_{M_i})]  \; . 
 \label{eq:correct_sigma}
\end{align}

To complete the proof of \eqref{eq:MSE_BP_bdd},
we hence obtain an upper bound of the last term in the RHS of \eqref{eq:BP_minus_mu_i}.
For any $\sigma'^2_{M_i}  \in \mathcal{S}^{M_i}$,
the conditional density of the random vector 
$\bar{\mu}_i (A_i, \sigma'^2_{M_i} ) -\mu_i $ conditioned on $\sigma^2 = \tilde{\sigma}^2$
is identical to 
$$f_{\bar{\mu}_i (A_i, \sigma'^2_{M_i} ) -\mu_i}[ x \mid \sigma^2 = \tilde{\sigma}^2] 
= \phi \left(x ~\bigg|~ 0, {\left(\bar{\sigma}^2_i( \sigma'^2_{M_i}) \right)^2}
 {\left({\frac{1}{\tau^2} +\sum_{u \in M_i}\frac{\tilde{\sigma}^2_u}{\sigma'^4_u}}\right) } \right)\;.$$
% can be interpreted as the Gaussian with zero mean and covariance matrix ${\left(\bar{\sigma}^2_i( \sigma'^2_{M_i}) \right)^2}/\left({\frac{1}{\tau^2} +\sum_{u \in M_i}\frac{\sigma^2_u}{\sigma'^4_u}}\right) $,
Using this with some linear algebra, 
it is straightforward to check that for all $\sigma'^2_{M_i}  \in \mathcal{S}^{M_i}$,
\begin{align*}
\EXP_{\tilde{\sigma}^2} \left[\|\bar{\mu}_i (A_i, \sigma'^2_{M_i} ) -\mu_i \|^4_2  \right] 
&=d(2+d)  
  \left( {\left(\bar{\sigma}^2_i( \sigma'^2_{M_i}) \right)^2}
  \left({\frac{1}{\tau^2} +\sum_{u \in M_i} \frac{\tilde{\sigma}^2_u}{\sigma'^4_u} } \right) \right)^2  \\
 &=d(2+d)  
  \left( \frac{\frac{1}{\tau^2} + \sum_{u \in M_i} \frac{\tilde{\sigma}^2_u}{\sigma'^4_u} }
  {
  \left(\frac{1}{\tau^2} +\sum_{u \in M_i} \frac{1}{\sigma'^2_u} \right)^2
  }
   \right)^2  \\
& \le  
d(2+d)  
  \left( \frac{\frac{1}{\tau^2} + \ell \frac{{\sigma}^2_{\max}}{\sigma^4_{\min}} }
  {
  \left(\frac{1}{\tau^2} +\ell \frac{1}{\sigma^2_{\min}} \right)^2
  }
   \right)^2 
%\frac{d\sigma^6_{\max}}{\ell^3 \sigma^2_{\min}} \left(2  + \frac{d\sigma^6_{\max}}{\ell^3 \sigma^2_{\min}} \right) \; ,
\end{align*}
where for the last inequality, we use the fact
 that 
 $|M_i| = \ell$ and 
 $\sigma^2_{\min} \le \sigma^2_s \le \sigma^2_{\max}$ for any $1 \le s \le S$.
Using Cauchy-Schwarz inequality with the above bound,
it is not hard to check that
 for any $\sigma'^2_{M_i}, \sigma''^2_{M_i}  \in \mathcal{S}^{M_i}$,
\begin{align}
&\EXP_{\tilde{\sigma}^2} \left[ 
\left(
\big( \bar{\mu}_i(A_i, \sigma'^2_{M_i} ) - \mu_i \big)^\top  
 \big( \bar{\mu}_i(A_i, \sigma''^2_{M_i} ) - \mu_i\big)
\right)^2
 \right] \nonumber \\
&~\le~ \EXP_{\tilde{\sigma}^2} \left[ 
\big\| \bar{\mu}_i(A_i, \sigma'^2_{M_i} ) - \mu_i \big\|^2_2  
\big\| \bar{\mu}_i(A_i, \sigma''^2_{M_i} ) - \mu_i \big\|^2_2
 \right] \nonumber \\
 &~\le~ \sqrt{\EXP_{\tilde{\sigma}^2} \left[ 
\big\| \bar{\mu}_i(A_i, \sigma'^2_{M_i} ) - \mu_i \big\|^4_2
 \right]}
 \sqrt{\EXP_{\tilde{\sigma}^2} \left[ 
\big\| \bar{\mu}_i(A_i, \sigma''^2_{M_i} ) - \mu_i \big\|^4_2 
 \right]}
  \nonumber \\
&~\le ~
%8  \frac{d\sigma^6_2}{\ell^3 \sigma^2_1} \left(2  + \frac{d\sigma^6_2}{\ell^3 \sigma^2_1} \right) 
d(2+d)  
  \left( \frac{\frac{1}{\tau^2} + \ell \frac{{\sigma}^2_{\max}}{\sigma^4_{\min}} }
  {
  \left(\frac{1}{\tau^2} +\ell \frac{1}{\sigma^2_{\min}} \right)^2
  }
   \right)^2 
\; . \label{eq:inner_bdd}
\end{align}
Combining
\eqref{eq:BP_minus_mu_i}, 
\eqref{eq:correct_sigma}
and \eqref{eq:inner_bdd},
we have
 \begin{align}
 \EXP_{\tilde{\sigma}^2 }  \left[ \| \hat{\mu}^{{\MF \bayIter} (k)}_i (A)  - \mu_i \|^2_2  \right] 
& \le 
 d  \bar{\sigma}^2_i (\tilde{\sigma}^2_{M_i})
%  d\EXP[\bar{\sigma}^2_i (\tilde{\sigma}^2_{M_i})] 
 %\nonumber \\&
 + 
\sqrt{ d(2+d) }  
  \left( \frac{\frac{1}{\tau^2} + \ell \frac{{\sigma}^2_{\max}}{\sigma^4_{\min}} }
  {
  \left(\frac{1}{\tau^2} +\ell \frac{1}{\sigma^2_{\min}} \right)^2
  }
   \right) 
 \sum_{\sigma''^2_{M_i}}
  \sum_{\sigma'^2_{M_i} \neq \tilde{\sigma}^2_{M_i}}
\sqrt{
\EXP_{\tilde{\sigma}^2} 
\left[ \left( b^k_i (\sigma'^2_{M_i})   b^k_i (\sigma''^2_{M_i}) \right)^2 \right]} \;.
\label{eq:MSE_BP_bdd_pre}
\end{align}
Using Cauchy-Schwarz inequality and Jensen's inequality sequentially,
it follows that
\begin{align*}
 \sum_{\sigma''^2_{M_i}}
  \sum_{\sigma'^2_{M_i} \neq \tilde{\sigma}^2_{M_i}}
\left({
\EXP_{\tilde{\sigma}^2} 
\left[ \left( b^k_i (\sigma'^2_{M_i})   b^k_i (\sigma''^2_{M_i}) \right)^2 \right]} 
\right)^{1/2}
&\le 
 \sum_{\sigma''^2_{M_i}}
  \sum_{\sigma'^2_{M_i} \neq \tilde{\sigma}^2_{M_i}}
\left(
\EXP_{\tilde{\sigma}^2} 
\left[ \left( b^k_i (\sigma'^2_{M_i})   \right)^4 \right]
\right)^{1/4}
\left(
\EXP_{\tilde{\sigma}^2} 
\left[ \left( b^k_i (\sigma''^2_{M_i}) \right)^4 \right]
\right)^{1/4} \\
&=
\left(
  \sum_{\sigma'^2_{M_i} \neq \tilde{\sigma}^2_{M_i}}
\left(
\EXP_{\tilde{\sigma}^2} 
\left[ \left( b^k_i (\sigma'^2_{M_i})   \right)^4 \right]
\right)^{1/4}
\right)
\left(
 \sum_{\sigma''^2_{M_i}}
\left(
\EXP_{\tilde{\sigma}^2} 
\left[ \left( b^k_i (\sigma''^2_{M_i}) \right)^4 \right]
\right)^{1/4} 
\right) 
\\
&\le
\left(
  \sum_{\sigma'^2_{M_i} \neq \tilde{\sigma}^2_{M_i}}
\EXP_{\tilde{\sigma}^2} 
\left[ \left( b^k_i (\sigma'^2_{M_i})   \right)^4 \right]
\right)^{1/4}
\left(
 \sum_{\sigma''^2_{M_i}}
\EXP_{\tilde{\sigma}^2} 
\left[ \left( b^k_i (\sigma''^2_{M_i}) \right)^4 \right]
\right)^{1/4} \\
&\le
\left(
  \sum_{\sigma'^2_{M_i} \neq \tilde{\sigma}^2_{M_i}}
\EXP_{\tilde{\sigma}^2} 
\left[  b^k_i (\sigma'^2_{M_i})    \right]
\right)^{1/4}
\left(
 \sum_{\sigma''^2_{M_i}}
\EXP_{\tilde{\sigma}^2} 
\left[  b^k_i (\sigma''^2_{M_i}) \right]
\right)^{1/4} \\
& = 
\left(
1-
\EXP_{\tilde{\sigma}^2} 
\left[  b^k_i (\tilde{\sigma}^2_{M_i})    \right]
\right)^{1/4}
\; ,
\end{align*}
where for the last inequality and the last equality, we use 
the fact that $b^k_i$ is normalized, i.e.,
$0 \le b^k_i(\sigma^2_{M_i}) \le 1$ and 
 $\sum_{\sigma^2_{M_i}} b^k_i(\sigma^2_{M_i}) = 1$.
This completes the proof of \eqref{eq:MSE_BP_bdd} with \eqref{eq:MSE_BP_bdd_pre}.

%where for the last inequality, we use \eqref{eq:inner_bdd}.
%with $\sigma^2_{\min} \le \sigma^2_1 \le ... \le \sigma^2_S \le \sigma^2_{\max}$.
%where the last term In addition, using the fact that $0 \le b^k_u (\sigma^2_u) \le 1$, it is straightforward to check
%that
%\begin{align*}
%&\sum_{\sigma'^2_{M_i} \neq \tilde{\sigma}^2_{M_i}} 
%%\sum_{\sigma''^2_{M_i} }
%\sqrt{
%\EXP_{\tilde{\sigma}^2} \left[  \prod_{u \in M_{i}} \left(b^k_u (\sigma'^2_u)  \right)^2 \right]} \\
%&\le
%\sum_{u \in M_i}  \sum_{\sigma'^2_{u} \neq \tilde{\sigma}^2_{u}} 
%\sqrt{
%\EXP_{\tilde{\sigma}^2} \left[   b^k_u (\sigma'^2_u)    \right]}
%\end{align*}

\subsection{Proof of Inequality \eqref{eq:MSE_diff_bound}}  
\label{sec:MSE_diff_bound_pf}

We start with rewriting the difference between MSE's of $\hat{\mu}^{{\MF ora}(k)}_\tau (A)$ and $\hat{\mu}^{{\MF {\bayIter}} (k)}_\tau (A)$ for $\tau \in V$ as follows:
%$\| \hat{\mu}^{*}_i (A)  - \mu_i \|^2_2 - \| \hat{\mu}^{{\MF BP} (k)}_i (A)  - \mu_i \|^2_2$
%as follows:
\begin{align*}
& \| \hat{\mu}^{{\MF ora}(k)}_\tau (A)  - \mu_\tau \|^2_2 - \| \hat{\mu}^{{\MF \bayIter} (k)}_\tau (A)  - \mu_\tau \|^2_2 \\
& = \sum_{\sigma'^2_{M_\tau}, \sigma''^2_{M_\tau}  \in \mathcal{S}^{\ell}}
\left(
 \Pr[\sigma^2_{M_\tau} = \sigma'^2_{M_\tau} \mid A, \sigma^2_{\partial W_{\tau, 2k+1}}]
 \Pr[\sigma^2_{M_\tau} = \sigma''^2_{M_\tau} \mid A, \sigma^2_{\partial W_{\tau, 2k+1}}]
  - b^k_\tau (\sigma'^2_{M_\tau}) b^k_\tau (\sigma''^2_{M_\tau}) \right)
\\ 
&  \quad\times 
\left( \bar{\mu}_\tau \left(A_\tau, \sigma'^2_{M_\tau} \right) - \mu_\tau \right)^\top  
 \left( \bar{\mu}_\tau \left(A_\tau, \sigma''^2_{M_\tau} \right) - \mu_\tau \right)  \;.
% & \le \sum_{\sigma'^2_{M_i}, \sigma''^2_{M_i}  \in \mathcal{S}^{\ell}}
%2
%\left( \Pr[\sigma^2_{M_i} = \sigma'^2_{M_i} \mid A] - b^k_i (\sigma'^2_{M_i}) \right)
%\left( \bar{\mu}_i\left(A_i, \sigma'^2_{M_i} \right) - \mu_i\right)^\top  
% \left( \bar{\mu}_i \left(A_i, \sigma''^2_{M_i} \right) - \mu_i \right) 
\end{align*}
%where for the last inequality, we use the fact that
%$b^k_i$ is normalized, i.e.,
%$0 \le b^k_i(\sigma^2_{M_i}) \le 1$.

Then, using Cauchy-Schwarz inequality 
for random variables $X$ and $Y$, i.e., $|\EXP[XY]| \le \sqrt{\EXP[X^2] \EXP[Y^2]}$,
we have 
\begin{align*}
&\EXP\left[ \big( \text{MSE} (\hat{\mu}^{{\MF ora}(k)}_\tau (A)) - \text{MSE} (\hat{\mu}^{{\MF \bayIter} (k)}_\tau (A))  \big) \right] \nonumber \\
& \le \sum_{\sigma^2_{M_\tau}, \sigma'^2_{M_\tau}  \in \mathcal{S}^{\ell}}
%\underbrace{
\sqrt{
 \EXP \bigg[ \Big(
\Pr[\sigma^2_{M_\tau}= \sigma'^2_{ M_{\tau}} \mid A, \sigma^2_{\partial W_{\tau, 2k+1}}] 
\Pr[\sigma^2_{M_\tau} = \sigma''^2_{ M_{\tau}} \mid A, \sigma^2_{\partial W_{\tau, 2k+1}}] 
- b^k_\tau (\sigma'^2_{ M_{\tau}}) 
b^k_\tau (\sigma''^2_{ M_{\tau}}) 
 \Big)^2 \bigg]} 
\\
&\quad \times \sqrt{  
\EXP \left[ 
\left(
\big( \bar{\mu}_\tau(A_\tau, \sigma^2_{M_\tau} ) - \mu_i \big)^\top  
 \big( \bar{\mu}_\tau(A_\tau, \sigma'^2_{M_\tau} ) - \mu_i\big)
\right)^2
 \right]
  }
%  }_{(b)}
  % \label{eq:MSE_bdd}
\end{align*}
which completes the proof of \eqref{eq:MSE_diff_bound}
with \eqref{eq:inner_bdd}.

\end{document}